\documentclass{article}
\usepackage{microtype}
\usepackage{graphicx, subfig} 
\usepackage{booktabs} 
\usepackage{amstext, amsmath, amsthm, amssymb}
\usepackage{nicefrac}
\usepackage{stackrel}

\usepackage{hyperref}
\usepackage{natbib}

\usepackage[accepted]{icml2020}

\newtheorem{theorem}{Theorem}[section]
\newtheorem{definition}{Definition}[section]
\newtheorem{lemma}[theorem]{Lemma}
\newtheorem{proposition}[theorem]{Proposition}

\newtheorem{corollary}[theorem]{Corollary}

\def\eqref#1{(\ref{#1})}
\def\frf{\hat{f}^{(RF)}}
\def\frfl{\hat{f}^{(RF)}_{\lambda}}
\def\frflg{\hat{f}^{(RF)}_{\lambda, \gamma}}

\def\frfrlg{\hat{f}^{(RF)}_{\lambda \searrow 0,\gamma}}

\def\fklambda{\hat{f}^{(K)}_{\lambda}}
\def\fkl{\hat{f}^{(K)}_{\tilde\lambda}}
\def\fkrl{\hat{f}^{(K)}_{\lambda \searrow 0}}
\def\RR{\mathbb R}
\def\EE{\mathbb E}
\def\Tr{\mathrm{Tr}}
\def\iK{K^{-1}}
\def\Lt {\tilde \lambda }
\def\dLt{\partial_\lambda \tilde \lambda}
\def\ynik{\| y \|_{\iK}}

\icmltitlerunning{Implicit Regularization of Random Feature Models}

\begin{document}

\twocolumn[
\icmltitle{Implicit Regularization of Random Feature Models}



\icmlsetsymbol{equal}{*}

\begin{icmlauthorlist}
\icmlauthor{Arthur Jacot}{equal,Math}
\icmlauthor{Berfin \c{S}im\c{s}ek}{equal,Math,CS}
\icmlauthor{Francesco Spadaro}{Math}
\icmlauthor{Cl\'{e}ment Hongler}{Math}
\icmlauthor{Franck Gabriel}{Math}
\end{icmlauthorlist}

\icmlaffiliation{Math}{Chair of Statistical Field Theory, \'{E}cole Polytechnique F\'{e}d\'{e}rale de Lausanne, Lausanne, Switzerland}
\icmlaffiliation{CS}{Laboratory of Computational Neuroscience, \'{E}cole Polytechnique F\'{e}d\'{e}rale de Lausanne, Lausanne, Switzerland}

\icmlcorrespondingauthor{Arthur Jacot}{arthur.jacot@epfl.ch}

\icmlkeywords{Random Features, Kernel Methods, Machine Learning}

\vskip 0.3in
]



\printAffiliationsAndNotice{\icmlEqualContribution}

\begin{abstract}
Random Feature (RF) models are used as efficient parametric approximations of kernel methods. We investigate, by means of random matrix theory, the connection between Gaussian RF models and Kernel Ridge Regression (KRR).
For a Gaussian RF model with $P$ features, $N$ data points, and a ridge $\lambda$, we show that the average (i.e. expected) RF predictor is close to a KRR predictor with an \textit{effective ridge} $\tilde{\lambda}$.
We show that $\tilde{\lambda} > \lambda$ and $\tilde{\lambda} \searrow \lambda$ monotonically as $P$ grows, thus revealing the \textit{implicit regularization effect} of finite RF sampling.
We then compare the risk (i.e. test error) of the $\tilde{\lambda}$-KRR predictor with the average risk of the $\lambda$-RF predictor and obtain a precise and explicit bound on their difference.
Finally, we empirically find an extremely good agreement between the test errors of the average $\lambda$-RF predictor and $\tilde{\lambda}$-KRR predictor.
\end{abstract}

\section{Introduction}
In this paper, we consider the Random Feature (RF) model which is an approximation of Kernel Methods \cite{rahimi-08} which has seen many recent theoretical developements.

The conventional wisdom suggests that to ensure good generalization performance, one should choose a model class that is complex enough to learn the signal from the training data, yet simple enough to avoid fitting spurious patterns therein \cite{bishop-06}. This view has been questioned by recent developments in machine learning.
First, \citet{zhang-16} observed that modern neural network models can perfectly fit randomly labeled training data, while still generalizing well.
Second, the test error as a function of parameters exhibits a so-called `double-descent' curve for many models including neural networks, random forests, and random feature models \cite{advani-17, spigler-18, belkin-18, mei-19, belkin-19, nakkiran-19}.

The above models share the feature that for fixed input, the learned predictor $\hat{f}$ is random: for neural networks, this is due to the random initialization of the parameters and/or to the stochasticity of the training algorithm; for random forests, to the random branching; for random feature models, to the sampling of random features.
The somehow surprising generalization behavior of these models has recently been the subject of increasing attention. In general, the risk (i.e. test error) is a random variable with two sources of randomness: the usual one due to the sampling of the training set, and the second one due to the randomness of the model itself.

We consider the Random Feature (RF) model \cite{rahimi-08} with features sampled from a Gaussian Process (GP) and study the RF predictor $\hat{f}$ minimizing the regularized least squares error, isolating the randomness of the model by considering fixed training data points. RF models have been the subject of intense research activity: they are (randomized) approximations of Kernel Methods aimed at easing the computational challenges of Kernel Methods while being asymptotically equivalent to them \cite{rahimi-08, yang-12, sriperumbudur-15, yu-16}. Unlike the asymptotic behavior, which is well studied, RF models with a finite number of features are much less understood.



\subsection{Contributions}

We consider a model of Random Features (RF) approximating a kernel method with kernel $ K $. This model consists of $ P $ Gaussian features, sampled i.i.d. from a (centered) Gaussian process with covariance kernel $ K $. For a given training set of size $N$, we study the distribution of the RF predictor $\frfl$ with ridge parameter $\lambda > 0$ ($L^2$ penalty on the parameters) and denote it by $\lambda$-RF. We show the following:
\begin{itemize}
\item The distribution of $\frfl$ is that of a mixture of Gaussian processes.
\item The expected RF predictor is close to the $ \tilde\lambda $-KRR (Kernel Ridge Regression) predictor for an effective ridge parameter $\tilde{\lambda}>0$.
\item The effective ridge $ \tilde \lambda > \lambda $ is determined by the number of features $ P$, the ridge $\lambda $ and the Gram matrix of $ K $ on the dataset; $ \tilde \lambda $ decreases monotonically to $\lambda$ as $ P $ grows, revealing the implicit regularization effect of finite RF sampling. Conversely, when using random features to approximate a kernel method with a specific ridge $\lambda^*$, one should choose a smaller ridge $\lambda<\lambda^*$ to ensure $\tilde{\lambda}(\lambda)=\lambda^*$.
\item The test errors of the expected $\lambda$-RF predictor and of the $\tilde\lambda$-KRR predictor $ \fkl $ are numerically found to be extremely close, even for small $ P $ and $ N $.
\item The RF predictor's concentration around its expectation can be explicitly controlled in terms of $ P $ and of the data; this yields in particular $\mathbb{E} [L( \frfl )] = L ( \fkl )+\mathcal{O}(P^{-1}) $ as $N,P\to\infty$ with a fixed ratio $\gamma = \nicefrac{P}{N}$ where $L$ is the MSE risk.
\end{itemize}
Since we compare the behavior of $\lambda$-RF and $\tilde\lambda$-KRR predictors on the same fixed training set, our result does not rely on any probabilistic assumption on the training data (in particular, we do not assume that our training data is sampled i.i.d.). While our proofs currently require the features to be Gaussian processes, we are confident that they could be generalized to a more general setting \cite{louart-17,benigni-2019}.

\subsection{Related works}


\textbf{Generalization of Random Features.} The generalization behavior of Random Feature models has seen intense study in the Statistical Learning Theory framework. \citet{rahimi-09} find that $\mathcal{O}(N)$ features are sufficient to ensure the $\mathcal{O}(\frac{1}{\sqrt{N}})$ decay of the generalization error of Kernel Ridge Regression (KRR). \citet{rudi-17} improve on their result and show that $\mathcal{O}(\sqrt{N} \log N)$ features is actually enough to obtain the $\mathcal{O}(\frac{1}{\sqrt{N}})$ decay of the KRR error.

\citet{hastie-19} use random matrix theory tools to compute the asymptotic risk when both $P,N \to \infty$ with $\frac P N \to \gamma > 0 $. When the training data is sampled i.i.d. from a Gaussian distribution, the variance is shown to explode at $ \gamma = 1$. In the same linear regression setup, \citet{bartlett-19} establish general upper and lower bounds on the excess risk. \citet{mei-19} prove that the double-descent (DD) curve also arises for random ReLU features, and adding a ridge suppresses the explosion around $ \gamma = 1 $.

\textbf{Double-descent and the effect of regularization.} For the cross-entropy loss, \citet{neyshabur-14} observed that for two-layer neural networks the test error exhibits the double-descent (DD) curve as the network width increases (without regularizers, without early stopping). For MSE and hinge losses, the DD curve was observed also in multilayer networks on the MNIST dataset \cite{advani-17, spigler-18}. \citet{neal-18} study the variance due to stochastic training in neural networks and find that it increases until a certain width, but then decreases down to $0$. \citet{nakkiran-19} establish the DD phenomenon across various models including convolutional and recurrent networks on more complex datasets (e.g. CIFAR-10, CIFAR-100).

\citet{belkin-18, belkin-19} find that the DD curve is not peculiar to neural networks and observe the same for random Fourier features and decision trees. In \citet{geiger-19}, the DD curve for neural networks is related to the variance associated with the random initialization of the Neural Tangent Kernel \cite{jacot-18}; as a result, ensembling is shown to suppress the DD phenomenon in this case, and the test error stays constant in the overparameterized regime. Recent theoretical work \cite{ascoli-20} study the same setting and derive formulas for the asymptotic error, relying on the so-called replica method.

\textbf{General Wishart Matrices.}
Our theoretical analysis relies on the study of the spectrum of the so-called general Wishart matrices of the form $W\Sigma W^{T}$ (for $N\times N$ matrix $\Sigma$ and $P\times N$ matrix $W$ with i.i.d. standard Gaussian entries) and in particular their Stieltjes transform $m_P(z)=\frac{1}{P}\mathrm{Tr}\left(W\Sigma W^{T}-zI_{P}\right)^{-1}$. A number of asymptotic results \cite{silverstein-1995, bai-2008} about the spectrum and Stieltjes transform of such matrices can be understood using the asymptotic freeness of $W^{T}W$ and $\Sigma$ \cite{gabriel-15, speicher-17}. In this paper, we provide non-asymptotic variants of these results for an arbitrary matrix $\Sigma$ (which in our setting is the kernel Gram matrix); the proofs in our setting are detailed in the Supp. Mat.

\subsection{Outline}
The rest of this paper is organized as follows:
\begin{itemize}
\item In Section \ref{sec:setup}, the setup (linear regression, Gaussian RF model, $\lambda$-RF predictor, and $\lambda$-KRR predictor) is introduced.

\item In Section \ref{sec:first-observations}, preliminary results on the distribution of the $\lambda$-RF model are provided: the RF predictors are Gaussian mixtures (Proposition \ref{gaussian-mixture}) and the $\lambda\searrow0$-RF model is unbiased in the overparameterized regime (Corollary \ref{cor:average-ridgless-overparametrized}). Graphical illustrations of the RF predictors in various regimes are presented (Figure \ref{fig:RF-predictor}).

\item In Section \ref{sec:average-RF}, the first main theorem is stated (Theorem \ref{average-rf}): the average (expected) $\lambda$-RF predictor is close to the $ \Lt $-KRR predictor for an explicit $ \Lt > \lambda $. As a consequence (Corollary \ref{cor:difference_loss_expected_kernel_loss}), the test errors of these two predictors are close. Finally, numerical experiments show that the test errors are in fact virtually identical (Figure \ref{fig:average-RF-main}).

\item In Section \ref{sec:variance-RF}, the second main theorem is stated (Theorem \ref{variance-ridge}): a bound on the variance of the $\lambda$-RF predictor is given, which show that it concentrates around the average $\lambda$-RF predictor. As a consequence, the test error of the $\lambda$-RF predictor is shown to be close to that of the $\Lt$-KRR predictor (Corollary \ref{cor:expected-loss-krr-loss}). The ridgeless $ \lambda \searrow 0 $ case is then investigated (Section \ref{sec:double-descent-curve}): a lower bound on the variance of the $ \lambda $-RF predictor is given, suggesting an explanation for the double-descent curve in the ridgeless case.

\item In Section \ref{sec:conclusion}, we summarize our results and discuss potential implications and extensions.
\end{itemize}

\section{Setup}
\label{sec:setup}

Linear regression is a parametric model consisting of linear combinations
\begin{equation}\label{eq:linear-func}
f_{\theta} = \frac{1}{\sqrt{P}} \left( \theta_1 \phi^{(1)} + \cdots + \theta_P \phi^{(P)} \right) \notag
\end{equation}
of (deterministic) features $\phi^{(1)}, \ldots ,\phi^{(P)}:\mathbb{R}^{d}\to\mathbb{R}$. We consider an arbitrary training dataset $(X,y)$ with $X=[x_{1},...,x_{N}]\in\mathbb{R}^{d\times N}$ and $y=[y_{1},\dots,y_{N}]\in\mathbb{R}^{N}$, where the labels could be noisy observations. For a ridge parameter $\lambda > 0$,
the linear estimator corresponds to the parameters $\hat{\theta}=[\hat{\theta}_{1},\dots,\hat{\theta}_{P}]\in\mathbb{R}^{P}$ that minimize the (regularized) Mean Square Error (MSE) functional $\hat{L}_{\lambda}$ defined by
\begin{equation}\label{eq:mse-reg}
 \hat{L}_{\lambda}(f_{\theta})=\frac{1}{N}\sum_{i=1}^{N}\left(f_{\theta}(x_{i})-y_{i}\right)^{2} + \frac{\lambda}{N} \| \theta \|^2.
\end{equation}
The \emph{data matrix} $ F $ is defined as the $N\times P$ matrix with entries $F_{ij}={1\over\sqrt{P}} \phi^{(j)}(x_{i})$. The minimization of \eqref{eq:mse-reg} can be rewritten in terms of $ F $ as
\begin{equation}
  \hat{\theta} = \mathrm{argmin}_{\theta} \| F\theta - y\|^{2} + \lambda \| \theta \|^2.
\end{equation}
The optimal solution $\hat{\theta}$ is then given by \begin{equation}\label{eq:theta-ridge}
 \hat{\theta}=F^{T}\left(FF^{T}+\lambda I_N\right)^{-1}y
\end{equation}
and the optimal predictor $\hat{f} = f_{\hat \theta} $ by
\begin{equation}\label{eq:opt-fun}
   \hat{f}(x)={1\over\sqrt{P}}\sum_{j=1}^{P}\phi^{(j)}(x)F_{:,j}^{T}\left(FF^{T}+\lambda I_N \right)^{-1}y.
\end{equation}
In this paper, we consider linear models of \textit{Gaussian random features} associated with a kernel $ K\!: \mathbb{R}^{d}\times\mathbb{R}^{d}\to\mathbb{R}$. We take $\phi^{(j)} = f^{(j)}$, where $ f^{(1)}, \ldots, f^{(P)} $ are sampled i.i.d. from a Gaussian Process of zero mean (i.e. $ \mathbb{E} [f^{(j)}(x)] = 0 $ for all $ x\in\mathbb{R}^d $) and with covariance $K$ (i.e. $\mathbb{E}[f^{(j)}(x)f^{(j)}(x')] = K(x,x') $ for all $ x, x'\in\mathbb{R}^d $). In our setup, the optimal parameter $\hat \theta$ still satisfies \eqref{eq:theta-ridge} where $F$ is now a random matrix. The associated predictor, called $\lambda$-RF predictor, is then given by

\begin{figure*}[t!]
    \centering
    \!\!\!
    \subfloat[$P=2, \lambda=10^{-4}$]{
        \includegraphics[width=0.246\textwidth]{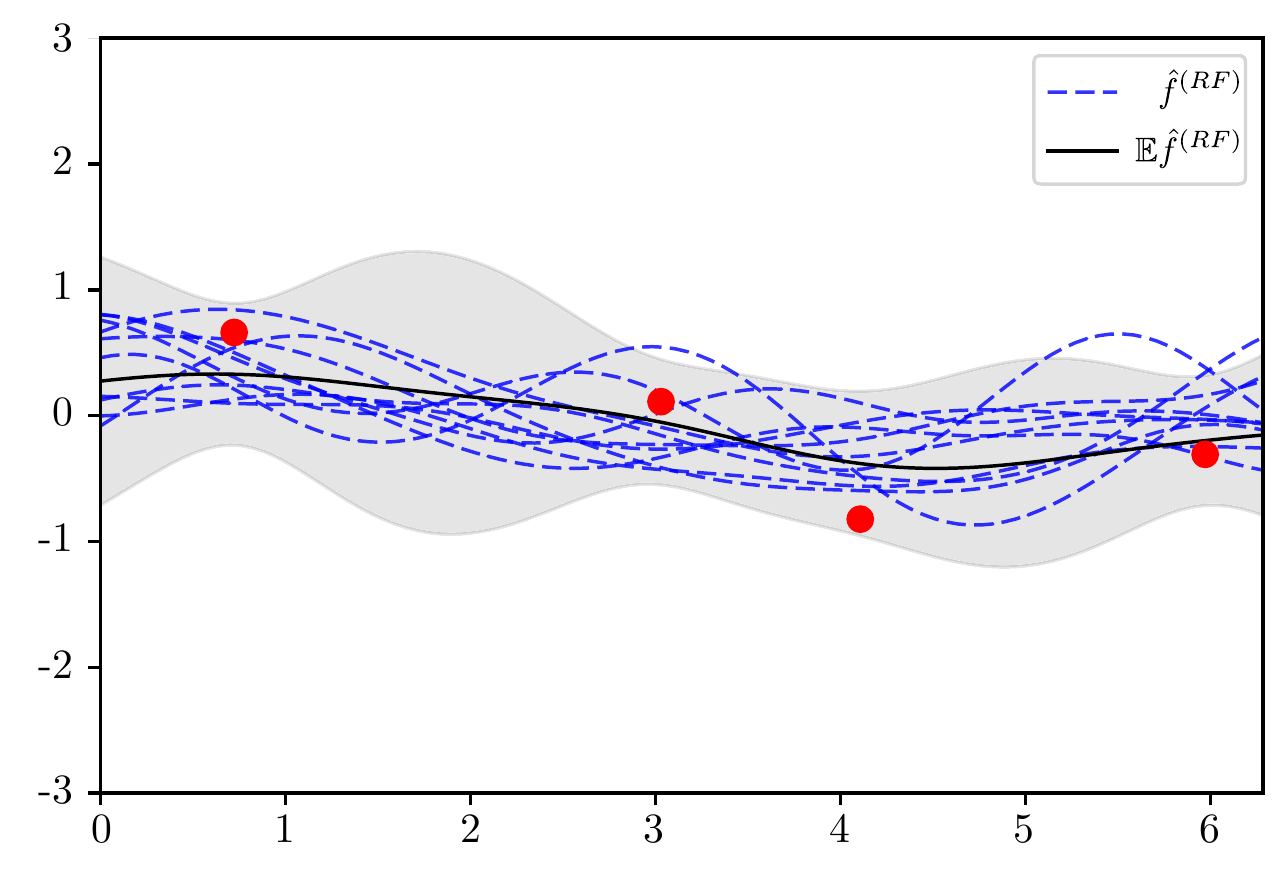}%
        \label{fig:RF1}
        } \!\!\!
    \subfloat[$P=4, \lambda=10^{-4}$]{
        \includegraphics[width=0.246\textwidth]{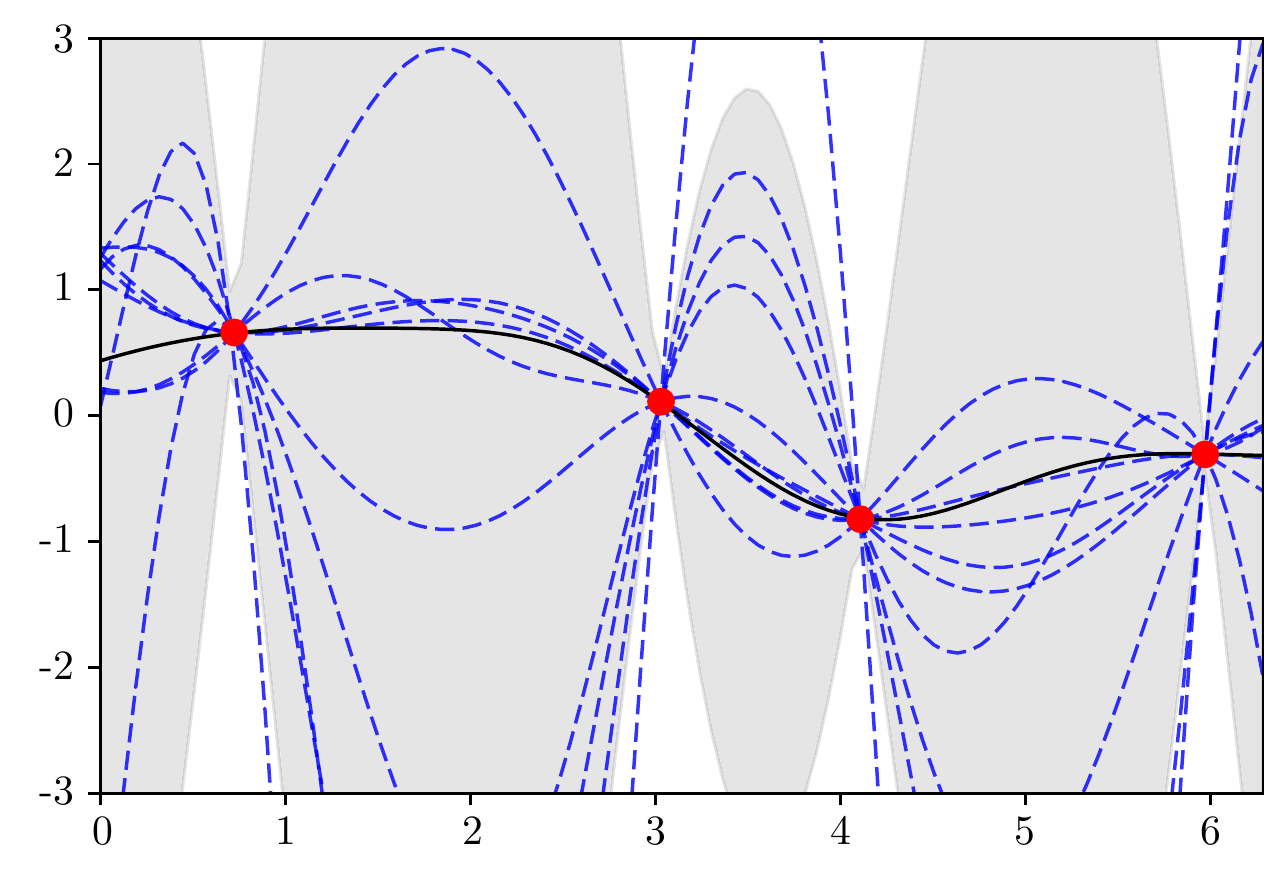}
        \label{fig:RF2}
        } \!\!\!
    \subfloat[$P=4, \lambda=0.1$]{
        \includegraphics[width=0.246\textwidth]{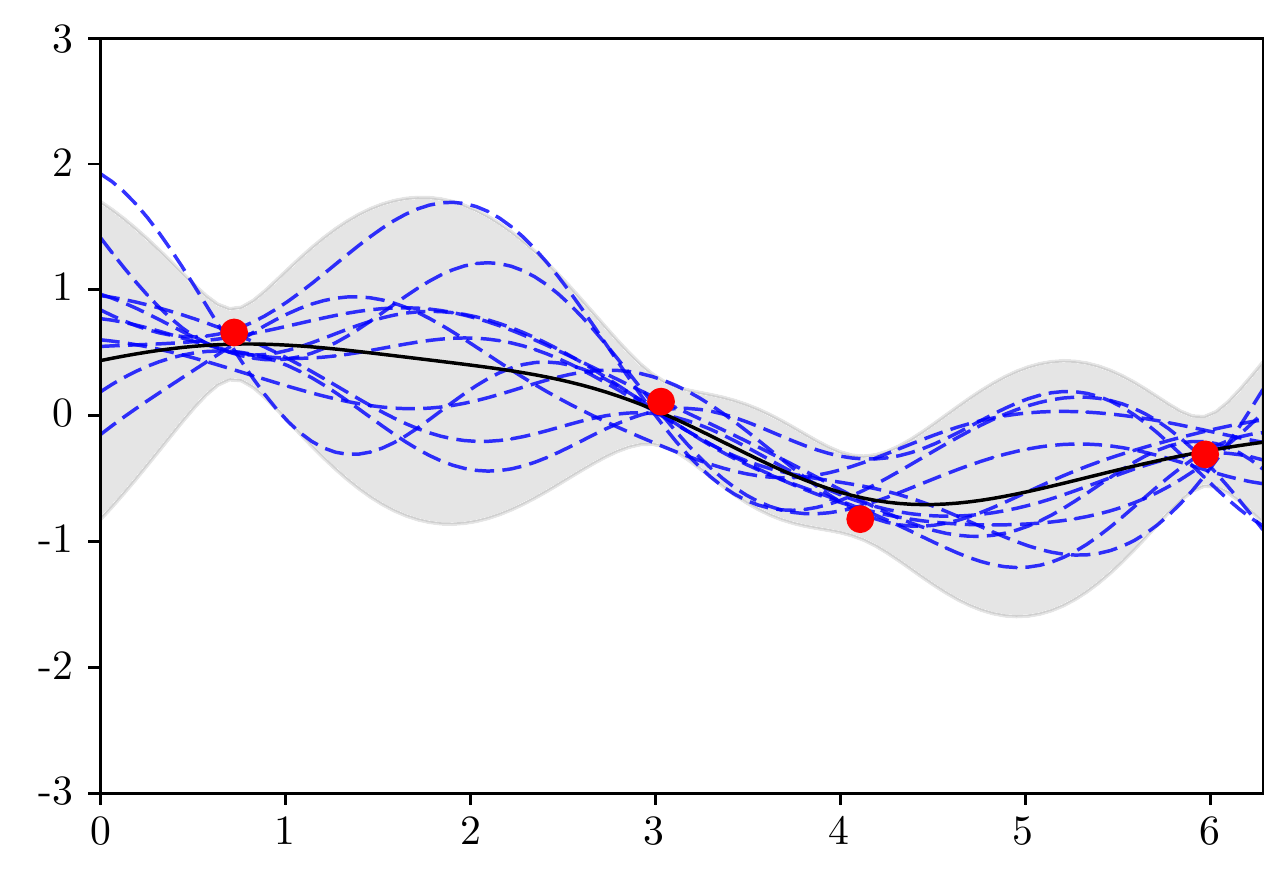}
          \label{fig:RF3}
          } \!\!\!
    \subfloat[$P=100, \lambda=10^{-4}$]{
        \includegraphics[width=0.246\textwidth]{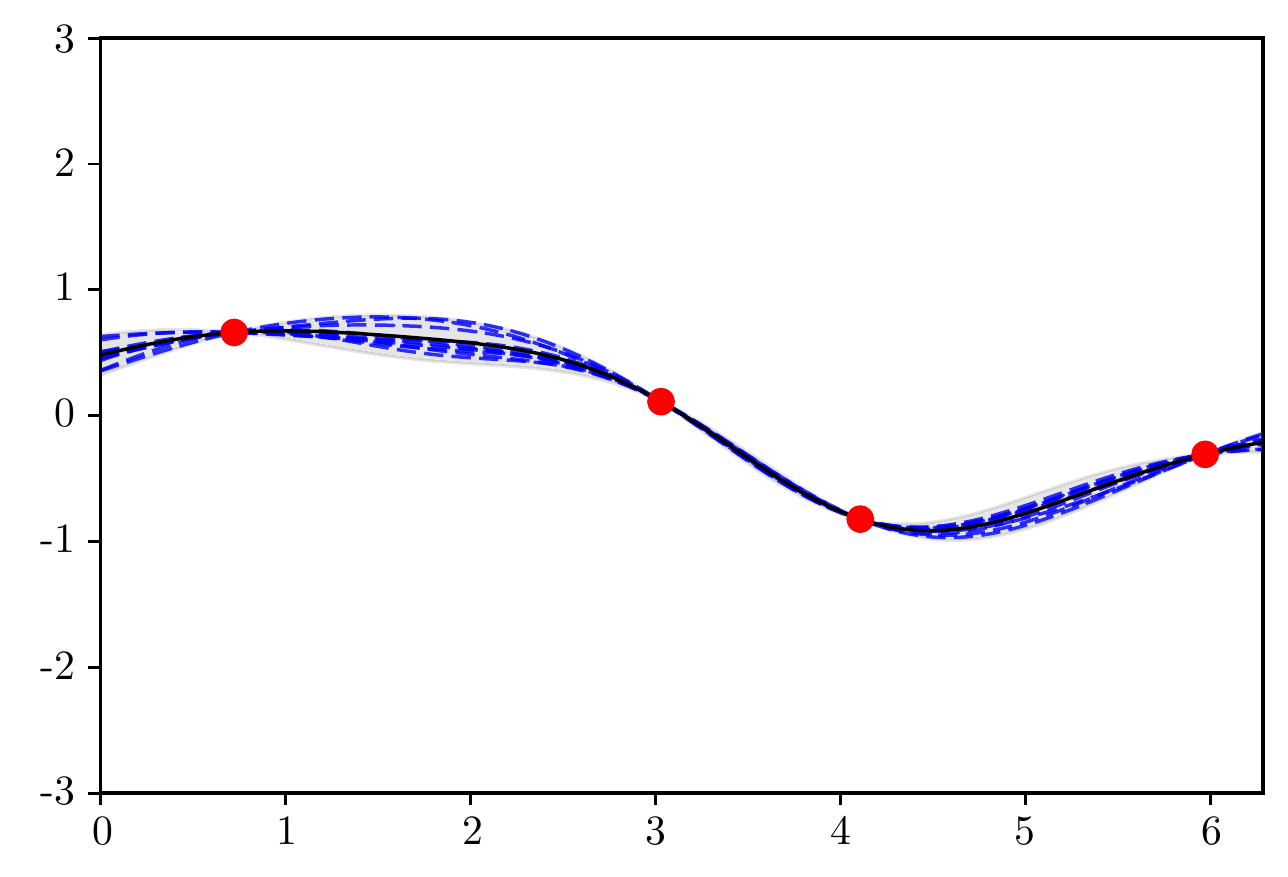}
          \label{fig:RF4}
          } \!\!\!
    \caption{\textit{Distribution of the RF Predictor.} Red dots represent a sinusoidal dataset $y_i=\sin(x_i)$ for $N=4$ points $x_i$ in $[0, 2\pi)$. For selected $P$ and $\lambda$, we sample ten RF predictors (blue dashed lines) and compute empirically the average RF predictor (black lines) with $\pm 2$ standard deviations intervals (shaded regions).}
    \label{fig:RF-predictor}
\end{figure*}

\begin{definition}[Random Feature Predictor]
Consider a kernel $ K:\mathbb{R}^{d}\times\mathbb{R}^{d}\to\mathbb{R}$, a ridge $ \lambda > 0 $, and random features $ f^{(1)}, \ldots, f^{(P)} $ sampled i.i.d. from a centered Gaussian Process of covariance $ K $. Let $\hat \theta$ be the optimal solution to \eqref{eq:mse-reg} taking $ \phi^{(j)} = f^{(j)} $. The Random Feature predictor with ridge $\lambda$ is the random function $\frfl: \mathbb{R}^{d} \to \mathbb{R}$ defined by
\begin{equation} \label{eq:rf-predictor}
  \frfl(x)={1\over\sqrt{P}}\sum_{j=1}^{P}\hat\theta_{j}f^{(j)}(x).
\end{equation}
\end{definition}

The $ \lambda $-RF can be viewed as an approximation of kernel ridge predictors: observing from \eqref{eq:opt-fun} that $ \frfl $ only depends on the scalar product
$ K_{P}(x,x') = \frac{1}{P} \sum_{j=1}^{P}f^{(j)}(x)f^{(j)}(x') $ between datapoints, we see that as $ P \to \infty $, $ K_P \to K $ and hence $ \frfl $ converges \cite{rahimi-08} to a kernel predictor with ridge $ \lambda $ \cite{Scholkopf-1998}, which we call $ \lambda $-KRR predictor.


\begin{definition}[Kernel Predictor] Consider a kernel function $K:\mathbb{R}^{d}\times\mathbb{R}^{d}\to\mathbb{R}$ and a ridge $\lambda > 0$. The Kernel Predictor is the function $\fklambda:\mathbb R^{d} \to \mathbb R$\begin{equation*}\label{eq:kernel-predictor}
  \ \fklambda(x) = K(x,X)(K(X,X)+\lambda \mathrm{I}_N)^{-1}y
\end{equation*}
where $K(X,X)$ is the $N\times N$ matrix of entries $\left(K(X,X)\right)_{ij}=K(x_{i},x_{j})$
and $K(\ \cdot\ ,X):\mathbb{R}^{d} \to \mathbb{R}^{N}$ is the map $\left(K(x,X)\right)_{i}=K(x,x_{i})$.
\end{definition}

\subsection{Bias-Variance Decomposition.}\label{subsec:bias-variance}
Let us assume that there exists a true regression function $f^{*}:\mathbb R^{d} \to \mathbb R$ and a data generating distribution $\mathcal{D}$ on $\mathbb R^{d}$. The risk of a predictor $f:\mathbb R^{d} \to \mathbb{R}$ is measured by the MSE defined as
\begin{equation*}
 \ L(f)=\mathbb{E}_{\mathcal{D}}\left[(f(x)-f^{*}(x))^{2}\right].
\end{equation*}
Let $\pi$ denote the joint distribution of the i.i.d. sample $f^{(1)},...,f^{(P)} $ from the centered Gaussian process with covariance kernel $ K $. The risk of $ \frfl $ can be decomposed into a bias-variance form as
\begin{equation*}\label{eq:pi-decomposition}
 \mathbb{E}_{\pi} \! \left[L(\frfl)\right]\!=\!L\left(\mathbb{E}_{\pi}[\frfl]\right)+\mathbb{E}_{\mathcal{D}} \! \left[ \mathrm{Var}_{\pi}(\frfl(x)) \right] \!.
\end{equation*}
This decomposition into the risk of the \textit{average} RF predictor and of the $ \mathcal{D} $-expectation of its variance will play a crucial role in the next sections.
This is in contrast with the classical bias-variance decomposition in \citet{geman-92}
\begin{equation}\label{eq:D-decomposition}
 \mathbb{{E}}_{\mathcal{D}^{\otimes N}}[L(f)]=L(\mathbb{E}_{\mathcal{D}^{\otimes N}}[f])+\mathbb{E}_{\mathcal{D}}[\mathrm{Var}_{\mathcal{D}^{\otimes N}}[f(x)]] \notag
\end{equation}
where $\mathcal{D}^{\otimes N}$ denotes the joint distribution on $x_{1},...,x_{N}$, sampled i.i.d. from $ \mathcal{D} $. Note that in our decomposition no probabilistic assumption is made on the data, which is fixed.

\subsection{Additional Notation}
In this paper, we consider a fixed dataset $(X,y)$ with distinct data points and a kernel $K$ (i.e. a positive definite symmetric function $ \RR^d \times \RR^d \to \RR $). We denote by $ \| y \|_{K^{-1}} $ the inverse kernel norm of the labels defined as $ y^T (K(X, X))^{-1} y $.

Let $UDU^{T}$ be the spectral decomposition of the kernel matrix $K(X,X)$, with $D=\mathrm{diag} (d_1, \ldots, d_N)$. Let $D^{\frac 1 2}=\mathrm{diag} (\sqrt{d_1}, \ldots, \sqrt{d_N})$ and set $K^{\frac{1}{2}} = UD^{\frac{1}{2}}U^{T}$. The law of the (random) data matrix $ F $ is now that of $\frac{1}{\sqrt{P}}K^{\frac{1}{2}}W^{T}$ where $W$ is a $P \times N$ matrix of i.i.d. standard Gaussian entries, so that $\mathbb E [F F^{T}] = K(X,X)$.

We will denote by $\gamma = \frac{P}{N}$ the parameter-to-datapoint ratio: the \emph{underparameterized regime} corresponds to $\gamma < 1$, while the \emph{overparameterized regime} corresponds to $\gamma \geq 1$. In order to stress the dependence on the ratio parameter $\gamma $, we write $ \frflg $ instead of $ \frfl $.


\section{First Observations}
\label{sec:first-observations}
The distribution of the RF predictor features a variety of behaviors depending on $ \gamma $ and $ \lambda $, as displayed in \autoref{fig:RF-predictor}. In the underparameterized regime $P < N$, sample RF predictors induce some \textit{implicit regularization} and do not interpolate the dataset (\ref{fig:RF1}); at the interpolation threshold $ P = N $, RF predictors interpolate the dataset but the variance explodes when there is no ridge (\ref{fig:RF2}), however adding some ridge suppresses variance explosion (\ref{fig:RF3}); in the overparameterized regime $ P \geq N $ with large $ P $, the variance vanishes thus the RF predictor converges to its average (\ref{fig:RF4}). We will investigate the average RF predictor (solid lines) in detail in Section \ref{sec:average-RF} and study its variance in Section \ref{sec:variance-RF}.


We start by characterizing the distribution of the RF predictor as a Gaussian mixture:
\begin{proposition}\label{gaussian-mixture}
Let $\frflg(x) $ be the random features predictor as in \eqref{eq:rf-predictor} and let $\hat{y} = F \hat{\theta}$ be the prediction vector on training data, i.e. $\hat{y}_i = \frflg(x_i)$. The process $ \frflg $ is a mixture of Gaussians: conditioned on $ F $, we have that $ \frflg $ is a Gaussian process. The mean and covariance of $ \frflg $ conditioned on $ F $ are given by
\begin{align}
& \mathbb E[\frflg(x) | F]  = K(x,X)K(X,X)^{-1}\hat{y}, \label{eq:mean-gaussian-mixture-main} \\
& \mathrm{Cov}[\frflg(x),\frflg(x') | F]  = \frac{\| \hat{\theta}\|^{2}}{P}\tilde{K}(x, x'), \label{eq:cov-gaussian-mixture-main}
\end{align}
with $\tilde{K}(x, x') = K(x,x')-K(x,X)K(X,X)^{-1}K(X,x')$ denoting the posterior covariance kernel.
\end{proposition}
The proof of Proposition \ref{gaussian-mixture} relies on the fact that $f^{(j)} $ conditioned on $ \left(f^{(j)}(x_{i})\right)_{i=1,\ldots,N}$ is a Gaussian Process.

Note that \eqref{eq:mean-gaussian-mixture-main} and \eqref{eq:cov-gaussian-mixture-main} depend on $ \lambda $ and $ P $ through $ \hat{y} $ and $ \| \hat{\theta}\|^{2} $; in fact, as the proof shows, these identities extend to the ridgeless case $ \lambda \searrow 0 $. For the ridgeless case, when one is in the overparameterized regime ($ P \geq N $), one can (with probability one) fit the labels $ y $ and hence $ \hat y = y $:
\begin{corollary}\label{cor:average-ridgless-overparametrized}
When $ P \geq N $, the average ridgeless RF predictor is equivalent to the ridgeless KRR predictor \begin{equation}
\EE \left[ \frfrlg (x) \right] = K(x,X)K(X,X)^{-1}y = \fkrl (x). \notag
\end{equation}
\end{corollary}

This corollary shows that in the overparameterized case, the ridgeless RF predictor is an unbiased estimator of the ridgeless kernel predictor. The difference between the expected loss of ridgeless RF predictor and that of the ridgeless KRR predictor is hence equal to the variance of the RF predictor. As will be demonstrated in this article, outside of this specific regime, a systematic bias appears, which reveals an implicit regularizing effect of random features.

\begin{figure*}[t]
    \centering 
      \!\!
      \subfloat[Evolution of $\tilde{\lambda}$]{
        \includegraphics[width=0.495\textwidth]{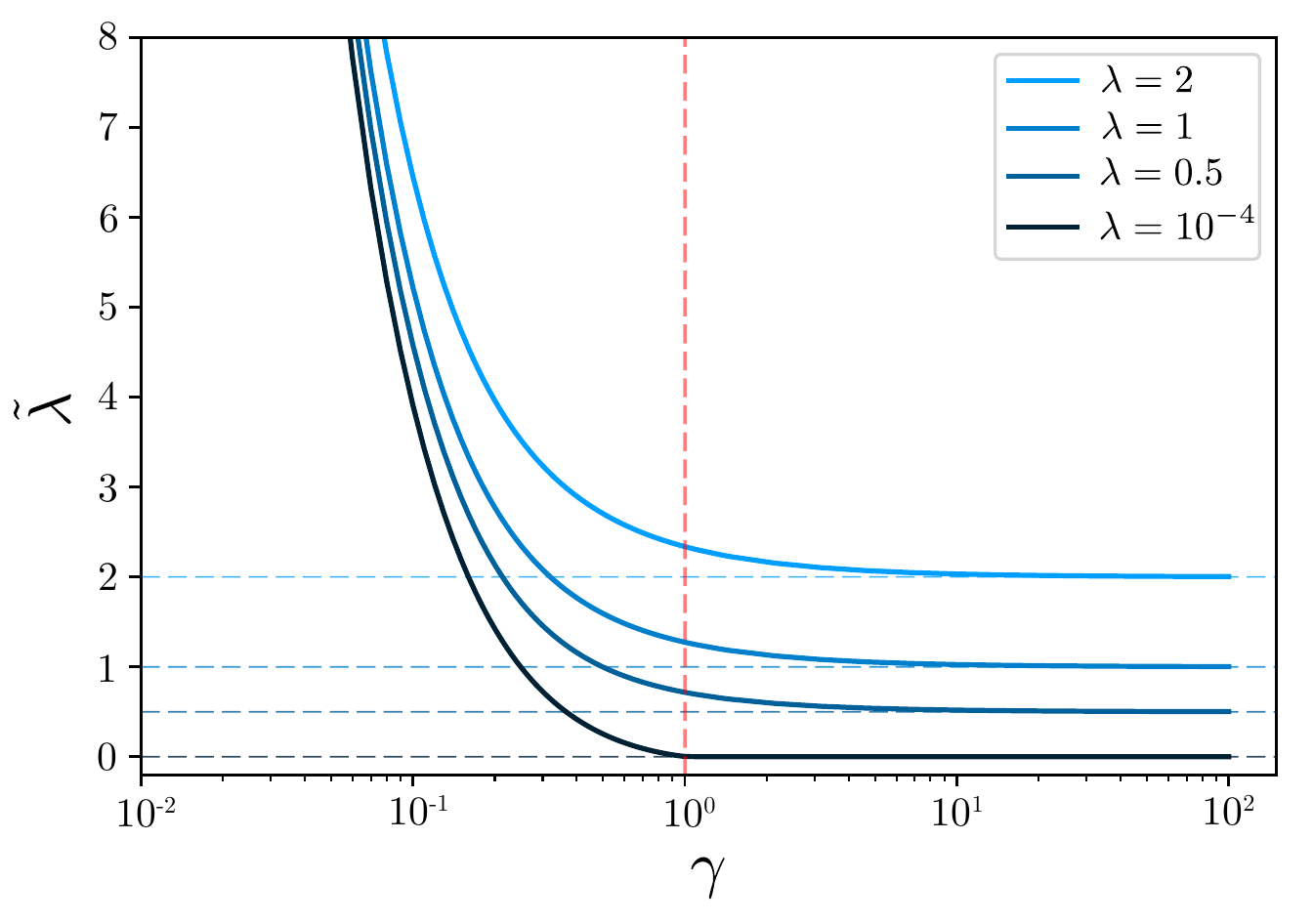}
        \label{fig:eff-ridge}
     }\!\!\!\!\!\!
    \subfloat[Average $\lambda$-RF predictor vs. $\tilde{\lambda}$-KRR]{
        \includegraphics[width=0.495\textwidth]{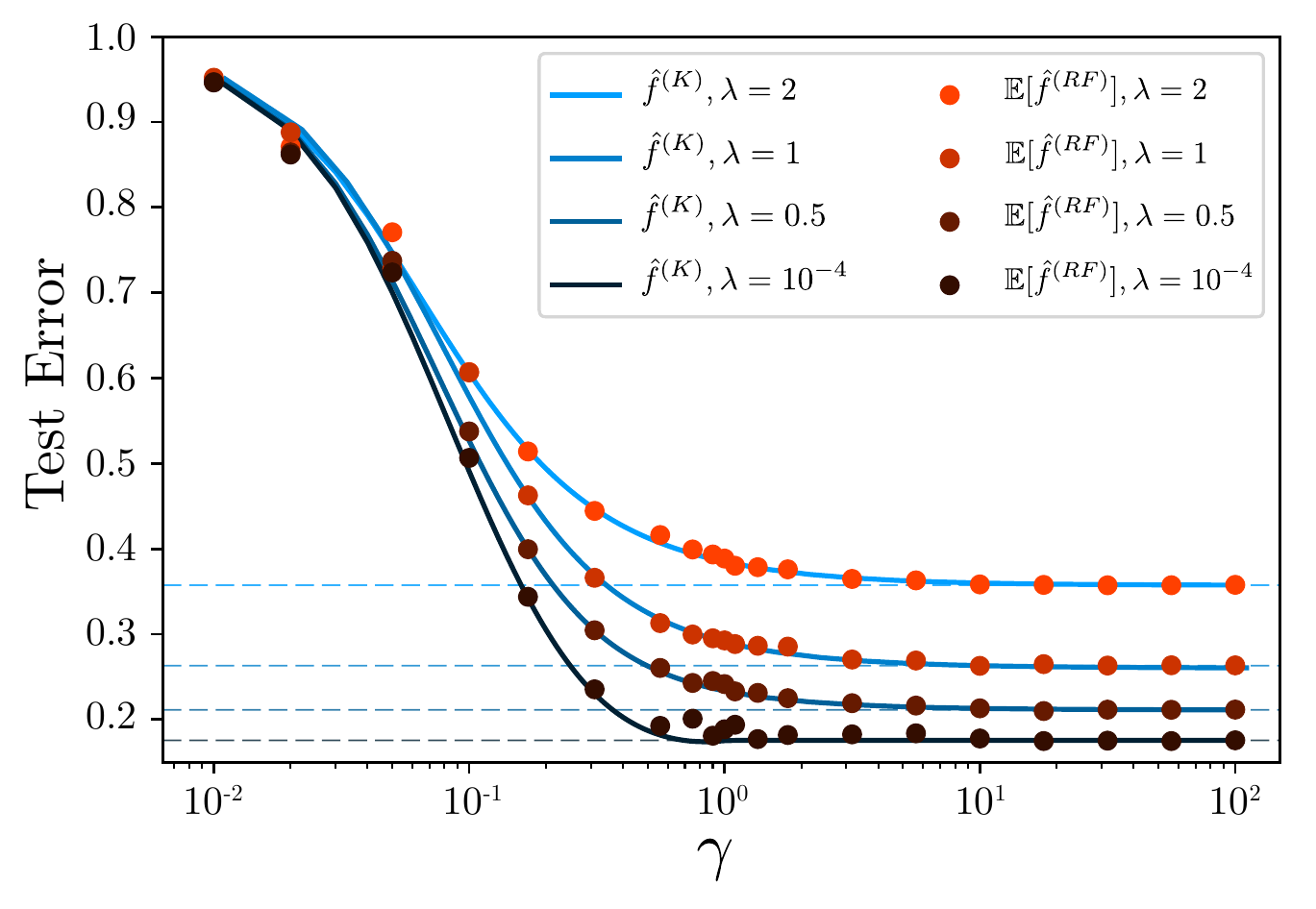}%
        \label{fig:averageRF-KRR}
      } \!\!
    \caption{\textit{Comparison of the test errors of the average $\lambda$-RF predictor and the $\tilde{\lambda}$-KRR predictor.} We train the RF predictors on $ N=100 $ MNIST data points where $ K  $ is the RBF kernel, i.e. $ K (x, x') = \exp \left( -\|x - x'\|^2 / \ell \right) $. We approximate the average $\lambda$-RF on $100$ random test points for various ridges $\lambda$.  In \textbf{(a)}, given $\gamma$ and $\lambda$, the effective ridge $ \tilde \lambda $ is computed numerically using \eqref{eq:defining-effective-ridge}.  In \textbf{(b)}, the test errors of the $\tilde{\lambda}$-KRR predictor (blue lines) and the empirical average of the ${\lambda}$-RF predictor (red dots) agree perfectly.}
    \label{fig:average-RF-main}
  \end{figure*}

\section{Average Predictor}
\label{sec:average-RF}
In this section, we study the average RF predictor $\mathbb E [ \frflg ]$. As shown by Corollary \ref{cor:average-ridgless-overparametrized} above, in the ridgeless overparmeterized regime, the RF predictor is an unbiased estimator of the ridgeless kernel predictor. However, in the presence of a non-zero ridge, we see the following \emph{implicit regularization effect}: the average $\lambda$-RF predictor is close to the $\tilde{\lambda}$-KRR predictor for an effective ridge $\tilde{\lambda} > \lambda $ (in other words, sampling a finite number $ P $ of features amounts to taking a greater kernel ridge $ \Lt $).
%
%


\begin{theorem}\label{average-rf}
For $N,P>0$ and $\lambda>0$, we have
\begin{equation}\label{eq:average-rf-bound}
\left|\mathbb{E} [ \frflg(x) ] - \fkl (x)\right|\leq\frac{c\sqrt{K(x,x)}\left\Vert y\right\Vert _{K^{-1}}}{P}
\end{equation}
where the effective ridge $\tilde{\lambda}(\lambda, \gamma) > \lambda$ is the unique
positive number satisfying
\begin{align}\label{eq:defining-effective-ridge}
\tilde{\lambda} & =\lambda+\frac{\tilde{\lambda}}{\gamma}\frac{1}{N}\sum_{i=1}^{N}\frac{d_{i}}{\tilde{\lambda}+d_{i}},
\end{align}
and where $ c > 0 $ depends on $ \lambda, \gamma $, and $ \frac 1 N \mathrm{Tr} K(X, X) $ only.
\end{theorem}
\begin{proof}
(Sketch; see Supp. Mat. for details) Set $A_{\lambda}=F(F^{T}F+\lambda I_{P})^{-1}F^{T}$. The vector of the predictions on the training set is given by $\hat{y}=A_{\lambda}y$ and the expected predictor is given by
\[
\mathbb{E}\left[\hat{f}_{\lambda,\gamma}^{(RF)}(x)\right]=K(x,X)K(X,X)^{-1}\mathbb{E}\left[A_{\lambda}\right]y.
\]

By a change of basis, we may assume the kernel Gram matrix to be diagonal,
i.e. $K(X,X)=\mathrm{diag}(d_1 , \ldots, d_N) $. In this basis $\mathbb{E}\left[A_{\lambda}\right]$
turns out to be diagonal too. For each $ i = 1,\ldots, N$ we can isolate the contribution of the $i$-th row of $F$: by the Sherman-Morrison formula, we have $ (A_\lambda)_{i i} = \frac{d_{i}g_{i}}{1+d_{i}g_{i}} $, where  \[
g_i = \frac 1 P W_i^T ( F_{(i)}^T F_{(i)} + \lambda \mathrm{I}_P)^{-1} W_i,
\] with $ W_i $ denoting the $i$-th column of $W=\sqrt{P} F^T K^{-\frac{1}{2}} $ and $ F_{(i)} $ being obtained by removing the $i$-th row of $F$. The $ g_i $'s are all within $ \mathcal O (1 / \sqrt P) $ distance to the Stieltjes transform
\[
	m_{P}(-\lambda)=\frac{1}{P}\mathrm{Tr}\left(F^{T}F+\lambda \mathrm{I}_{P}\right)^{-1}.
\]
By a fixed point argument, the Stieltjes transform $m_{P}(-\lambda)$ is itself within $ \mathcal O (1/\sqrt P) $ distance to the deterministic value $\tilde{m}(-\lambda)$, where $ \tilde m $ is the unique positive solution to
\[
\gamma=\frac{1}{N}\sum_{i=1}^{N}\frac{d_{i}\tilde{m}(z)}{1+d_{i}\tilde{m}(z)}-\gamma z\tilde{m}(z).
\]
(The detailed proof in the Supp. Mat. uses non-asymptotic variants of arguments found in \cite{bai-2008}; the constants in the $ \mathcal O $ bounds are in particular made explicit).

As a consequence, from the above results, we obtain
\[
\EE \left[(A_{\lambda})_{ii}\right] = \mathbb E \left[\frac{d_{i}g_{i}}{1+d_{i}g_{i}}\right] \approx\frac{d_{i}\tilde{m}}{1+d_{i}\tilde{m}}=\frac{d_{i}}{\tilde{\lambda}+d_{i}},
\]
revealing the effective ridge $\tilde{\lambda}=\nicefrac{1}{\tilde{m}(-\lambda)}$.

This implies that $\mathbb{E}\left[A_{\lambda}\right]\approx K(X,X)(K(X,X)+\tilde{\lambda}\mathrm{I}_{N})^{-1}$
and
\[
\mathbb{E}\!\left[\hat{f}_{\lambda,\gamma}^{(RF)}(x)\right]\!\approx \!K(x,X)(K(X,X)+\tilde{\lambda}\mathrm{I}_{N})^{-1}y\!=\!\fkl(x),
\]
yielding the desired result.
%
\end{proof}
Note that asymptotic forms of equations similar to the ones in the above proof appear in different settings \cite{dobriban2018, mei-19, liu-20}, related to the study of the Stieltjes transform of the product of asymptotically free random matrices.

While the above theorem does not make assumptions on $ P, N $, and $ K $, the case of interest is when the right hand side $\frac{cK(x,x)\left\Vert y\right\Vert _{K^{-1}}}{P}$ is small. The constant $ c > 0 $ is uniformly bounded whenever $ \gamma $ and $ \lambda $ are bounded away from $ 0 $ and $ \frac 1 N \Tr K(X,X) $ is bounded from above. As a result, to bound the right hand side of \eqref{eq:average-rf-bound}, the two quantities we need to bound are $ T = \frac 1 N \Tr K(X,X) $ and $ \ynik $.
\begin{itemize}
\item
The boundedness of $ T $ is guaranteed for kernels that are translation-invariant, i.e. of the form $K(x,y)=k(\left\Vert x-y\right\Vert )$: in this case, one has $ T = k(0) $.
\item
If we assume $\EE_{\mathcal{D}}\left[K(x,x)\right]<\infty$ (as is commonly done in the literature \cite{rudi-17}), $ T $ converges to $ \mathbb{E}_{\mathcal{D}}\left[K(x,x)\right] $ as $ N \to \infty $ (assuming i.i.d. data points).
\item
For $ \ynik $, under the assumption that the labels are of the form $y_{i}=f^{*}(x_{i})$ for
a true regression function $ f^{*} $ lying in Reproducing Kernel Hilbert Space (RKHS) $ \mathcal{H} $ of the kernel $K$ \cite{Scholkopf-1998}, we have $ \ynik \leq \| f^{*} \|_{\mathcal H} $.
\end{itemize}
Our numerical experiments in Figure (\ref{fig:averageRF-KRR}) show excellent agreement between the test error of the expected $\lambda$-RF predictor and the one of the $\tilde \lambda$-KRR predictor suggesting that the two functions are indeed very close, even for small $N,P$.


Thanks to the implicit definition of the effective ridge $ \tilde \lambda $ (which depends on $ \lambda, \gamma, N $ and on the eigenvalues $ d_i $ of $ K(X, X) $) we obtain the following:
\begin{proposition}
\label{prop:fact-effective-ridge}The effective ridge $ \tilde{\lambda} $
satisfies the following properties:
\begin{enumerate}
\item for any $\gamma>0$, we have $ \lambda <  \tilde{\lambda}(\lambda,\gamma) \leq \lambda+\frac{1}{\gamma} T $;
\item the function $\gamma\mapsto\tilde{\lambda}(\lambda,\gamma)$ is decreasing;
\item for $\gamma>1$, we have $\tilde{\lambda}\leq \frac{\gamma}{\gamma-1} \lambda$;
\item for $\gamma<1$, we have $\tilde{\lambda}\geq \frac{1-\sqrt{\gamma}}{\sqrt{\gamma}} \min_{i} d_i $.
\end{enumerate}
\end{proposition}
The above proposition shows the implicit regularization effect of the RF model: sampling fewer features (i.e. decreasing $ \gamma $) increases the effective ridge $ \Lt $.

Furthermore, as $ \lambda \to 0 $ (ridgeless case), the effective ridge $ \tilde \lambda $ behave as follows:
\begin{itemize}
\item
in the overparameterized regime ($\gamma>1$), $\tilde \lambda $ goes to $0$;
\item
in the underparameterized regime ($\gamma<1$), $ \tilde \lambda $ goes to a limit $ \tilde \lambda_0 > 0 $.
\end{itemize}
These observations match the profile of $\tilde \lambda$ in Figure (\ref{fig:eff-ridge}).

\emph{Remark.}
When $ \lambda \searrow 0 $, the constant $ c $ in our bound \eqref{eq:average-rf-bound} explodes (see Supp. Mat.). As a result, this bound is not directly useful when $\lambda=0$. However, we know from Corollary \ref{cor:average-ridgless-overparametrized} that in the ridgeless overparametrized case ($\gamma>1$), the average RF predictor is equal to the ridgeless KRR predictor. In the underparametrized case ($\gamma<1$), our numerical experiments suggest that the ridgeless RF predictor is an excellent approximation of the $ \tilde \lambda_0 $-KRR predictor.

\begin{figure*}[t!]
    \centering
    \!\!\!
    \subfloat[Ridgeless vs. Ridge]{
        \includegraphics[width=0.33\textwidth]{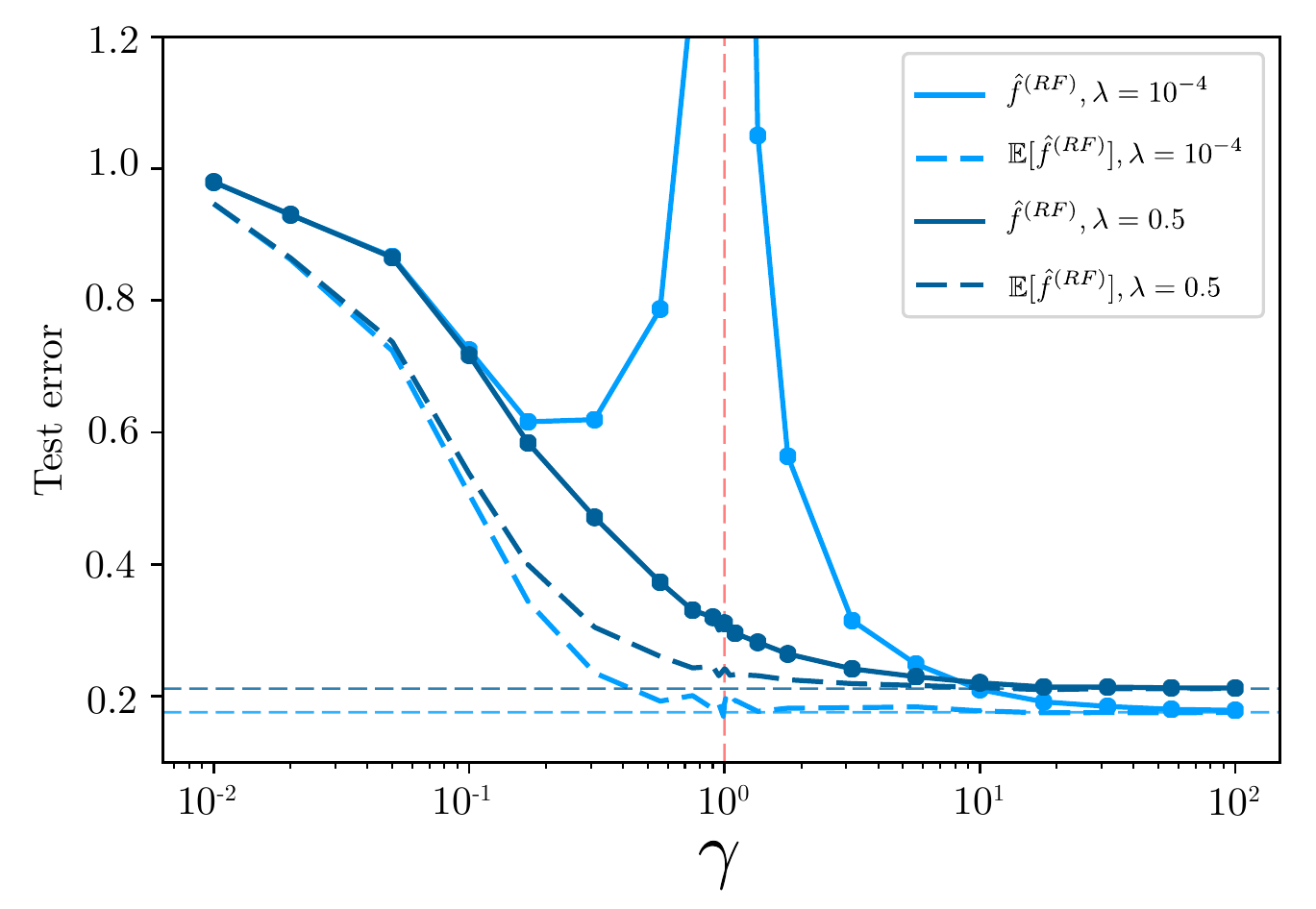}%
        \label{fig:DD-diff-N}
        } \!\!\!
    \subfloat[Variance of $\lambda$-RF]{
        \includegraphics[width=0.33\textwidth]{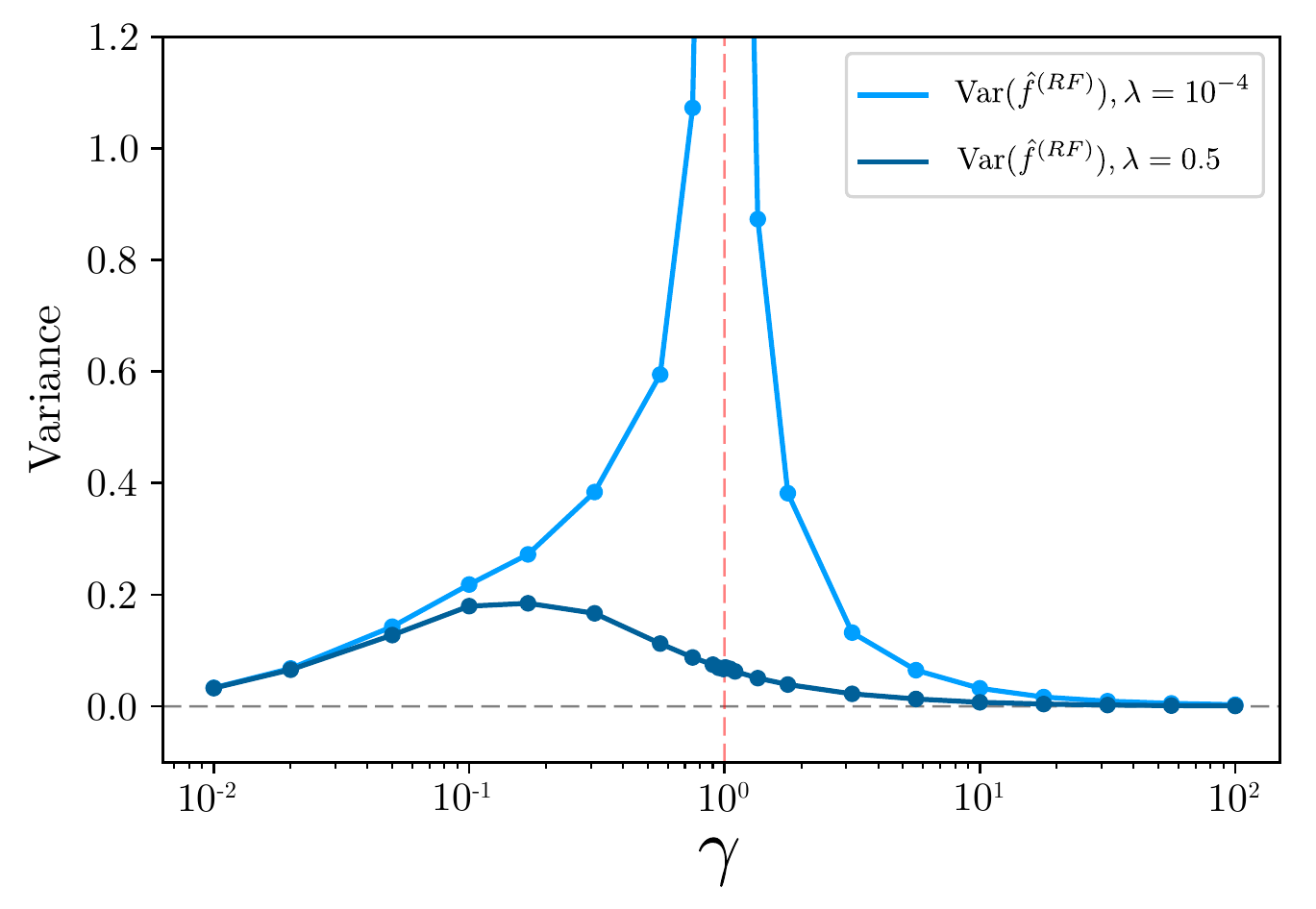}
        \label{fig:variance}
        } \!\!\!
    \subfloat[Evolution of $ \dLt $]{
        \includegraphics[width=0.33\textwidth]{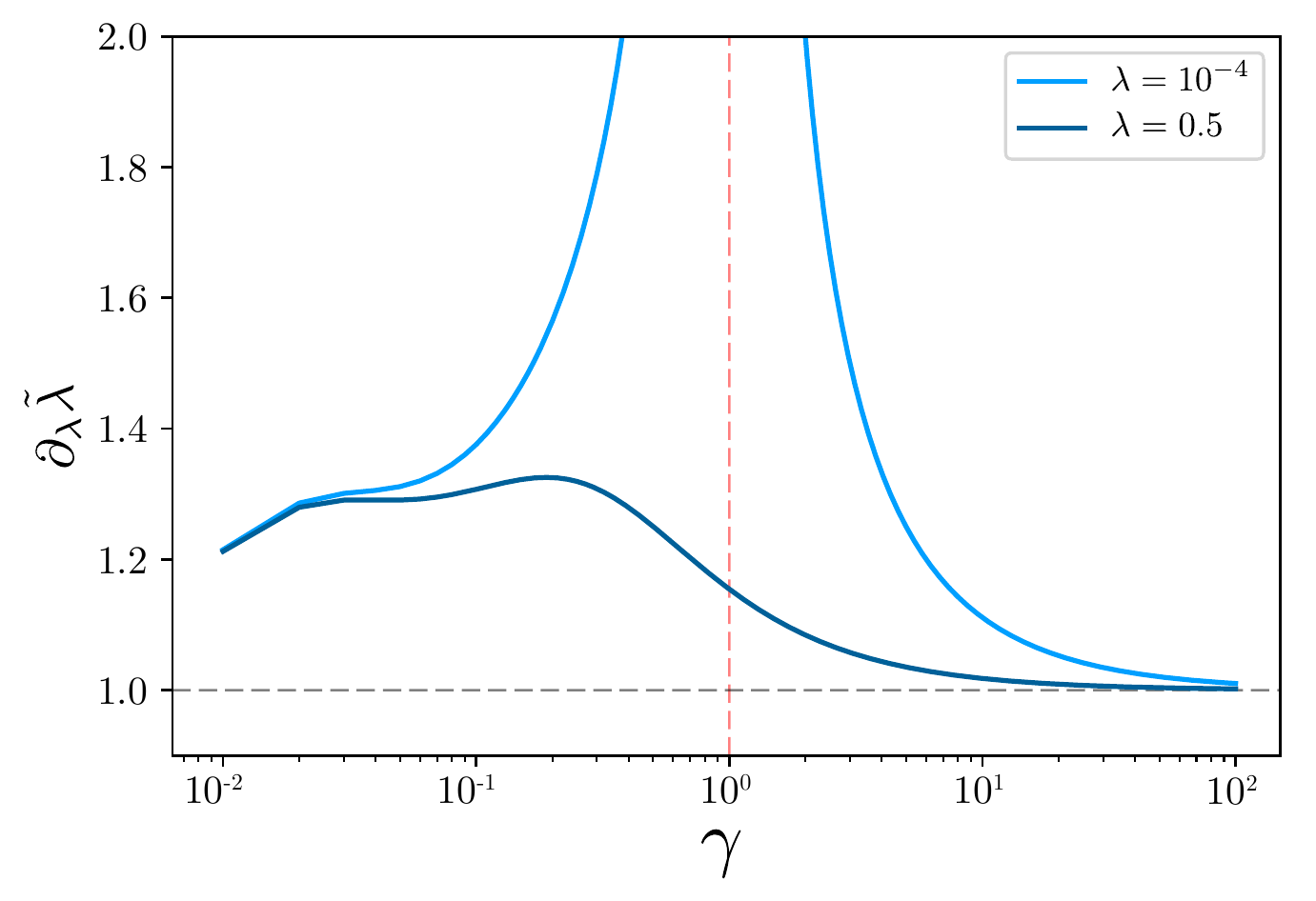}
        \label{fig:effective-ridge-der}
          } \!\!\!
    \caption{\textit{Average test error of the ridgeless vs. ridge $\lambda$-RF predictors.} In \textbf{(a)}, the average test errors of the ridgeless and the ridge RF predictors (solid lines) and the effect of ensembling (dashed lines) for $N=100$ MNIST data points. In \textbf{(b)}, the variance of the RF predictors and in \textbf{(c)}, the evolution of $ \dLt $ in the ridgeless and ridge cases. The experimental setup is the same as in Figure \ref{fig:average-RF-main}.}
    \label{fig:variance-RF-main}
\end{figure*}

\subsection{Effective Dimension}

The effective ridge $ \tilde \lambda $ is closely related to the so-called effective dimension appearing in statistical learning theory. For a linear (or kernel) model with ridge $ \lambda $, the \emph{effective dimension} $ \mathcal{N}(\lambda) \leq N $ is defined as $ \sum_{i=1}^N \frac{d_i}{\lambda+d_i} $ \cite{zhang-03, caponnetto-07}. It allows one to measure the effective complexity of the Hilbert space in the presence of a ridge.

For a given $ \lambda > 0 $, the effective ridge $ \tilde \lambda $ introduced in Theorem \ref{average-rf} is related to the effective dimension $ \mathcal N (\tilde \lambda) $ by
\[
	\mathcal{N}( \tilde \lambda ) = P \left( 1 - \frac{\lambda}{\tilde \lambda} \right).
\]
In particular, we have that $ \mathcal{N} ( \tilde \lambda ) \leq \min ( N, P ) $: this shows that the choice of a finite number of features corresponds to an automatic lowering of the effective dimension of the related kernel method.

Note that in the ridgeless underparameterized case ($ \lambda \searrow 0 $ and $ \gamma < 1 $), the effective dimension $ \mathcal{N}(\tilde{\lambda}) $ equals precisely the number of features $ P $.

\subsection{Risk of the Average Predictor}

A corollary of Theorem \ref{average-rf} is that the loss of the expected RF
predictor is close to the loss of the KRR predictor with ridge $\tilde{\lambda}$:
\begin{corollary}
\label{cor:difference_loss_expected_kernel_loss}If $\EE_{\mathcal{D}} [K(x,x) ]<\infty$,
we have that the difference of errors
$ \delta_E = \left|L (\mathbb{E} [ \frflg ] )-L (\hat{f}_{\tilde{\lambda}}^{(K)})\right| $
is bounded from above by
\[
\delta_E \leq
\frac{C\left\Vert y\right\Vert _{K^{-1}}}{P}\left(2\sqrt{L\left(\hat{f}_{\tilde{\lambda}}^{(K)}\right)}+\frac{C\left\Vert y\right\Vert _{K^{-1}}}{P}\right),
\]
where $ C $ is given by $ c \sqrt { \EE_{\mathcal{D}} [K(x,x) ] } $, with $ c $ the constant appearing in \eqref{eq:average-rf-bound} above.
\end{corollary}

As a result, $ \delta_E $ can be bounded in terms of $ \lambda, \gamma, T, \ynik $, which are discussed above, and of the kernel generalization error $L(f_{\tilde{\lambda}}^{(K)})$.
Such a generalization error can be controlled in a number of settings as $ N $ grows:
in \cite{caponnetto-07, marteau-19}, for instance, the loss is shown to vanish as $ N \to \infty $. Figure (\ref{fig:averageRF-KRR}) shows that the two test losses are indeed very close.


\section{Variance}
\label{sec:variance-RF}


In the previous sections, we analyzed the loss of the expected predictor $ \EE [ \frflg ] $.
In order to analyze the expected loss of the RF predictor $ \frflg $, it remains to control
the variance of the RF predictor: this follows from the bias-variance decomposition \[
 \EE \! \left[L(\frflg)\right]\!=\!L\left(\EE[\frflg]\right)+\EE_{\mathcal{D}} \! \left[ \mathrm{Var}(\frflg(x)) \right] \!,
\]
introduced in Section \ref{subsec:bias-variance}.

The variance $\mathrm{Var}\left(\frflg(x)\right)$ of the RF predictor can itself be written as the sum
\[
\mathrm{Var}\left(\mathbb{E}\left[\frflg(x)\mid F\right]\right)+\mathbb{E}\left[\mathrm{Var}\left(\frflg(x)\mid F\right)\right].
\]

By Proposition \ref{gaussian-mixture}, we have
\begin{align*}
\mathbb{E}\left[\frflg(x)\mid F\right] & =K(x,X)K(X,X)^{-1}\hat{y}\\
\mathrm{Var}\left(\frflg(x)\mid F\right) & =\frac{\| \hat{\theta}\| ^{2}}{P}\tilde{K}(x,x).
\end{align*}

\subsection{RF Predictor Concentration}

The following theorem allows us to bound both terms:
\begin{theorem}\label{variance-ridge}
There are constants $ c_1, c_2 > 0 $ depending on $ \lambda, \gamma, T $ only such that
\begin{align*}
& \mathrm{Var}\left(K(x,X)K(X,X)^{-1}\hat{y}\right) \leq\frac{c_1 K(x,x) \ynik^2}{P} \\
& \left|\mathbb{E}\| [\hat{\theta} \| ^{2}]-\dLt y^{T}M_{\Lt}y\right| \leq\frac{c_2 \ynik^{2}}{P},
\end{align*}
where $\dLt $ is the derivative of $\tilde{\lambda}$
with respect to $\lambda$ and for $ M_{\Lt} = K(X,X)(K(X,X)+\Lt
I_N )^{-2} $. As a result
\[
\mathrm{Var}\left(\frflg(x)\right)\leq\frac{c_3 K(x,x)\ynik^2}{P},
\]
where $ c_3 > 0 $ depends on $ \lambda, \gamma, T $.
\end{theorem}
Putting the pieces together, we obtain the following bound on the difference $ \Delta_E = | \EE [ L(\frflg) ] - L (\fkl) | $  between the expected RF loss and the KRR loss:

\begin{corollary}\label{cor:expected-loss-krr-loss}
If $\mathbb{E}_{\mathcal{D}}[K(x,x)]<\infty$,
we have
\[
\Delta_E \leq \frac{C_1 \ynik }{P}\left(\sqrt{L ( \fkl ) }+C_2 \ynik\right).
\]
where $C_1$ and $ C_2 $ depend on $\lambda,\gamma,T$ and $\mathbb{E}_{\mathcal{D}}[K(x,x)]$ only.
\end{corollary}

\subsection{Double Descent Curve}
\label{sec:double-descent-curve}
We now investigate the neighborhood of the frontier $ \gamma=1 $ between the under- and overparameterized regimes, known empirically to exhibit a double descent curve, where the test error explodes at $ \gamma=1 $ (i.e. when $ P \approx N $) as exhibited in Figure \ref{fig:variance-RF-main}.

Thanks to Theorem \ref{variance-ridge}, we get a lower bound on the variance of $ \frflg $:
\begin{corollary}\label{lower-bound-variance}
There exists $ c_4 > 0 $ depending on $ \lambda, \gamma, T $ only such that $ \mathrm{Var} (\frflg(x)) $ is bounded from below by
\[
  \dLt \frac{y^T M_{\Lt} y }{P}\tilde{K} (x,x)
	 - \frac{c_4 K(x,x) \ynik^2}{P^2}.
\]
\end{corollary}

If we assume the second term of Corollary \ref{lower-bound-variance} to be negligible, then the only term which depends on $P$ is $ \dLt \frac{y^T M_{\Lt} y }{P}$. The derivative $ \dLt  $ has an interesting behavior as a function of $ \lambda $ and $ \gamma $:
\begin{proposition} \label{prop:fact-effective-ridge-derivative}For $ \gamma > 1 $, as $ \lambda \to 0 $, the derivative $ \dLt $ converges to $ \frac{\gamma}{\gamma-1} $. As $ \lambda \gamma \to \infty $, we have $ \dLt (\lambda, \gamma) \to 1 $.

\end{proposition}
The explosion of $\dLt$ in $(\gamma=1, \lambda=0)$ is displayed in Figure (\ref{fig:effective-ridge-der}).

Corollary \ref{lower-bound-variance} can be used to explain the double-descent curve numerically observed for small $ \lambda > 0 $. It is natural to assume that in this case $ \dLt \gg 1 $ around $ \gamma = 1 $, dominating the lower bound in Corollary \ref{lower-bound-variance}. In turn, by Proposition \ref{prop:fact-effective-ridge-derivative} this implies that the variance of $ \frf $ gets large. Finally, by the bias-variance decomposition, we obtain a sharp increase of the test error around $ \gamma = 1 $, which is in line with the results of \cite{hastie-19, mei-19}.


\section{Conclusion}
\label{sec:conclusion}

In this paper, we have identified the implicit regularization arising
from the finite sampling of Random Features (RF): using a Gaussian
RF model with ridge parameter $\lambda>0$ ($\lambda$-RF) and feature-to-datapoints
ratio $\gamma=\frac{P}{N}$ is essentially equivalent to using a Kernel
Ridge Regression with effective ridge $\tilde{\lambda}>\lambda$ ($\tilde{\lambda}$-KRR) which we characterize explicitly.
More precisely, we have shown the following:
\begin{itemize}
\item The expectation of the $\lambda$-RF predictor is very close to the
$\tilde{\lambda}$-KRR predictor (Theorem \ref{average-rf}).
\item The $\lambda$-RF predictor concentrates around its expectation when $\lambda$
is bounded away from zero (Theorem \ref{variance-ridge}); this implies in particular
that the test errors of the $\lambda$-RF and $\tilde{\lambda}$-KRR
predictors are close to each other (Corollary \ref{cor:expected-loss-krr-loss}).
\end{itemize} Both theorems are proven using tools from random matrix theory, in
particular finite-size results on the concentration of the Stieltjes
transform of general Wishart matrix models. While our current
proofs require the assumption that the RF model is Gaussian, it seems
natural to postulate that the results and the proofs extend to more
general setups, along the lines of \cite{louart-17,benigni-2019}.

Our numerical verifications on the expected $\lambda$-RF predictor and the $\tilde{\lambda}$-KRR predictor have shown that
both are in excellent agreement. This shows in particular that in order to use RF predictors to approximate KRR predictors with a given ridge, one should choose both the number of features and the explicit ridge appropriately.

Finally, we investigate the ridgeless limit case $\lambda\searrow0$.
In this case, we see a sharp transition at $\gamma=1$: in the overparameterized regime $ \gamma > 1 $,
the effective ridge goes to zero, while in the underparameterized regime $ \gamma < 1 $,
it converges to a positive value. At the interpolation threshold $\gamma=1$, the variance of the $\lambda$-RF  explodes, leading to the double descent
curve emphasized in \cite{advani-17, spigler-18, belkin-18, nakkiran-19}. We investigate this numerically and prove a lower bound yielding a plausible explanation for this phenomenon.

\begin{figure}[t!]
  \centering
  \subfloat[$N=100$ vs. $N=1000$]{
      \includegraphics[width=0.45\textwidth]{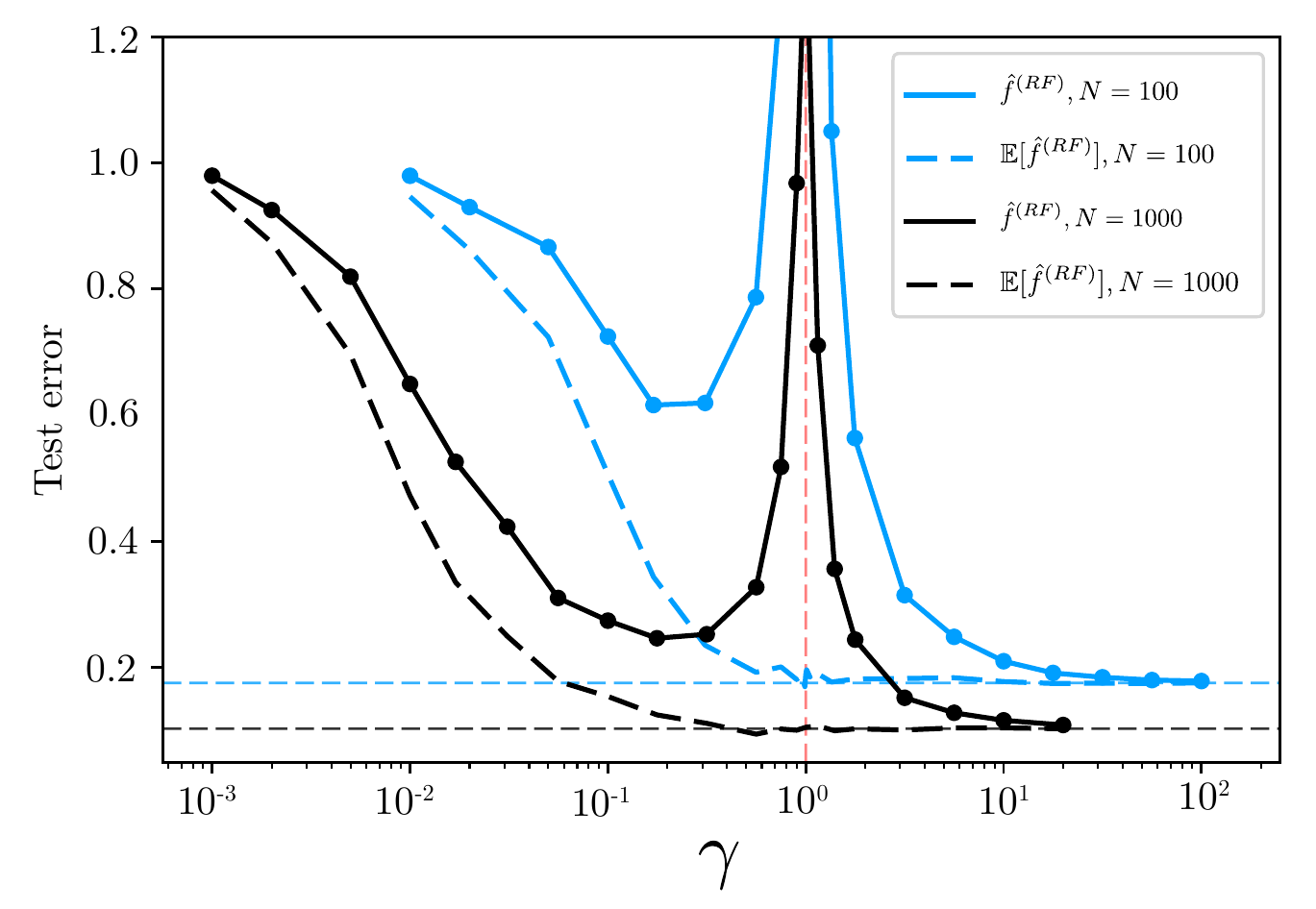}%
        }
  \caption{\textit{Average test error of the $\lambda$-RF predictor for two values of $N$ and $\lambda=10^{-4}$.} For $N=1000$, the test error is naturally lower and the cusp at $\gamma=1$ is narrower than for $N=100$. The experimental setup is the same as in Figure \ref{fig:average-RF-main}. \label{fig:DD-diff-ridge}}
\end{figure}

\section*{Thanks and Acknowledgements}
The authors would like to thank Andrea Montanari, Song Mei, L\'ena\"ic Chizat and Alessandro Rudi for the helpful discussions. Cl\'ement Hongler acknowledges support from the ERC SG CONSTAMIS grant, the NCCR SwissMAP grant, the Minerva Foundation, the Blavatnik Family Foundation, and the Latsis foundation.

\vfill

\pagebreak

\bibliographystyle{icml2020}
\bibliography{references}

\onecolumn

\icmltitle{Supplementary Material for \\
           Implicit Regularization of Random Feature Models}

\appendix

We organize the Supplementary Material (Supp. Mat.) as follows:
\begin{itemize}
  \item In Section \ref{sec:exp-details}, we present the details for the numerical results presented in the main text (and in the Supp. Mat.).
  \item In Section \ref{sec:additional-exps}, we present additional experiments and some discussions.
  \item In Section \ref{sec:proofs}, we present the proofs of the mathematical results presented in the main text.
\end{itemize}

\section{Experimental Details}
\label{sec:exp-details}

The experimental setting consists of $N$ training and $N_\text{tst}$ test datapoints $\{(x_i, y_i)\}_{i=1}^{N + N_\text{tst}} \in \mathbb R^{d} \times \mathbb R$. We sample $P$ Gaussian features $f^{(1)}, \ldots, f^{(P)}$ of $N + N_\text{tst}$ dimension with zero mean and covariance matrix entries thereof $C_{i,j}=K(x_i, x_j)$ where $K(x, x')=\exp(-\|x-x'\|^{2} / \ell)$ is a Radial Basis Function (RBF) Kernel with lengthscale $\ell$.
The extended data matrix $\bar{F} = \frac{1}{\sqrt{P}}[f^{(1)}, \ldots, f^{(P)}]$ of size $(N + N_\text{tst}) \times P$ is decomposed into two matrices: the (training) data matrix $F = \bar{F}_{[:N, :]}$ of size $N \times P$, and a test data matrix $F_\text{tst} = \bar{F}_{[N:, :]}$ of size $N_\text{tst} \times P$ so that $\bar F=[F;F_\text{tst}]$.
For a given ridge $\lambda$, we compute the optimal solution using the data matrix $F$, i.e. $\hat{\theta}=F^{T}\left(FF^{T}+\lambda \mathrm{I}_N\right)^{-1}y$ and obtain the predictions on the test datapoints $\hat{y}_{\text{tst}} = F_\text{tst} F^{T}\left(FF^{T}+\lambda \mathrm{I}_N\right)^{-1}y$.

Using the procedure above, we performed the following experiments:

\subsection{Experiments with Sinusoidal data}
We consider a dataset of $N=4$ training datapoints $(x_{i},\sin(x_{i}))\in[0,2\pi)\times[-1,1]$ and $N_\text{tst} = 100$ equally spaced test data points in the interval $[0, 2\pi)$.
In this experiment, the lengthscale of the RBF Kernel is $\ell = 2$.
We compute the average and standard deviation the $\lambda$-RF predictor using 500 samplings of $\bar F$ (see Figure 1 in the main text and Figure \ref{fig:RF-predictor-regimes} in the Supp. Mat.).

\subsection{MNIST experiments}
We sample $N=100$ and $N_\text{tst}=100$ images of digits $7$ and $9$ from the MNIST dataset (image size $d=24\times24$, edge pixels cropped, all pixels rescaled down to $[0,1]$ and recentered around the mean value) and label each of them with $+1$ and $-1$ labels, respectively.
In this experiment, the lengthscale of the RBF Kernel is $ \ell = d \ell_0$ where $\ell_0 = 0.2$.
We approximate the expected $\lambda$-RF predictor on the test datapoints using the average of $\hat y_\text{tst}$ over $50$ instances of $\bar{F}$ and compute the MSE (see Figures 2, 3 in the main text; in the ridgeless case --$\lambda = 10^{-4}$ in our experiments--  when $P$ is close to $N$, the average is over $500$ instances).
In Figure 4 of the main text, using $N_\text{tst}=100$ test points, we compare two predictors trained over $N = 100$ and $N=1000$ training datapoints.

\subsection{Random Fourier Features}
\label{sec:RFF-exp}
We sample random Fourier Features corresponding to the RBF Kernel with lengthscale $ \ell = d \ell_0$ where $\ell_0 = 0.2$ (same as above) and consider the same dataset as in the MNIST experiment. The extended data matrix $\bar{F}$ for Fourier features is obtained as follows: we sample $d$-dimensional i.i.d. centered Gaussians $w^{(1)}, \ldots, w^{(P)}$ with standard deviation $\sqrt{2 / \ell}$, sample $b^{(1)}, \ldots, b^{(P)}$ uniformly in $[0, 2 \pi)$, and define $\bar{F}_{i,j} = \sqrt{ \frac{2}{P}} \cos(x^{T}_iw^{(j)} + b^{(j)})$.
We approximate the expected Fourier Features predictor on the test datapoints using the average of $\hat y_\text{tst}$ over $50$ instances of $\bar F$ (see Figure \ref{fig:RFF}).

\clearpage

\section{Additional Experiments}
\label{sec:additional-exps}

We present the following complementary simulations:
\begin{itemize}
  \item In Section \ref{sec:B1}, we present the distribution of the $\lambda$-RF predictor for the selected $P$ and $\lambda$.
  \item In Section \ref{sec:B2}, we present the evolution of $\tilde{\lambda}$ and its derivative $\partial_\lambda\tilde\lambda$ for different eigenvalue spectra.
  \item In Section \ref{sec:B3}, we show the evolution of the eigenvalue spectrum of $\mathbb E [A_{\lambda}]$.
  \item In Section \ref{sec:RFF}, we present numerical experiments on MNIST using random Fourier features.
\end{itemize}

\vskip0.5cm

\subsection{Distribution of the RF predictor}
\label{sec:B1}
\begin{figure}[h]
    \centering
    \subfloat[$P=2, \lambda=0$]{
        \includegraphics[width=0.23\textwidth]{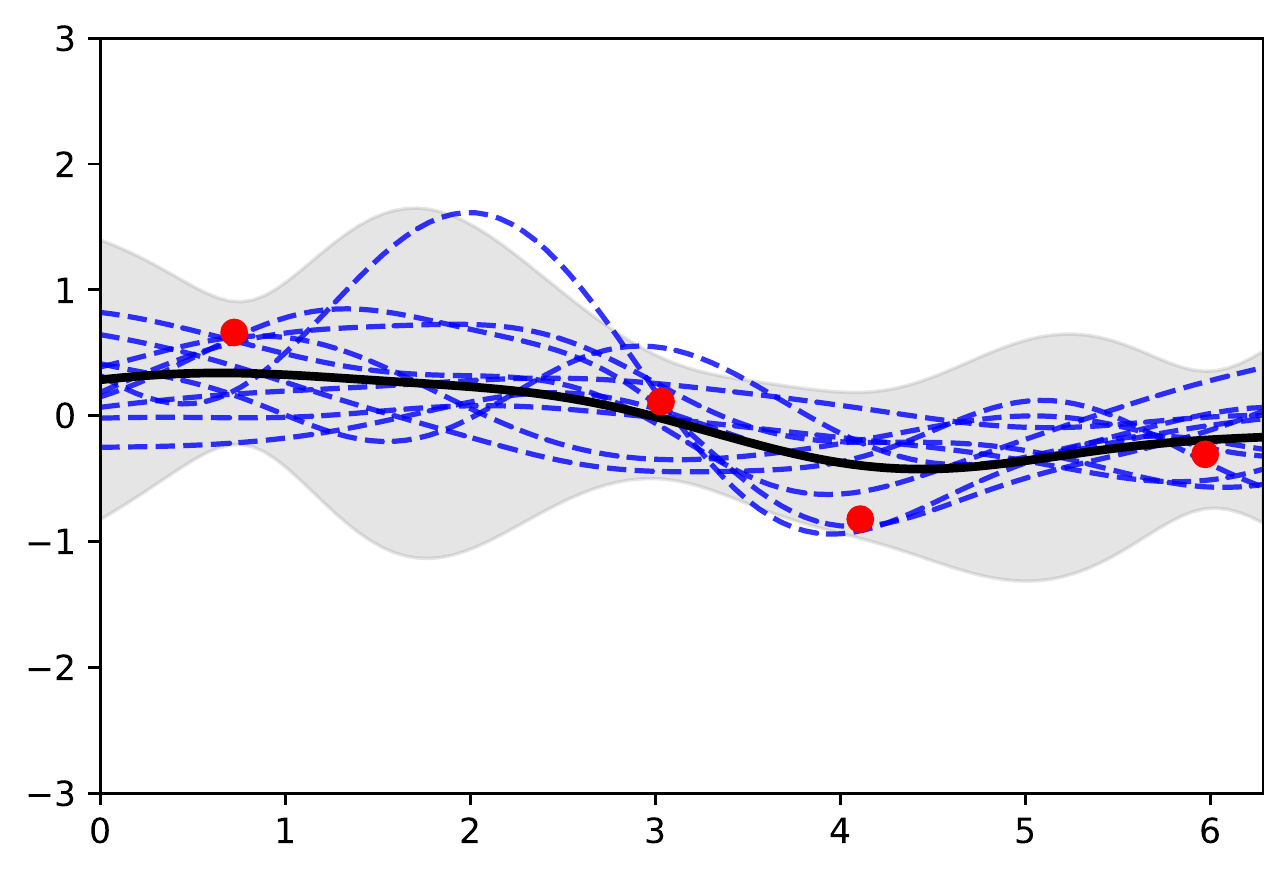}
        \label{fig:RF-A1}
        } \hfill
    \subfloat[$P=4, \lambda=0$]{
        \includegraphics[width=0.23\textwidth]{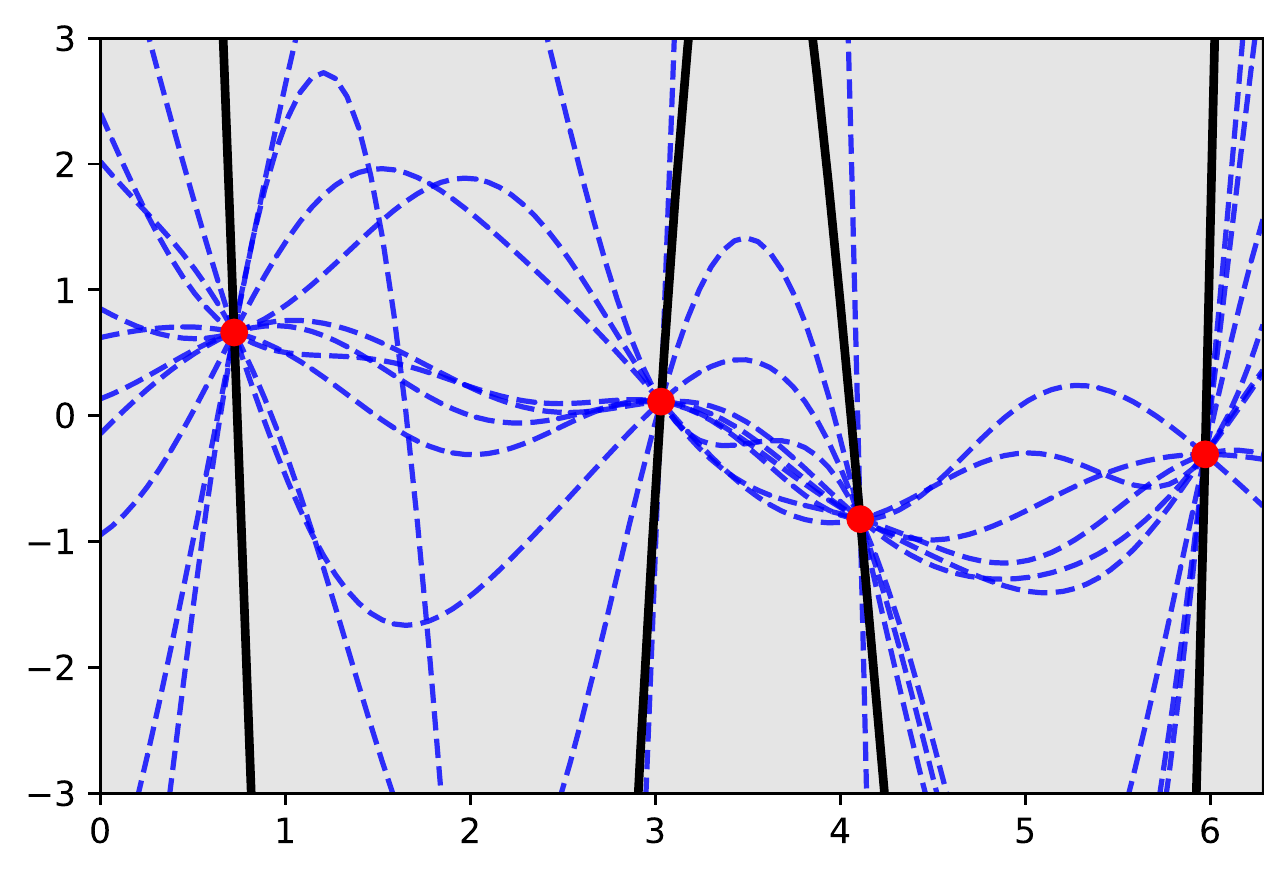}
        \label{fig:RF-A2}
        } \hfill
    \subfloat[$P=10, \lambda=0$]{
        \includegraphics[width=0.23\textwidth]{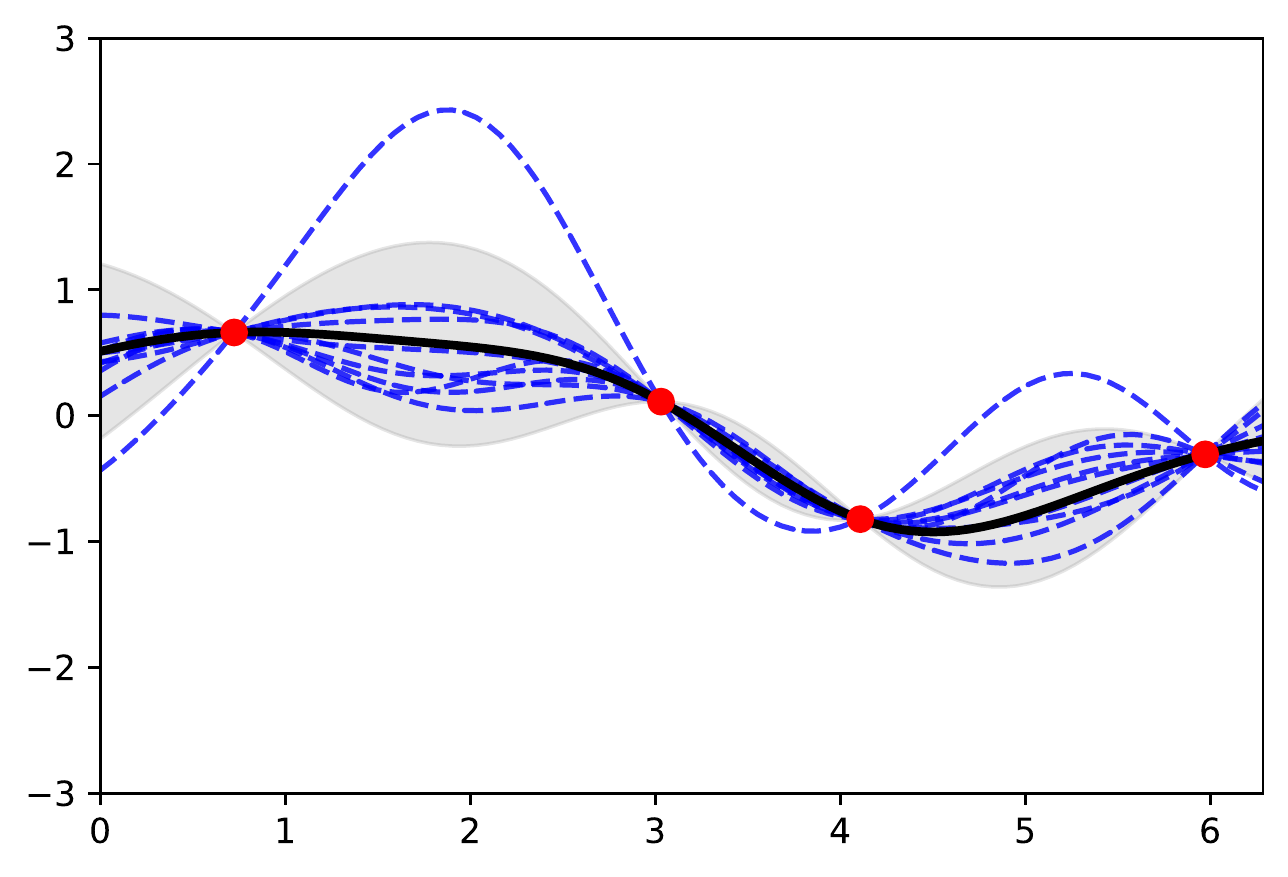}
        \label{fig:RF-A3}
        } \hfill
    \subfloat[$P=100, \lambda=0$]{
        \includegraphics[width=0.23\textwidth]{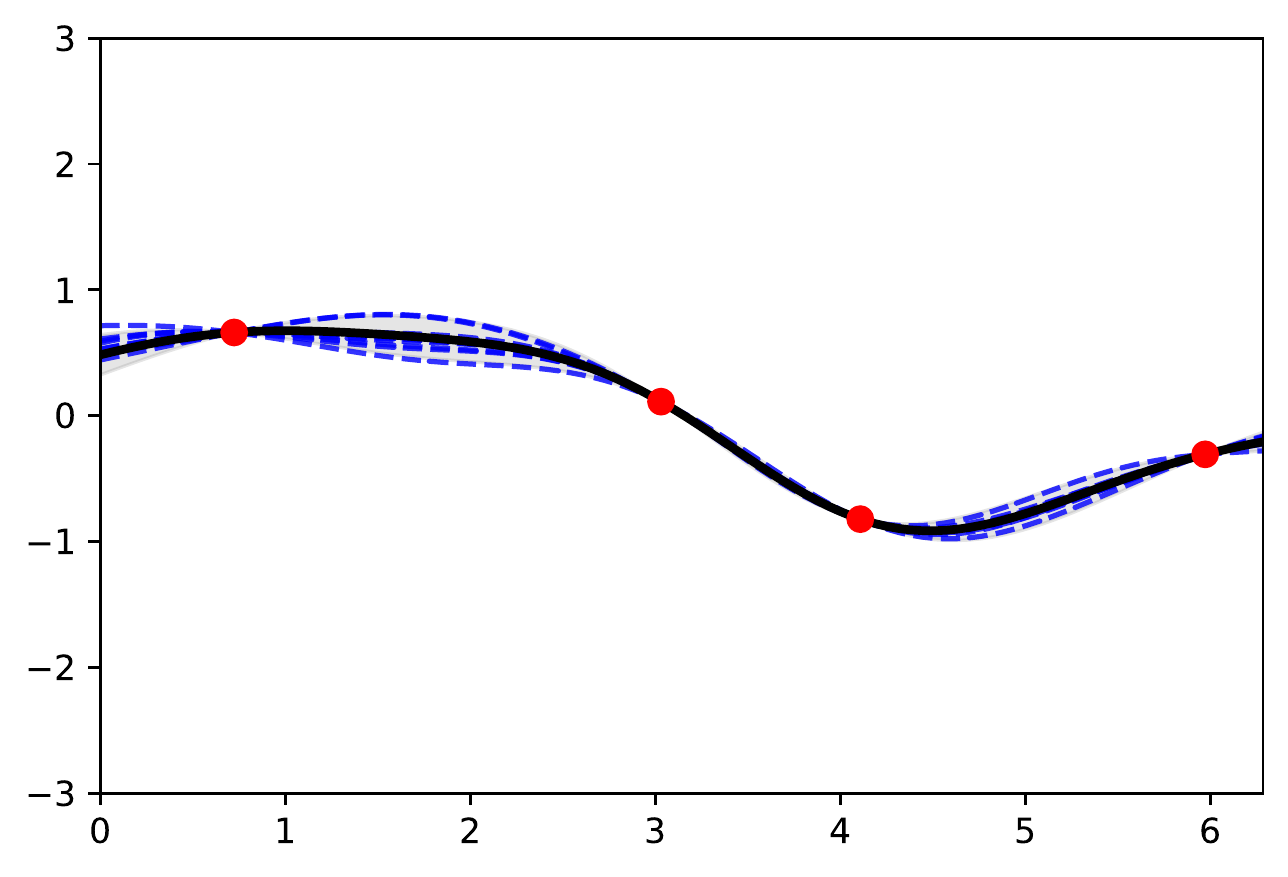}
        \label{fig:RF-A4}
        } \vfill
    \subfloat[$P=2, \lambda=10^{-4}$]{
          \includegraphics[width=0.23\textwidth]{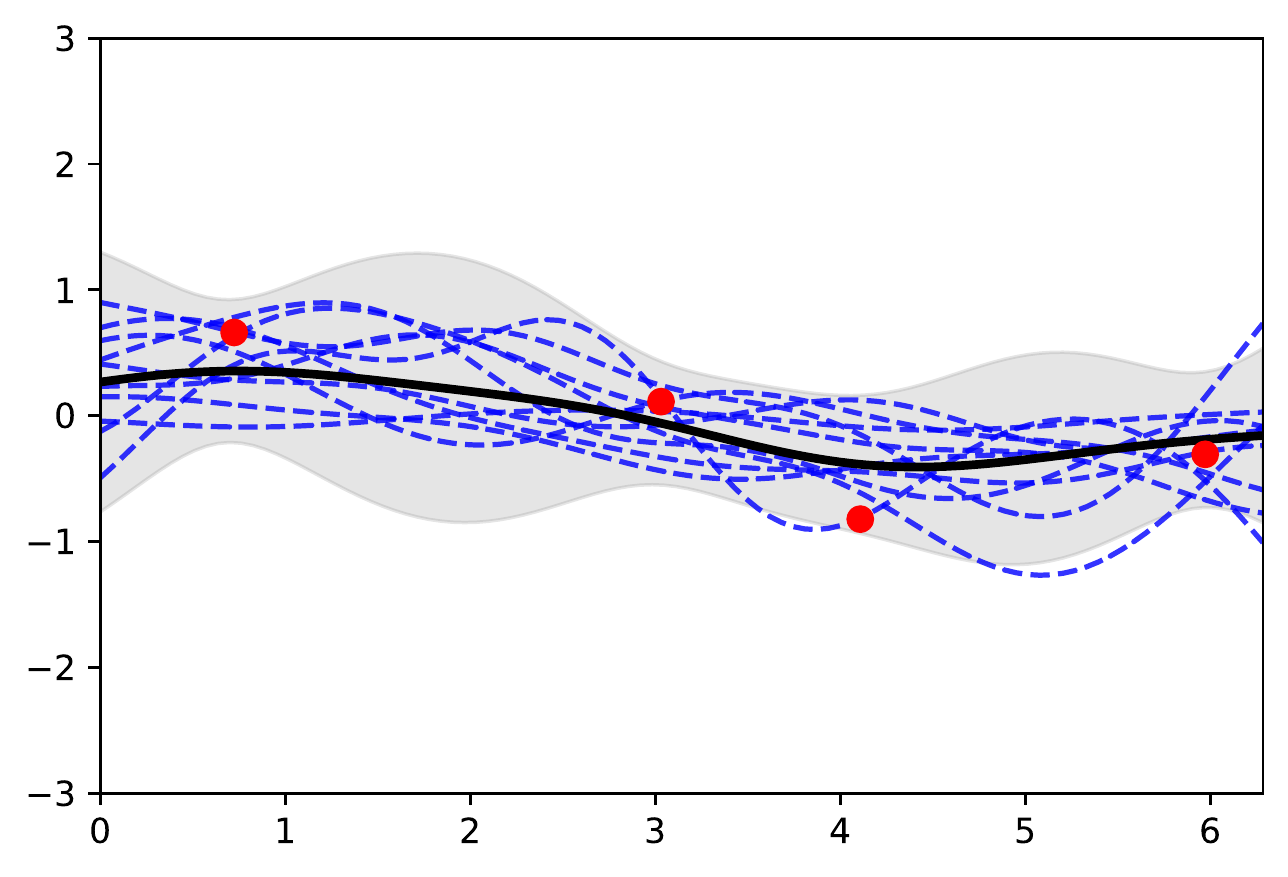}
          \label{fig:RF-A5}
        } \hfill
    \subfloat[$P=4, \lambda=10^{-4}$]{
          \includegraphics[width=0.23\textwidth]{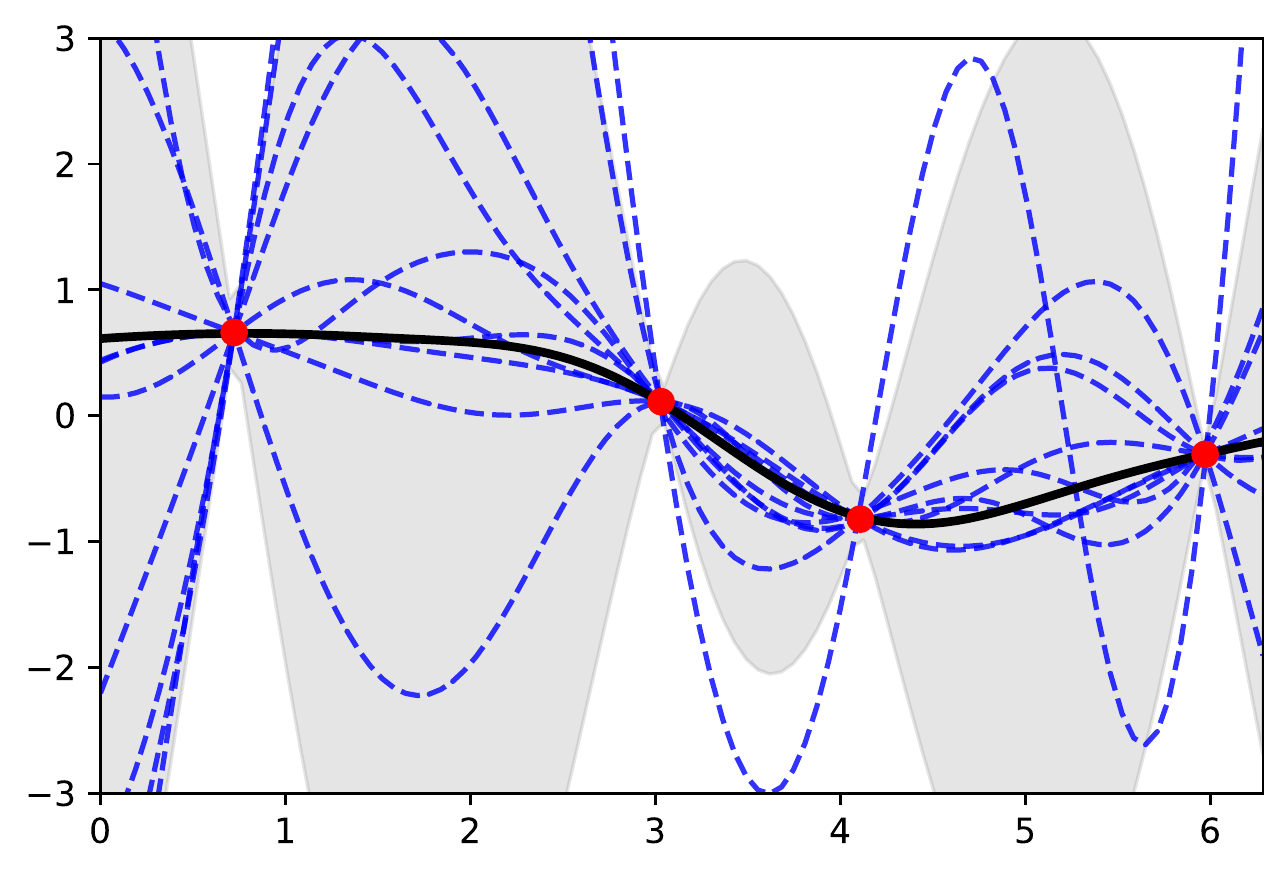}
          \label{fig:RF-A6}
        } \hfill
    \subfloat[$P=10, \lambda=10^{-4}$]{
          \includegraphics[width=0.23\textwidth]{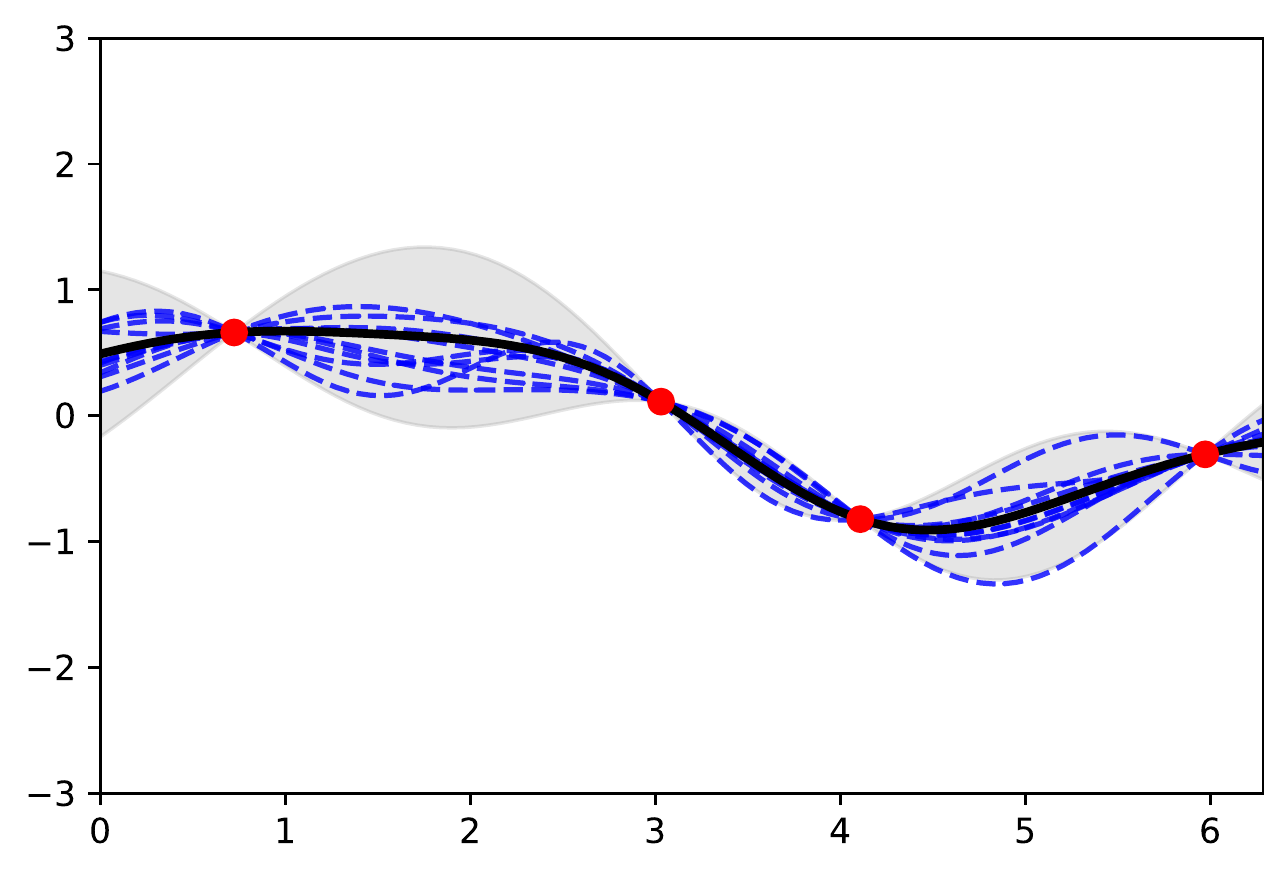}
          \label{fig:RF-A7}
        } \hfill
    \subfloat[$P=100, \lambda=10^{-4}$]{
          \includegraphics[width=0.23\textwidth]{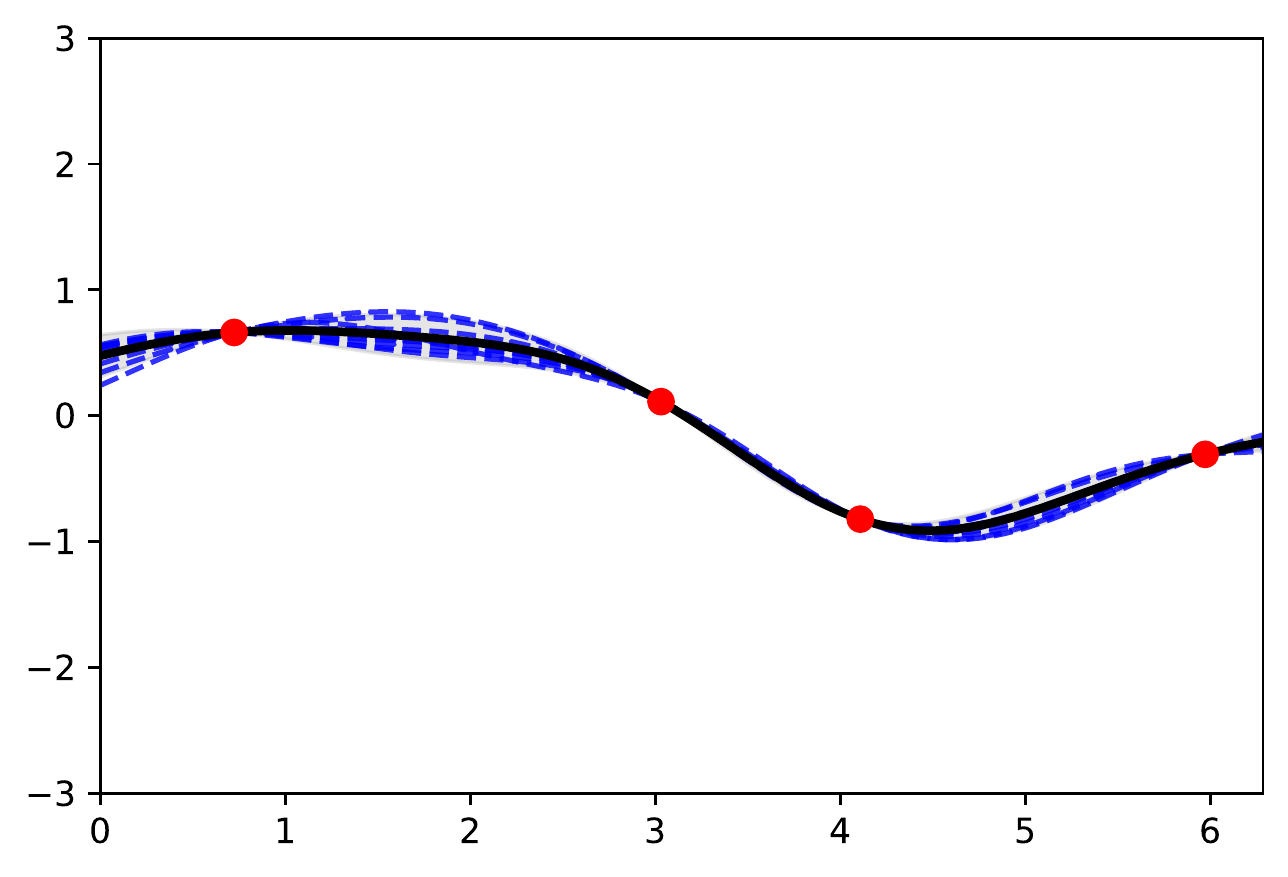}
          \label{fig:RF-A8}
        } \vfill
    \subfloat[$P=2, \lambda=10^{-1}$]{
          \includegraphics[width=0.23\textwidth]{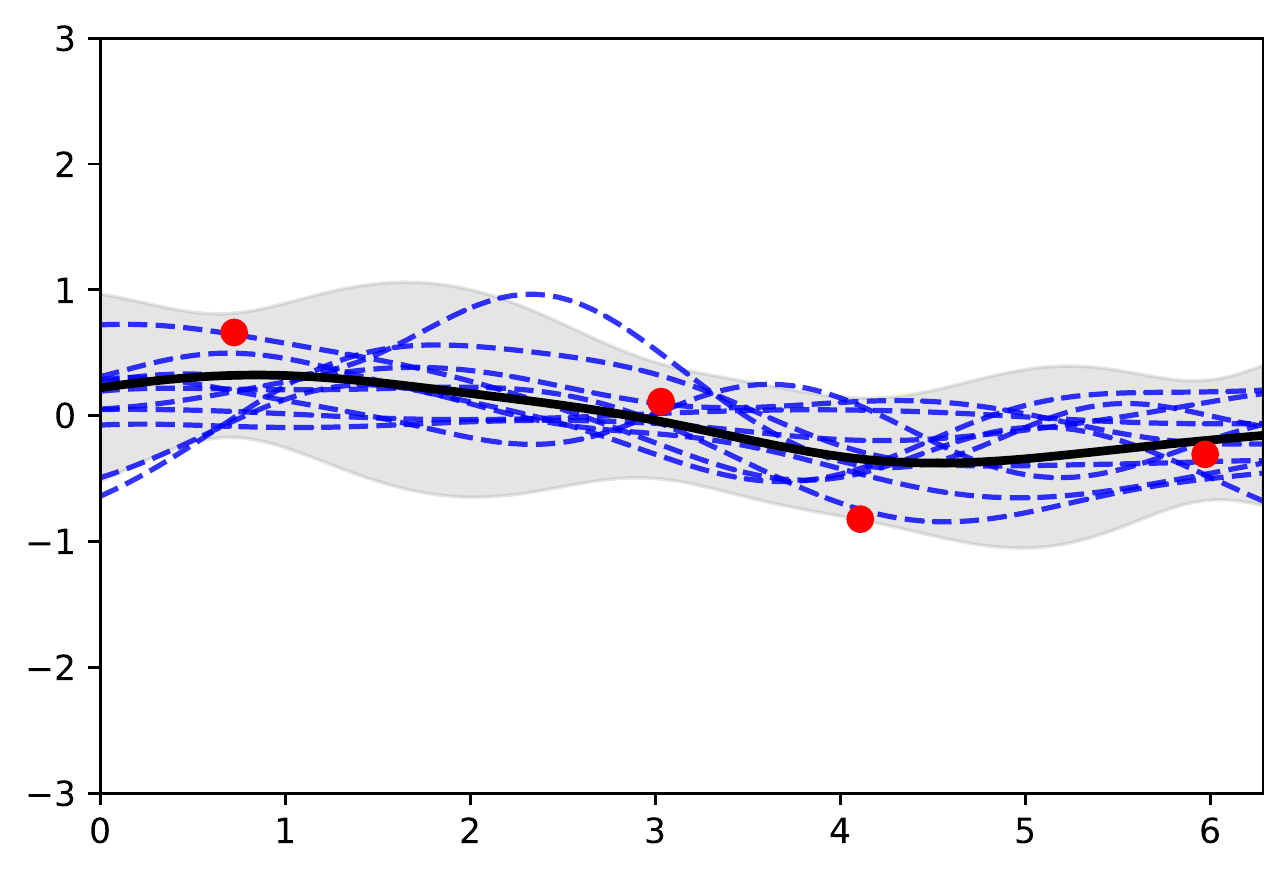}
          \label{fig:RF-A9}
        } \hfill
    \subfloat[$P=4, \lambda=10^{-1}$]{
          \includegraphics[width=0.23\textwidth]{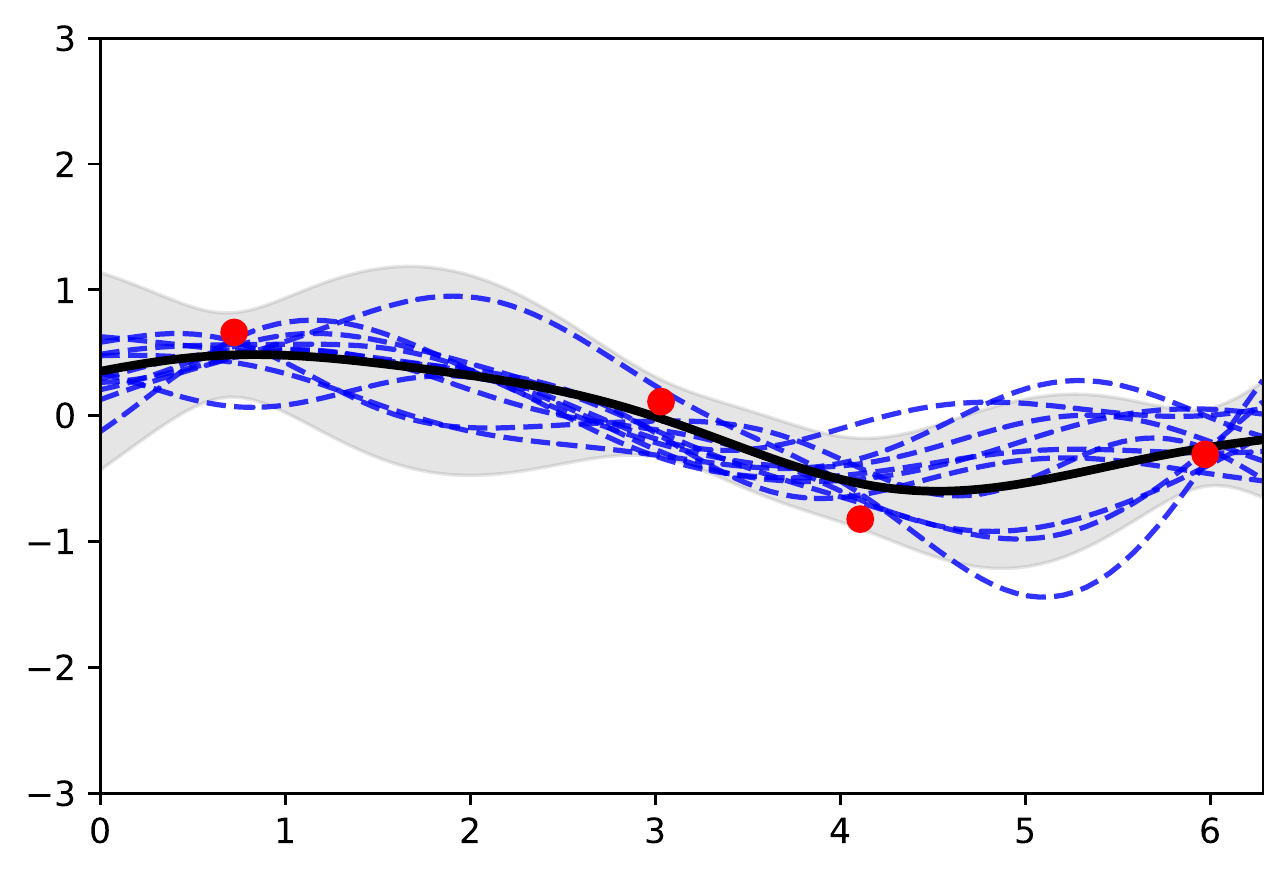}
          \label{fig:RF-A10}
        } \hfill
    \subfloat[$P=10, \lambda=10^{-1}$]{
          \includegraphics[width=0.23\textwidth]{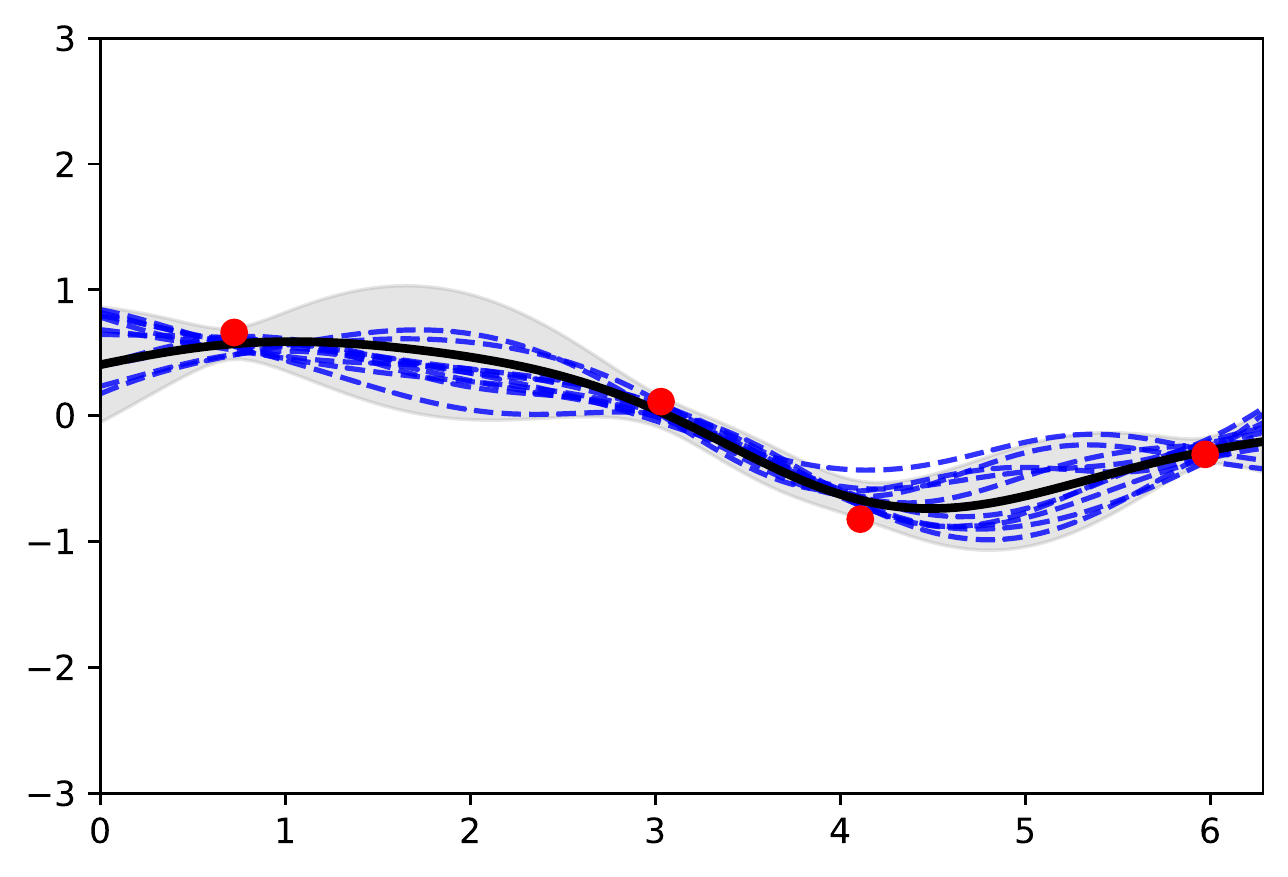}
          \label{fig:RF-A11}
        } \hfill
    \subfloat[$P=100, \lambda=10^{-1}$]{
          \includegraphics[width=0.23\textwidth]{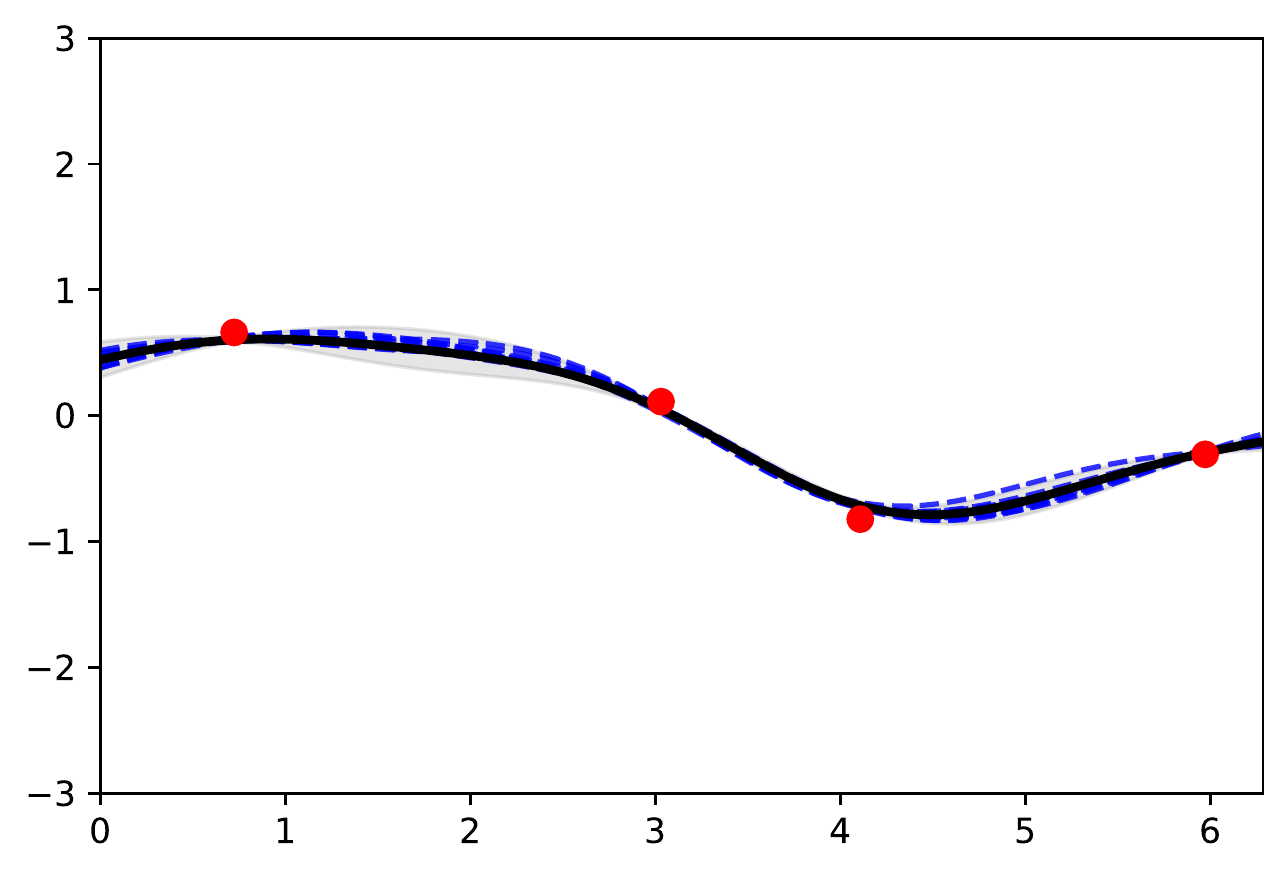}
          \label{fig:RF-A12}
        } \vfill
    \subfloat[$P=2, \lambda=1$]{
          \includegraphics[width=0.23\textwidth]{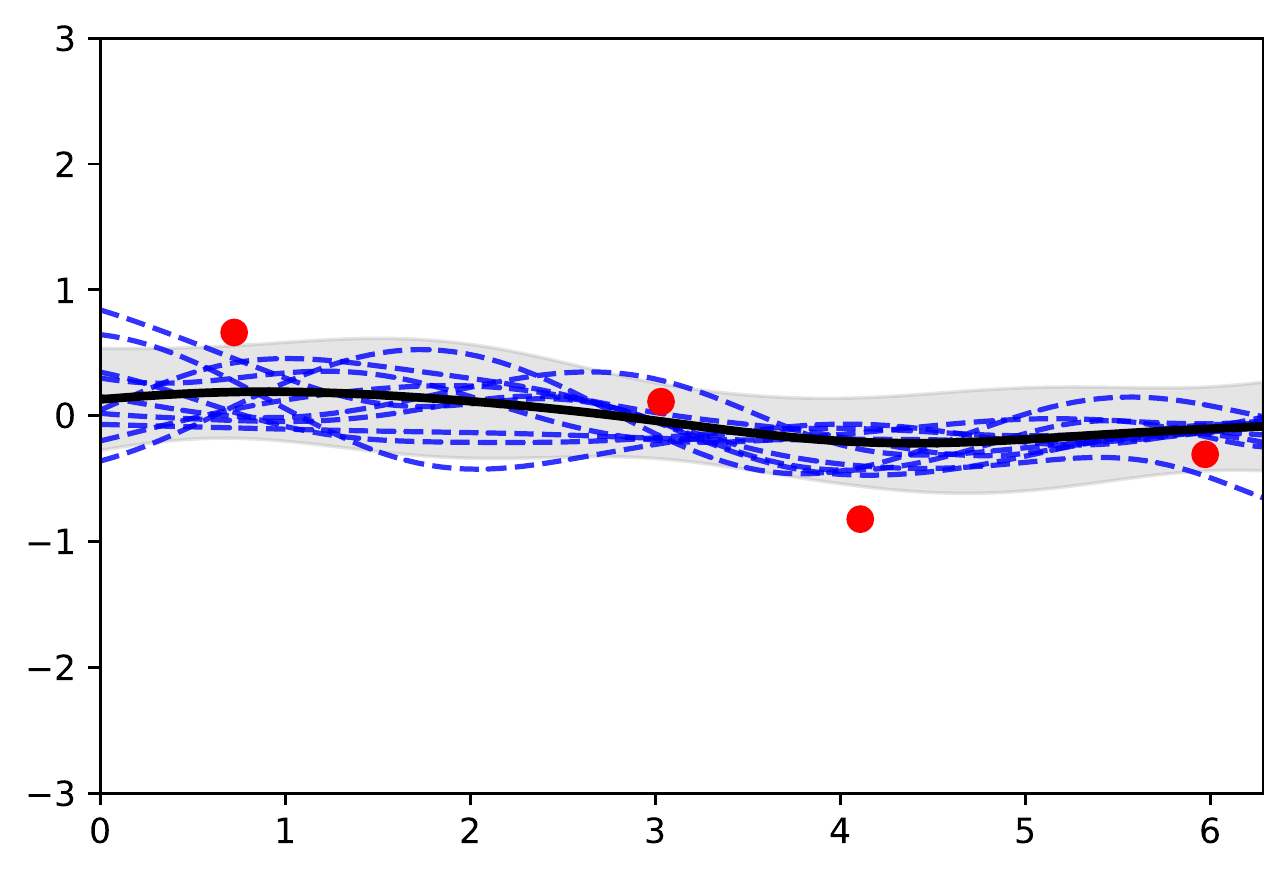}
          \label{fig:RF-A13}
        } \hfill
    \subfloat[$P=4, \lambda=1$]{
          \includegraphics[width=0.23\textwidth]{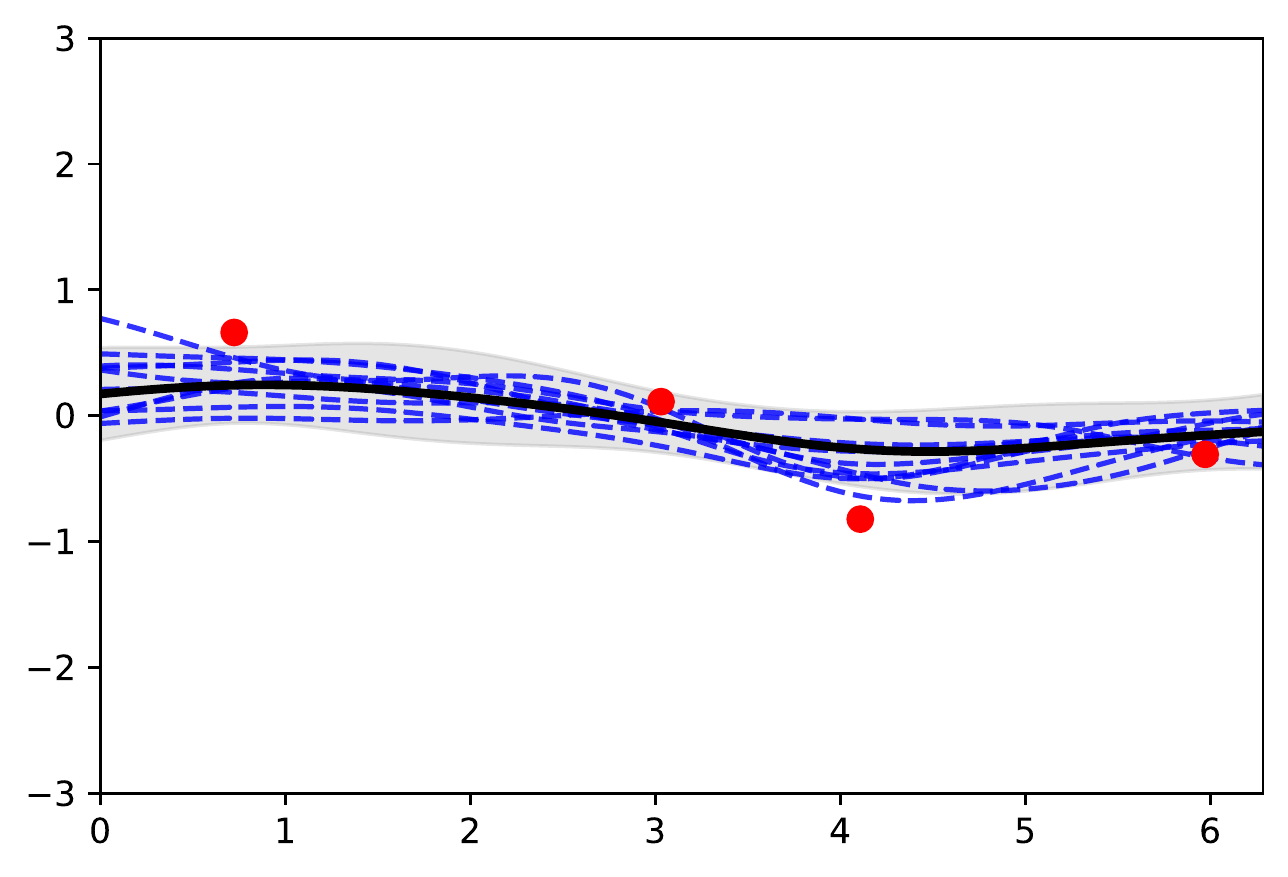}
          \label{fig:RF-A14}
        } \hfill
    \subfloat[$P=10, \lambda=1$]{
          \includegraphics[width=0.23\textwidth]{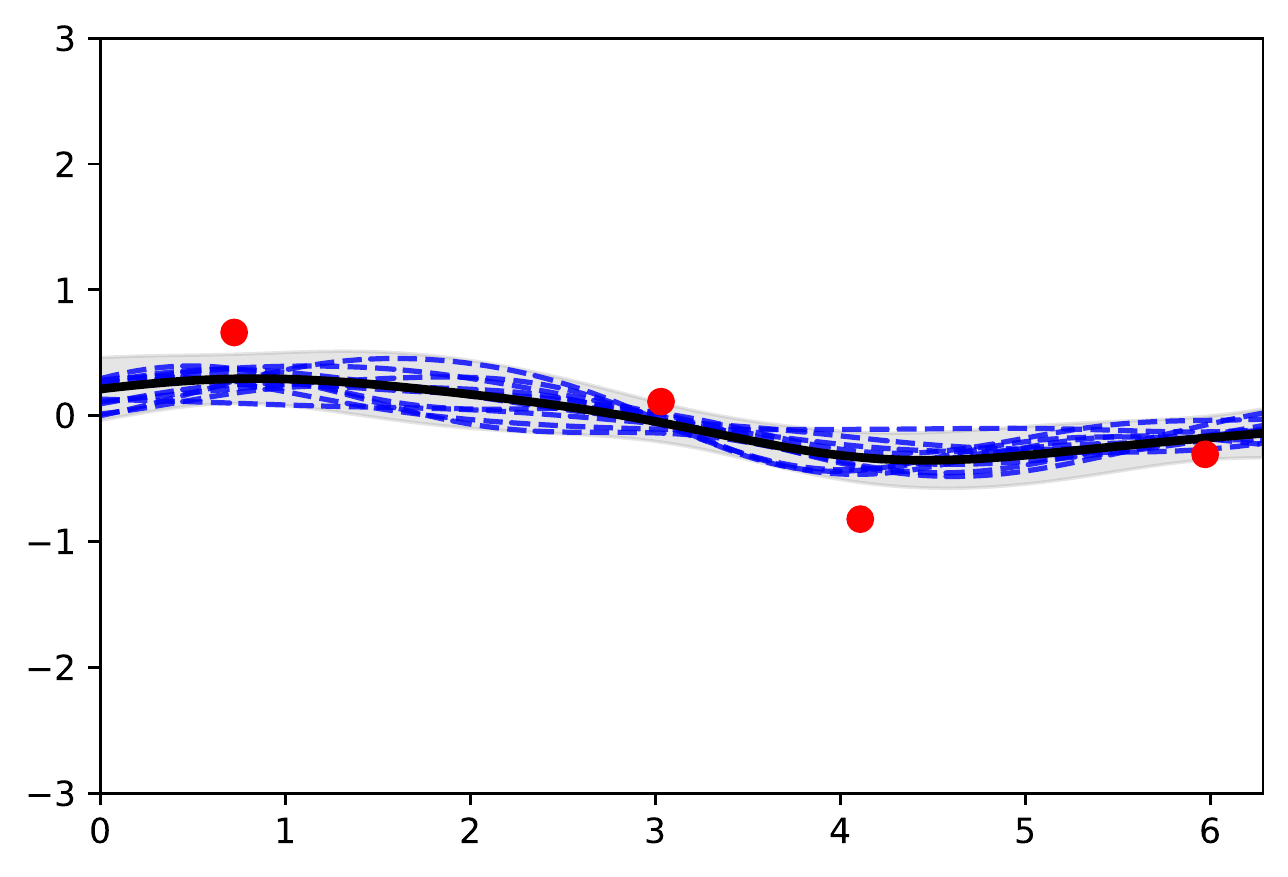}
          \label{fig:RF-A15}
        } \hfill
    \subfloat[$P=100, \lambda=1$]{
          \includegraphics[width=0.23\textwidth]{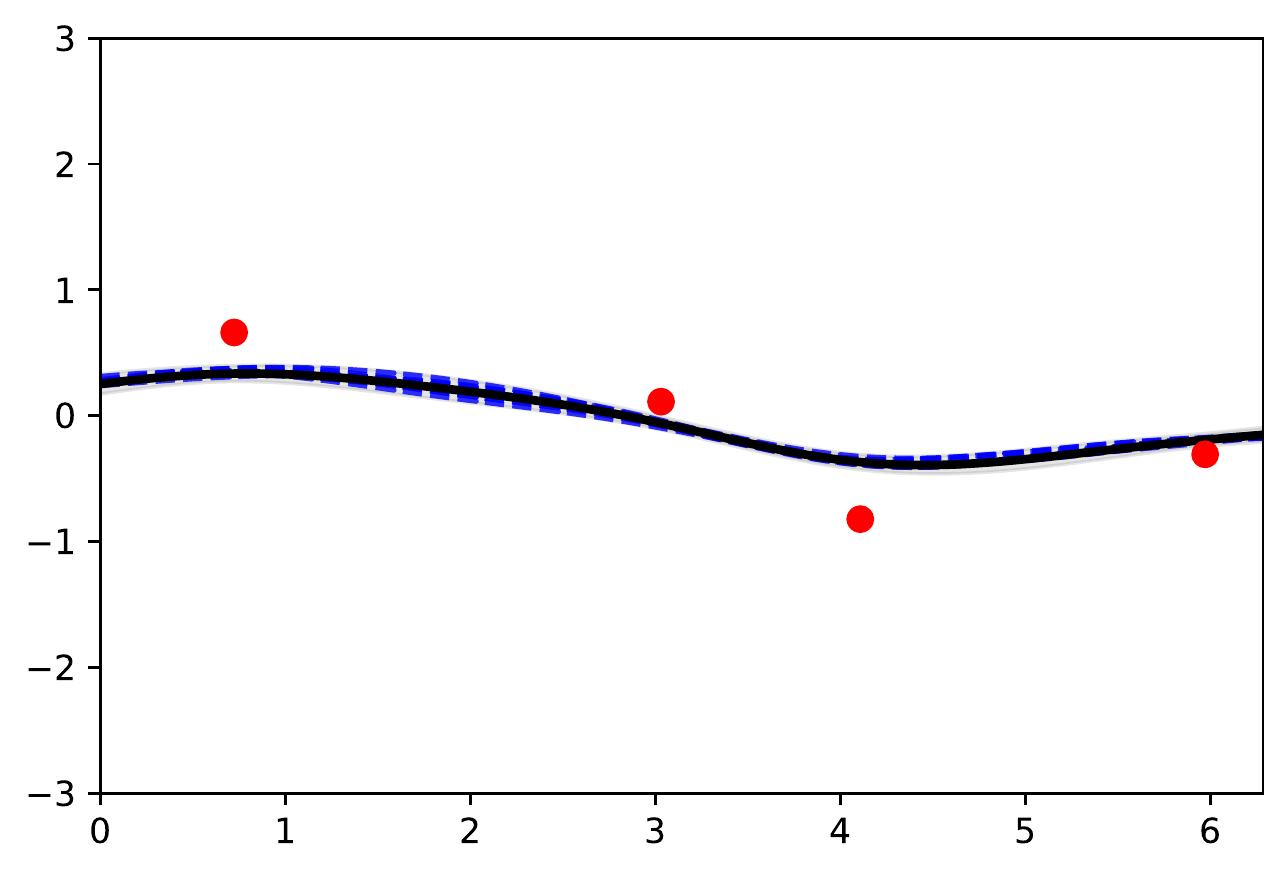}
          \label{fig:RF-A16}
        }
    \caption{\textit{Distribution of the RF predictor.} Red dots represent a sinusoidal dataset $y_i=\sin(x_i)$ for $N=4$ points $x_i$ in $[0, 2\pi)$. For $P\in \{2,4,10,100\} $ and $\lambda\in \{0,10^{-4},10^{-1},1\}$, we sample ten RF predictors (blue dashed lines) and compute empirically the average RF predictor (black lines) with $\pm 2$ standard deviations intervals (shaded regions). \label{fig:RF-predictor-regimes}}
\end{figure}

\clearpage

\subsection{Evolution of the Effective Ridge $\tilde{\lambda}$}
\label{sec:B2}
In Figure \ref{fig:effective-ridge-compare}, we show how the effective ridge $\tilde{\lambda}$ and its derivative $\partial_\lambda\tilde\lambda$ evolve for the selected eigenvalue spectra with various decays (exponential or polynomial) as a function of $\gamma$ and $\lambda$. In Figure \ref{fig:effective-ridge-various-N}, we compare the evolution of $\tilde{\lambda}$ for various $N$.
\vskip0.8cm

\begin{figure}[h!]
    \centering
    \subfloat[Exponential, $\tilde\lambda$]{
        \includegraphics[width=0.325\textwidth]{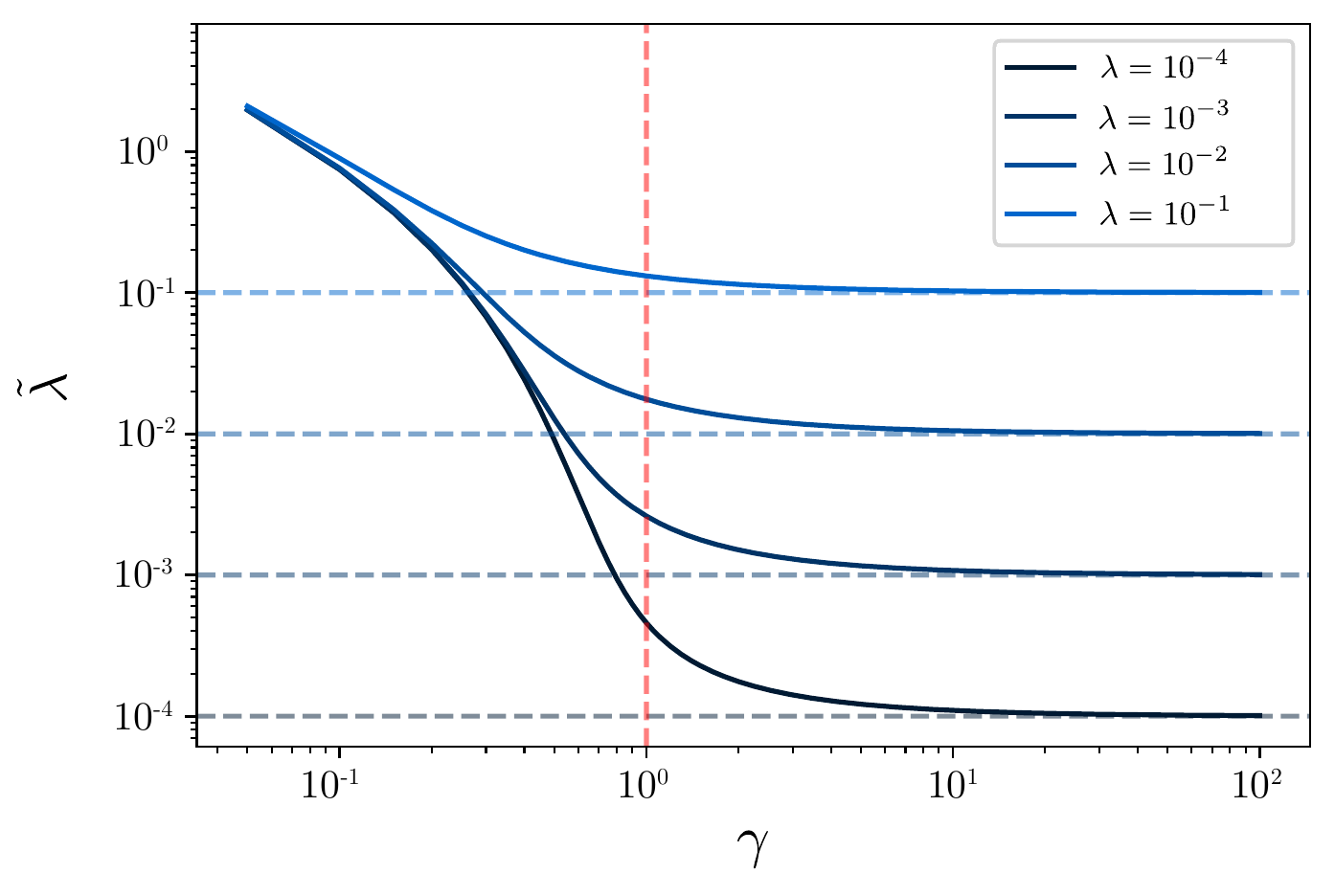}
        \label{fig:eff-ridge-exp}
        } \hfill
    \subfloat[Exponential, $\partial_\lambda\tilde\lambda$]{
        \includegraphics[width=0.315\textwidth]{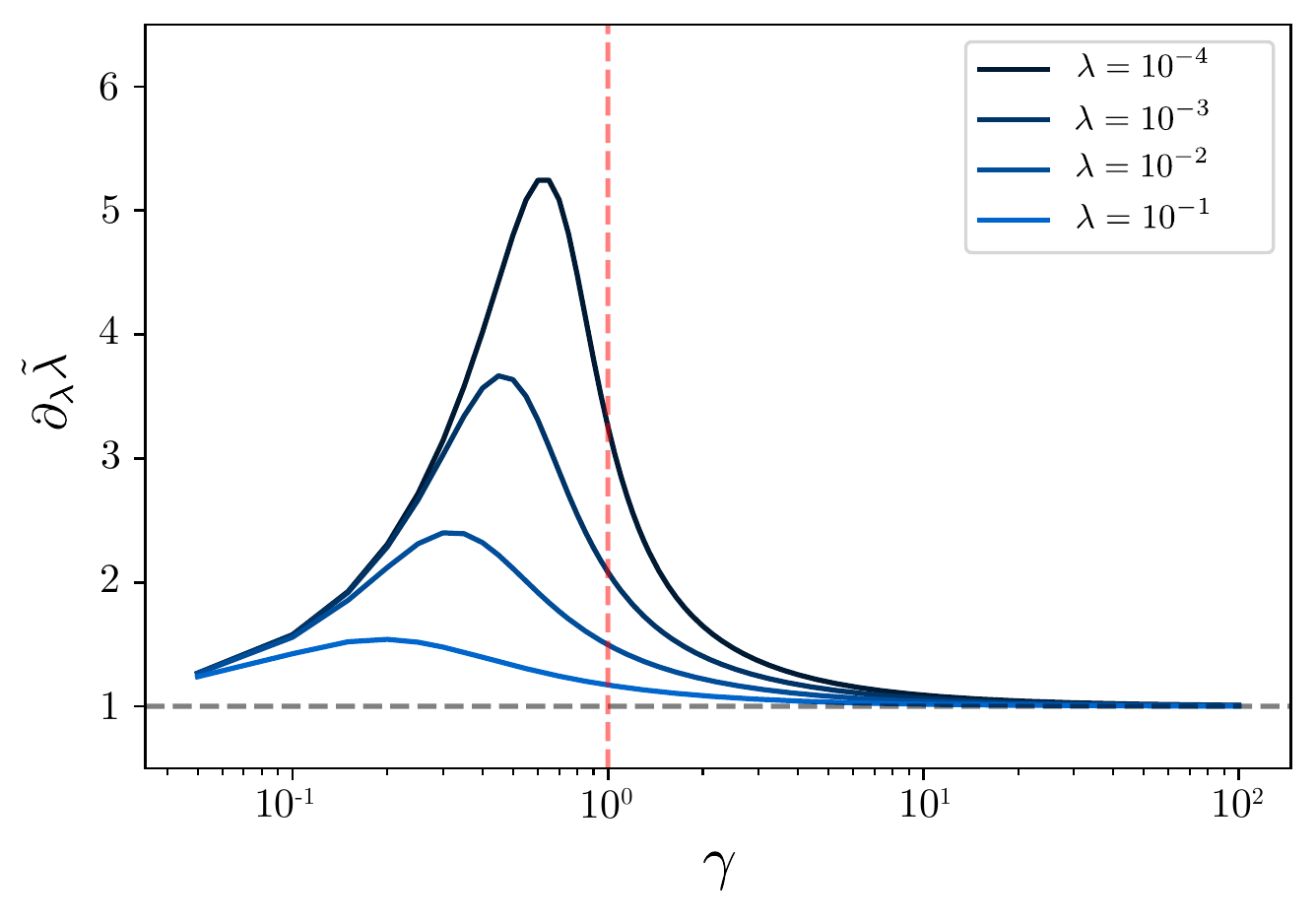}
        \label{fig:der-eff-ridge-exp}
        } \hfill
    \subfloat[Exponential, $\tilde\lambda$]{
        \includegraphics[width=0.325\textwidth]{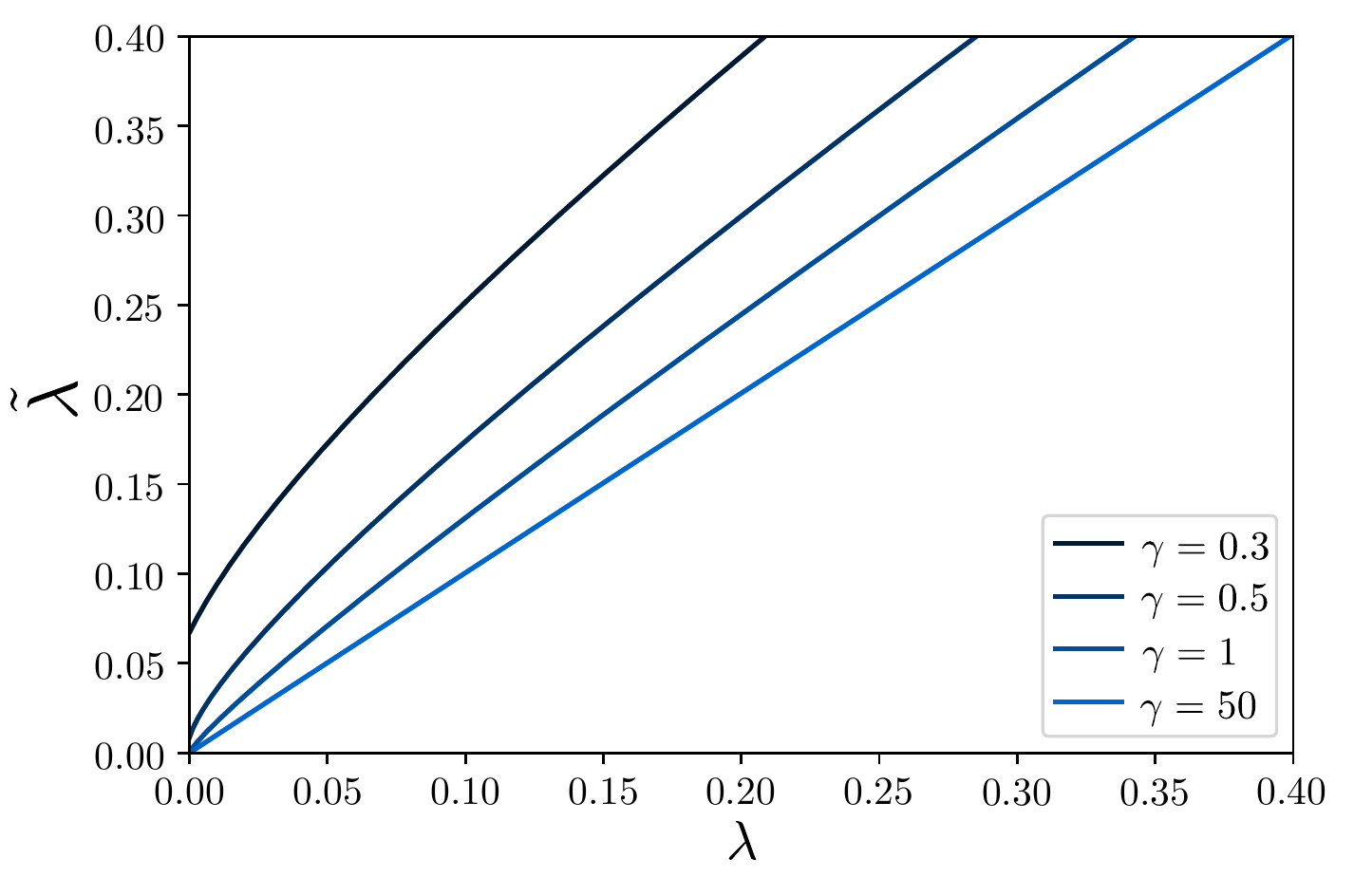}
        \label{fig:eff-ridge-pol}
        } \vfill
    \subfloat[Polynomial, $\tilde\lambda$]{
        \includegraphics[width=0.325\textwidth]{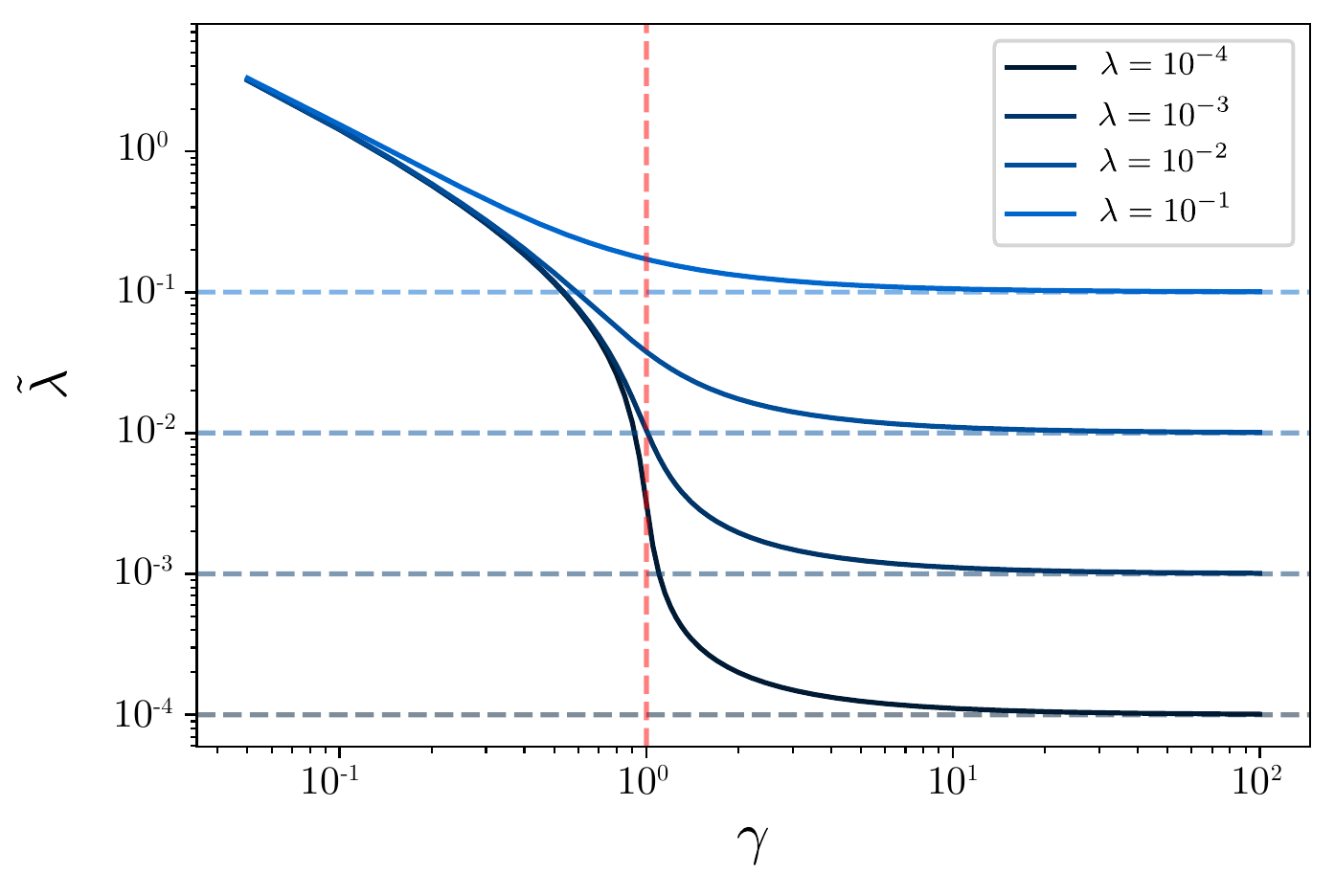}
        \label{fig:der-eff-ridge-pol}
        } \hfill
    \subfloat[Polynomial, $\partial_\lambda\tilde\lambda$]{
        \includegraphics[width=0.315\textwidth]{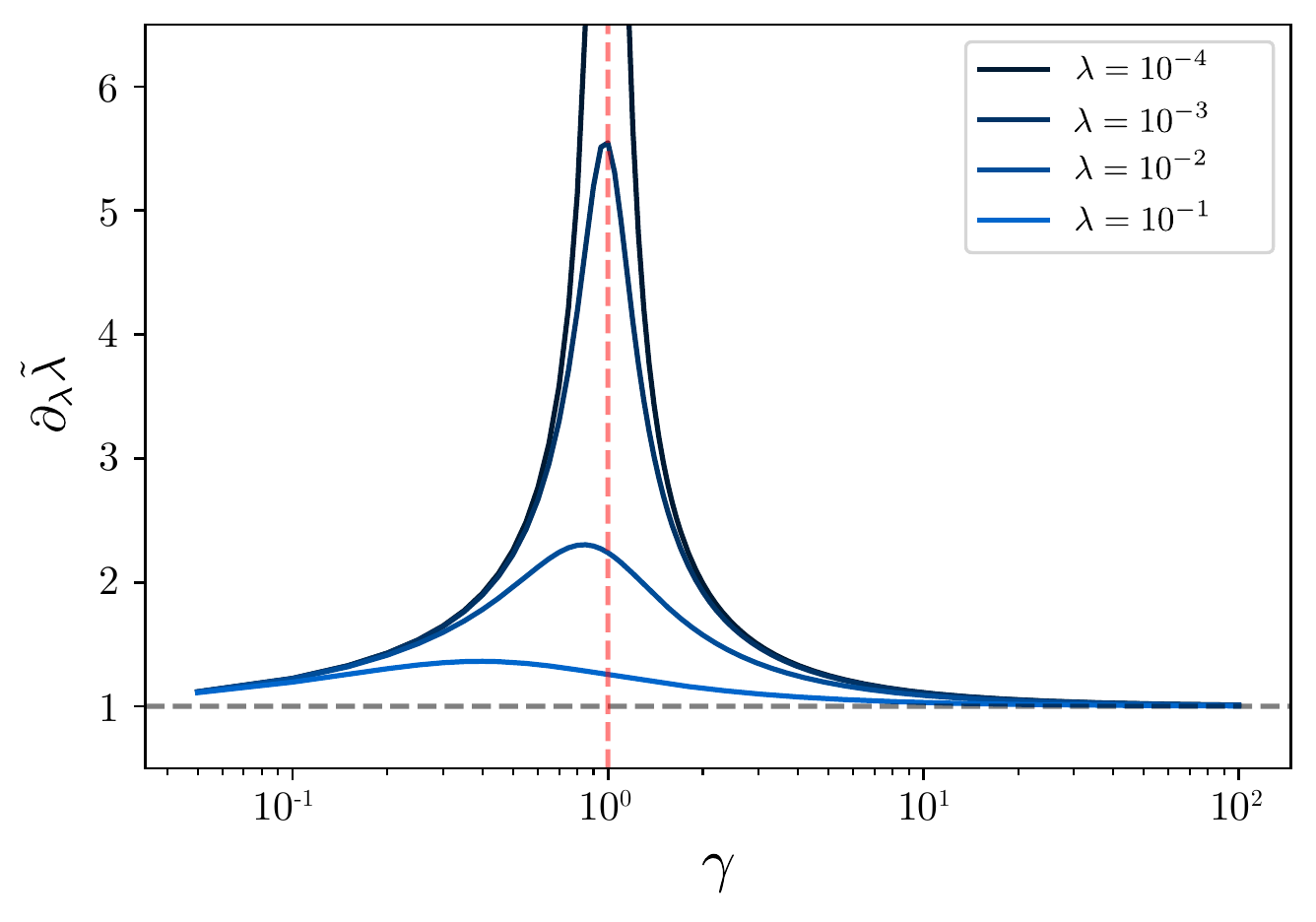}
        \label{fig:der-eff-ridge-pol}
        } \hfill
    \subfloat[Polynomial, $\tilde\lambda$]{
        \includegraphics[width=0.325\textwidth]{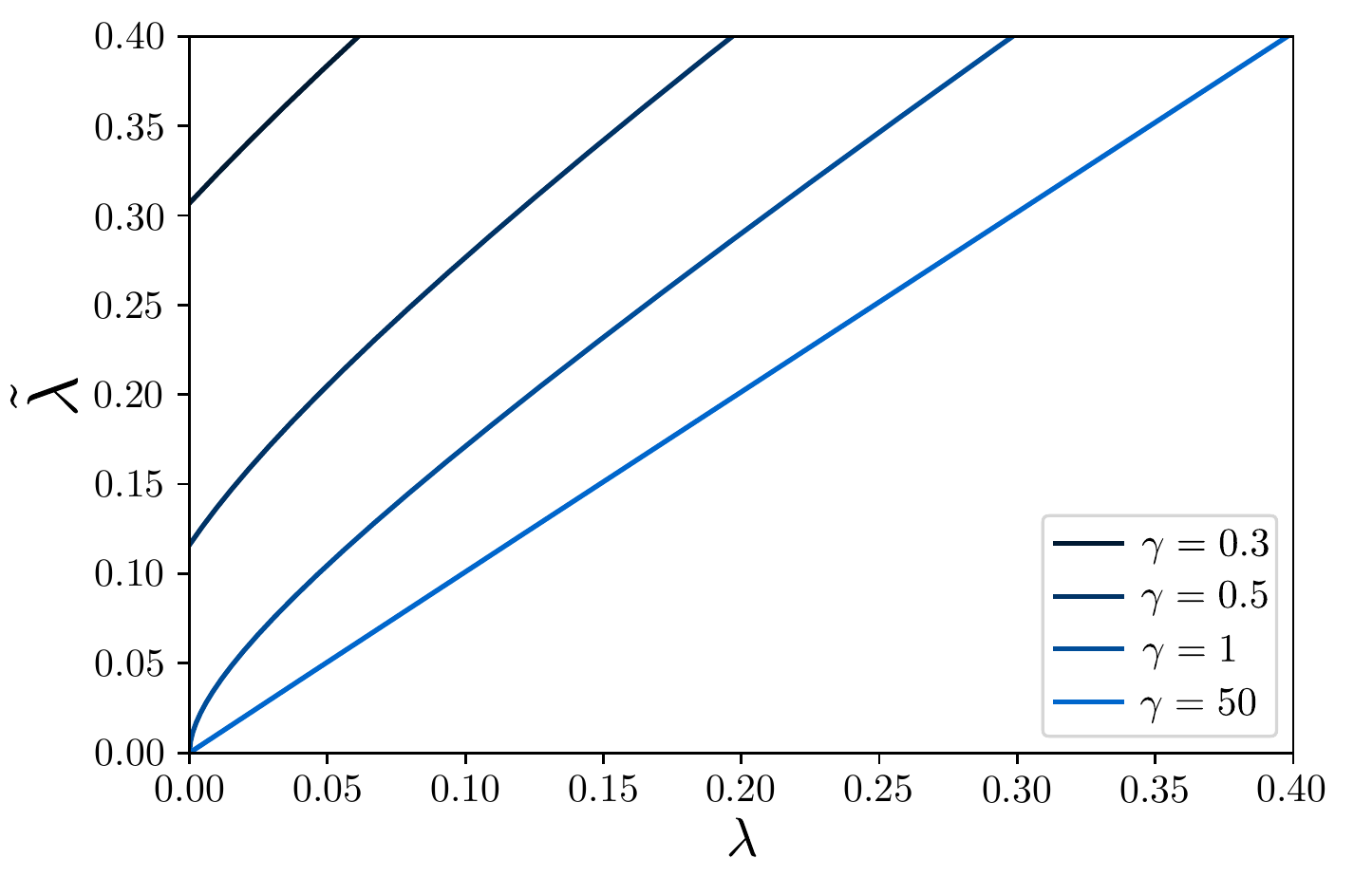}
        \label{fig:der-eff-ridge-pol}
        }

    \caption{\textit{Evolution of the effective ridge $\tilde{\lambda}$ and its derivative $\partial_\lambda \tilde{\lambda}$ for various levels of ridge $\lambda$ (or $\gamma$) and for $N=20$.}  We consider two different decays for $d_1,\ldots,d_N$: (i) exponential decay in $i$ (i.e. $d_i = e^{-\frac{(i-1)}{2}}$, top plots) and (ii) polynomial decay in $i$ (i.e. $d_i = \frac{1}{i}$, bottom plots). \label{fig:effective-ridge-compare}}
\end{figure}

\vskip1cm

\begin{figure}[h!]
    \centering
    \subfloat[$\lambda = 10^{-4}$]{
        \includegraphics[width=0.33\textwidth]{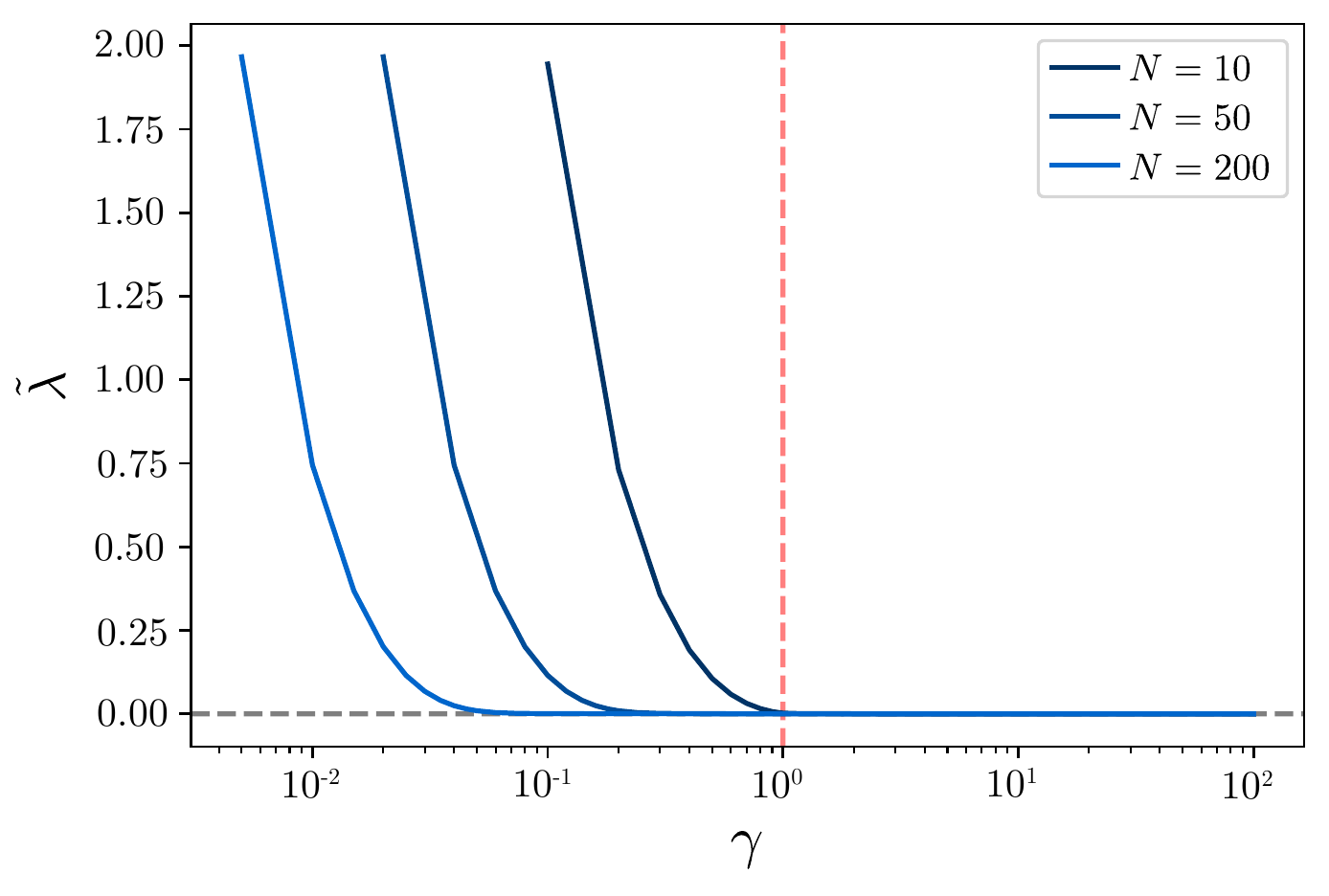}
        \label{fig:ridgeless}
        }
    \subfloat[$\lambda = 0.5$]{
        \includegraphics[width=0.33\textwidth]{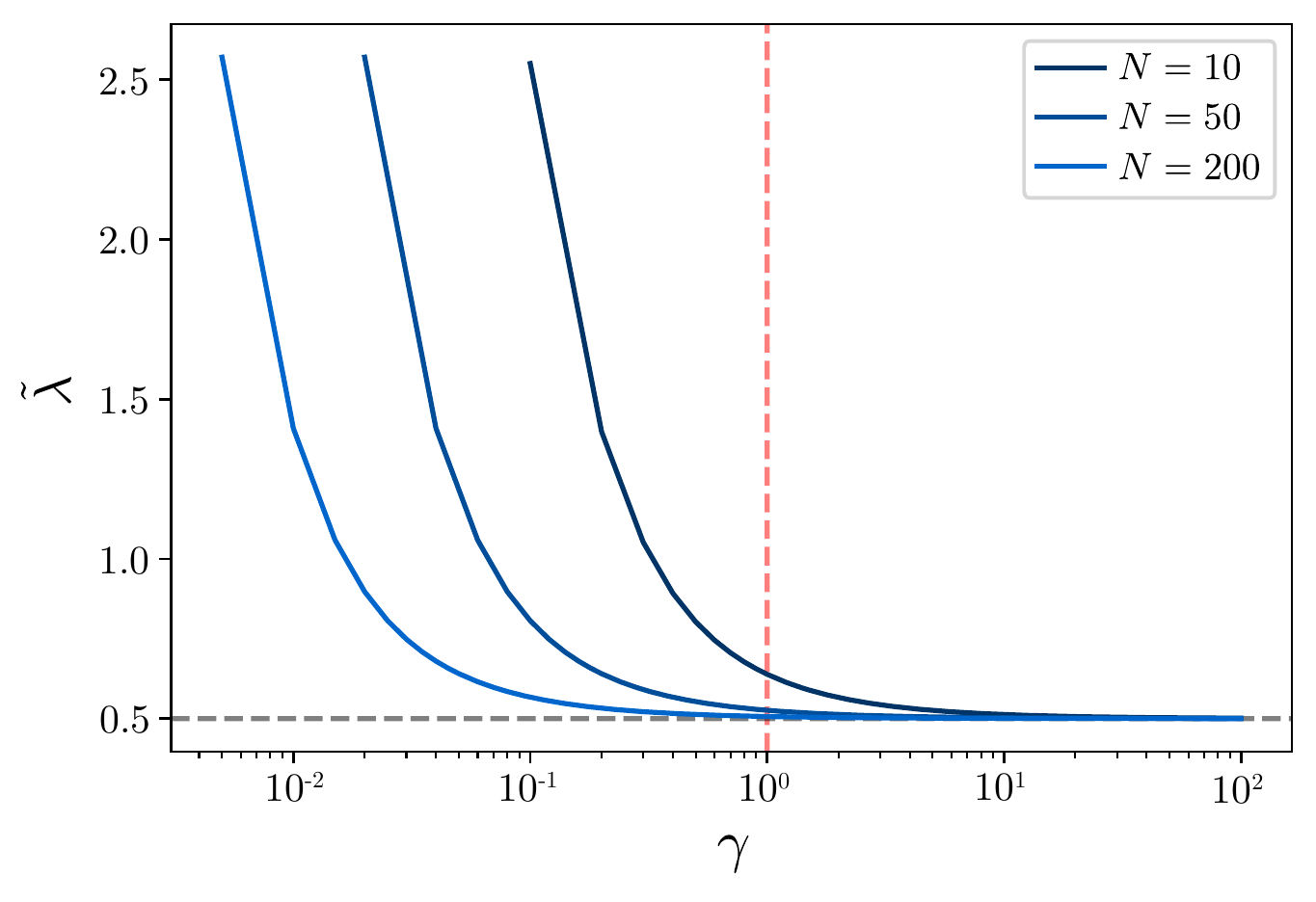}
        \label{fig:ridge}
        }
    \caption{\textit{Evolution of effective ridge $\tilde{\lambda}$ as a function of $\gamma$ for two ridges \textbf{(a)} $\lambda = 10^{-4}$ and \textbf{(b)} $\lambda = 0.5$ and for various $N$.} We consider an exponential decay for $d_1,\ldots,d_N$, i.e. $d_i = e^{-\frac{(i-1)}{2}}$.\label{fig:effective-ridge-various-N}}
\end{figure}

\clearpage

\subsection{Eigenvalues of $A_{\lambda}$}
\label{sec:B3}
The (random) prediction $\hat{y}$ on the training data is given by $\hat{y}=A_{\lambda}y$ where $A_{\lambda}=F(F^{T}F + \lambda I)^{-1}F^{T}$. The average $\lambda$-RF predictor is $\mathbb E[\hat{f}_{\lambda}^{(RF)}(x)]= K(x,X)K(X,X)^{-1} \mathbb E [A_{\lambda}] y$. We denote by $\tilde{d}_1, \ldots \tilde{d}_N$ the eigenvalues of $\mathbb E [A_{\lambda}]$. By Proposition \ref{prop:ridge_expectation}, the $\tilde{d}_i$'s converge to the eigenvalues $\frac{d_1}{d_1+\tilde{\lambda}},\ldots,\frac{d_N}{d_N+\tilde{\lambda}}$ of $K(K+\tilde{\lambda}I_{N})^{-1}$ as $P$ goes to infinity. We illustrate the evolution of $\tilde{d}_i$ and their convergence to $\frac{d_i}{d_i+\tilde{\lambda}}$ for two different eigenvalue spectrums $d_1, \ldots d_N$.

\hskip0.1cm

\begin{figure}[h!]
        \label{fig:figure_E_A_ev}
    \centering
    {
        \includegraphics[width=0.36\textwidth]{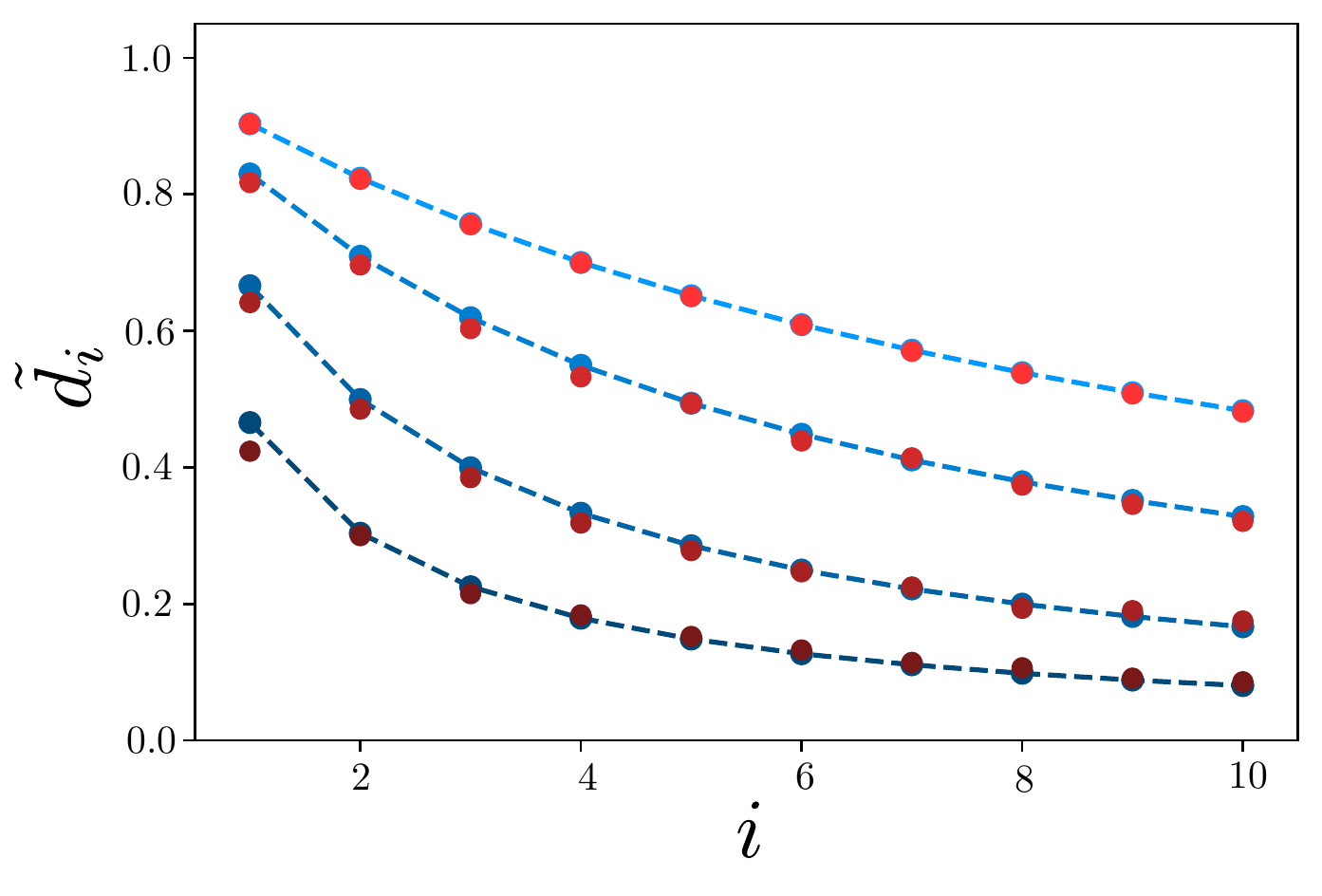}%
      }  \hskip0.5cm
    {
        \includegraphics[width=0.36\textwidth]{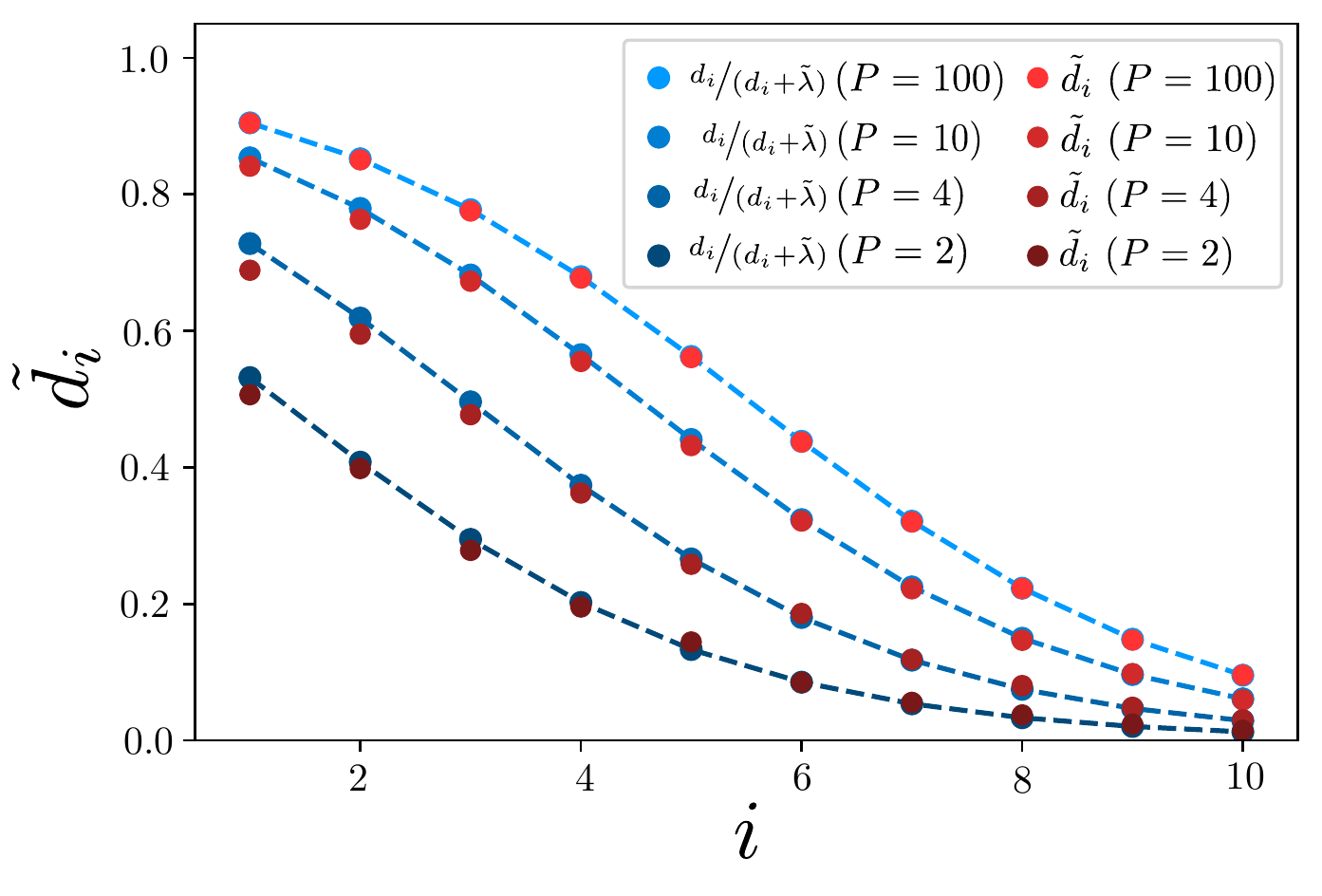}
        \hskip0.2cm
        \rotatebox[origin=r]{-90}{$\lambda = 10^{-1} \qquad \qquad $}
        } \vfill
    {
        \includegraphics[width=0.36\textwidth]{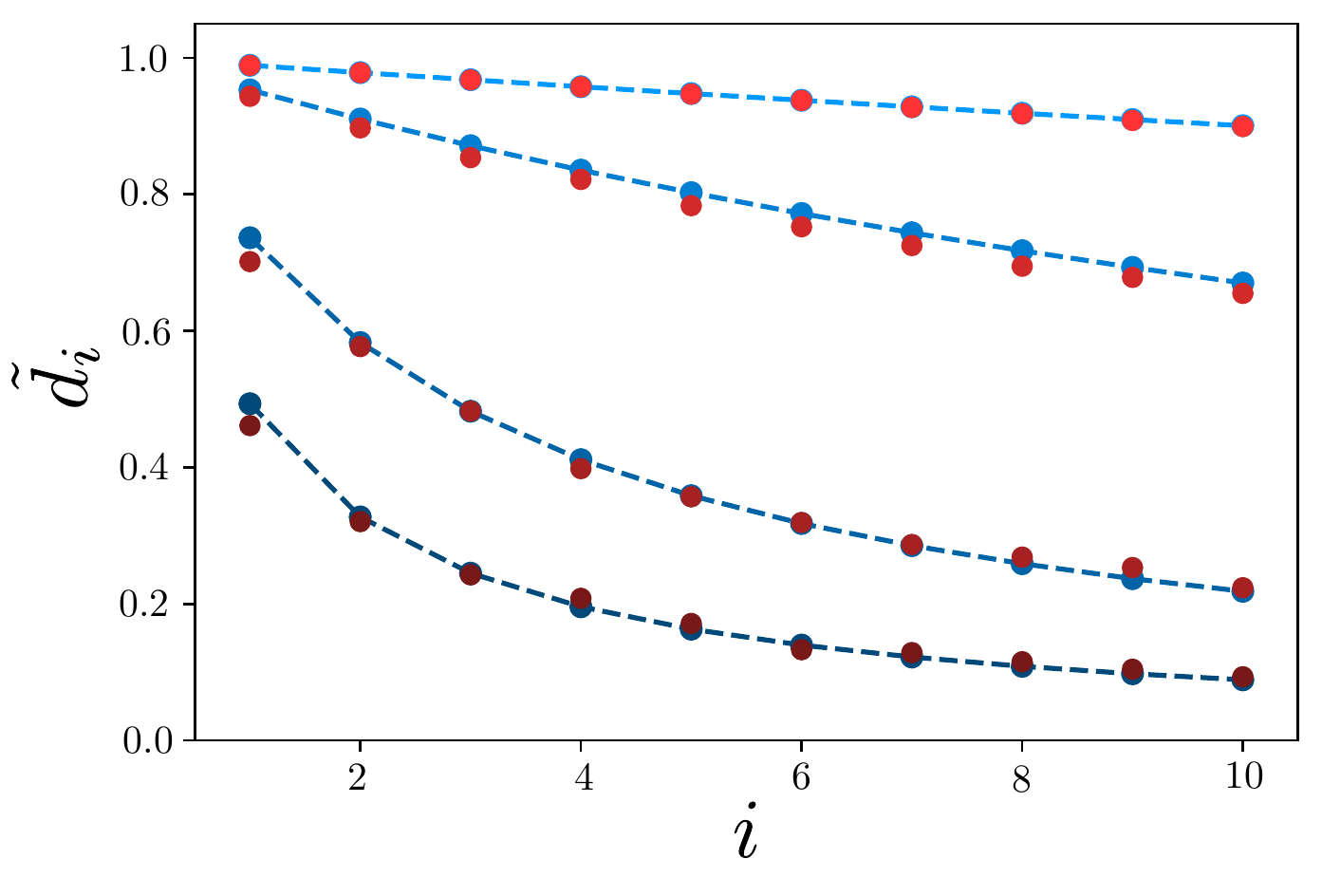}
        } \hskip0.5cm
    {
        \includegraphics[width=0.36\textwidth]{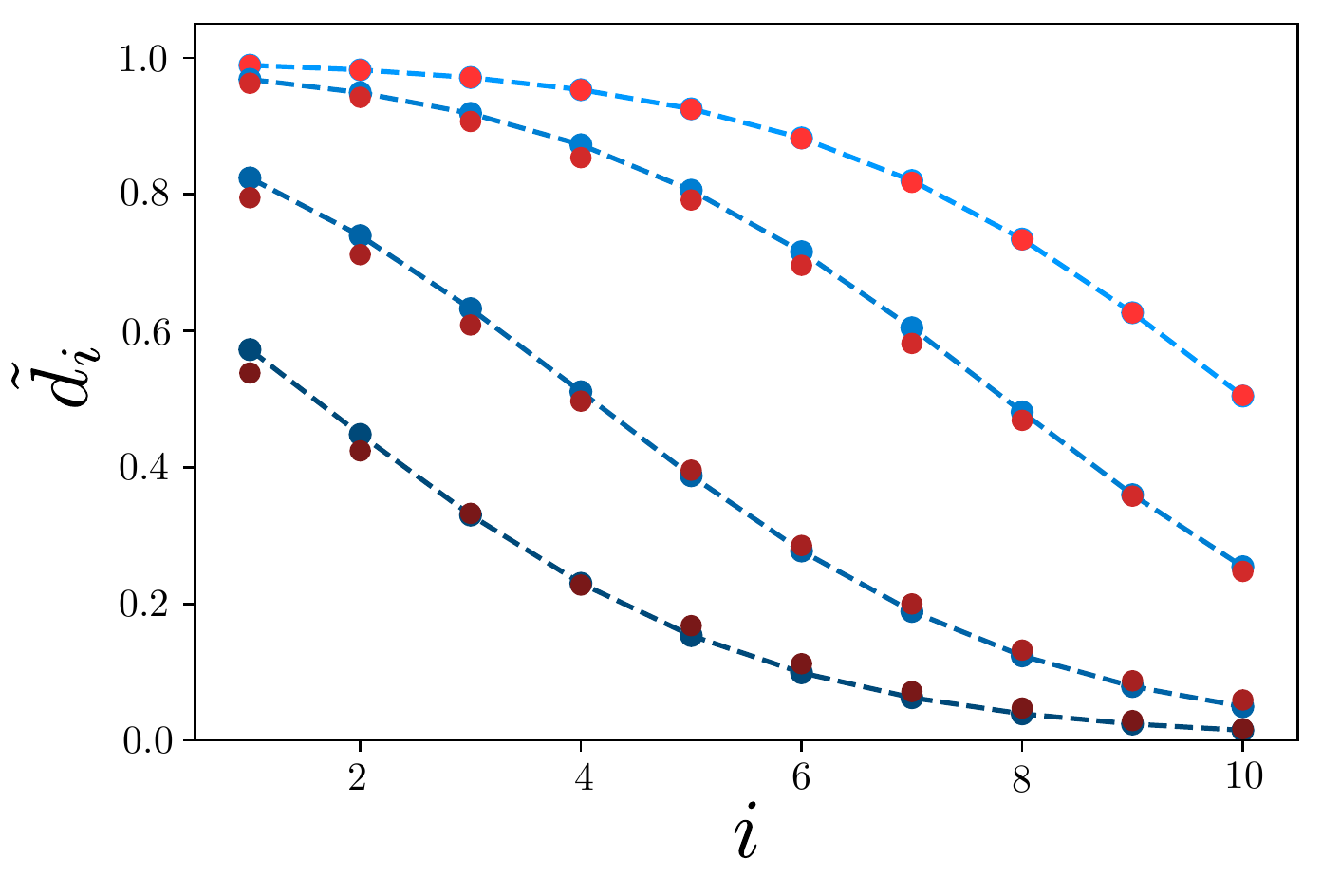}
        \hskip0.2cm
        \rotatebox[origin=r]{-90}{$\lambda = 10^{-2} \qquad \qquad $}
          } \vfill
    {
        \includegraphics[width=0.36\textwidth]{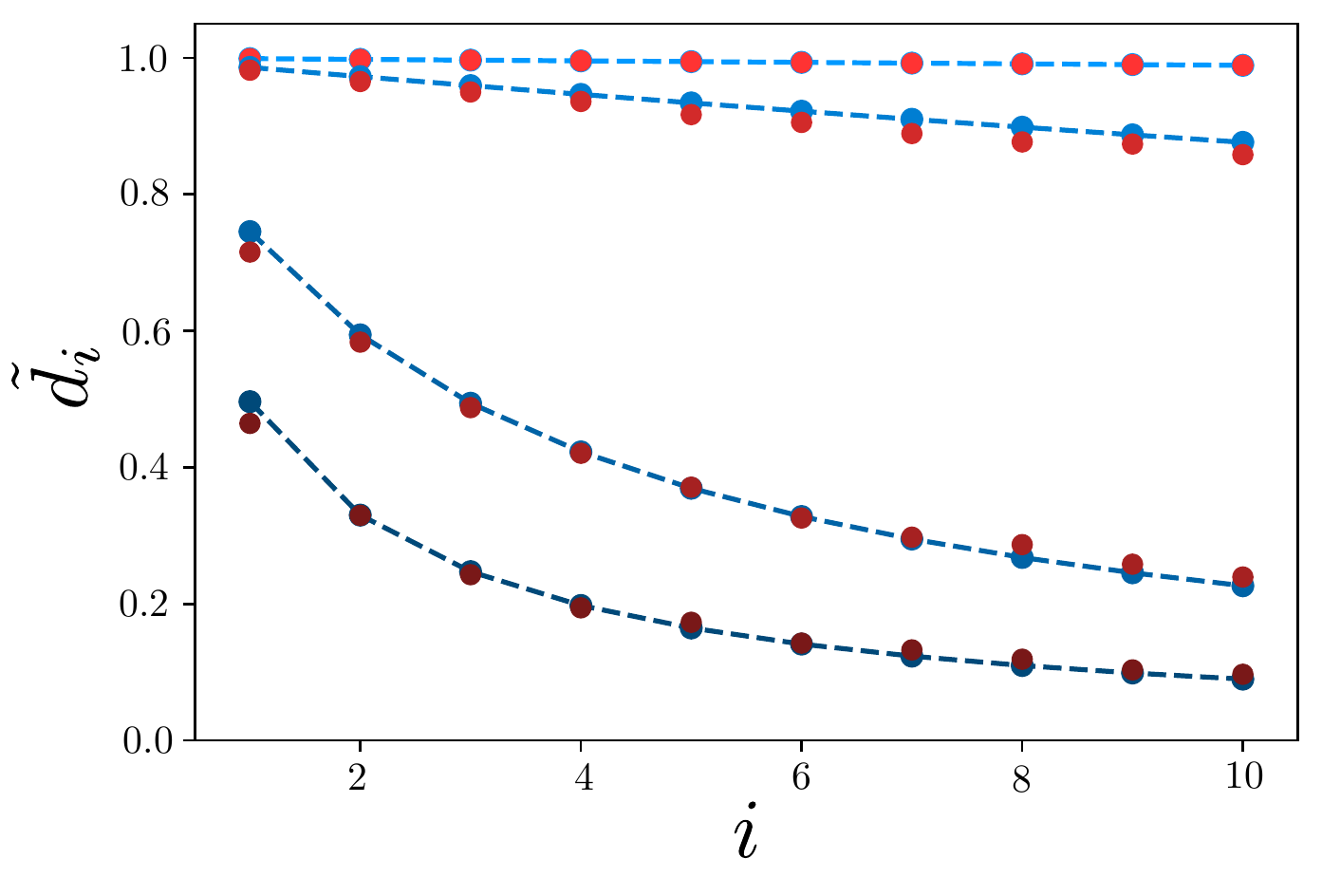}%
        } \hskip0.5cm
    {
        \includegraphics[width=0.36\textwidth]{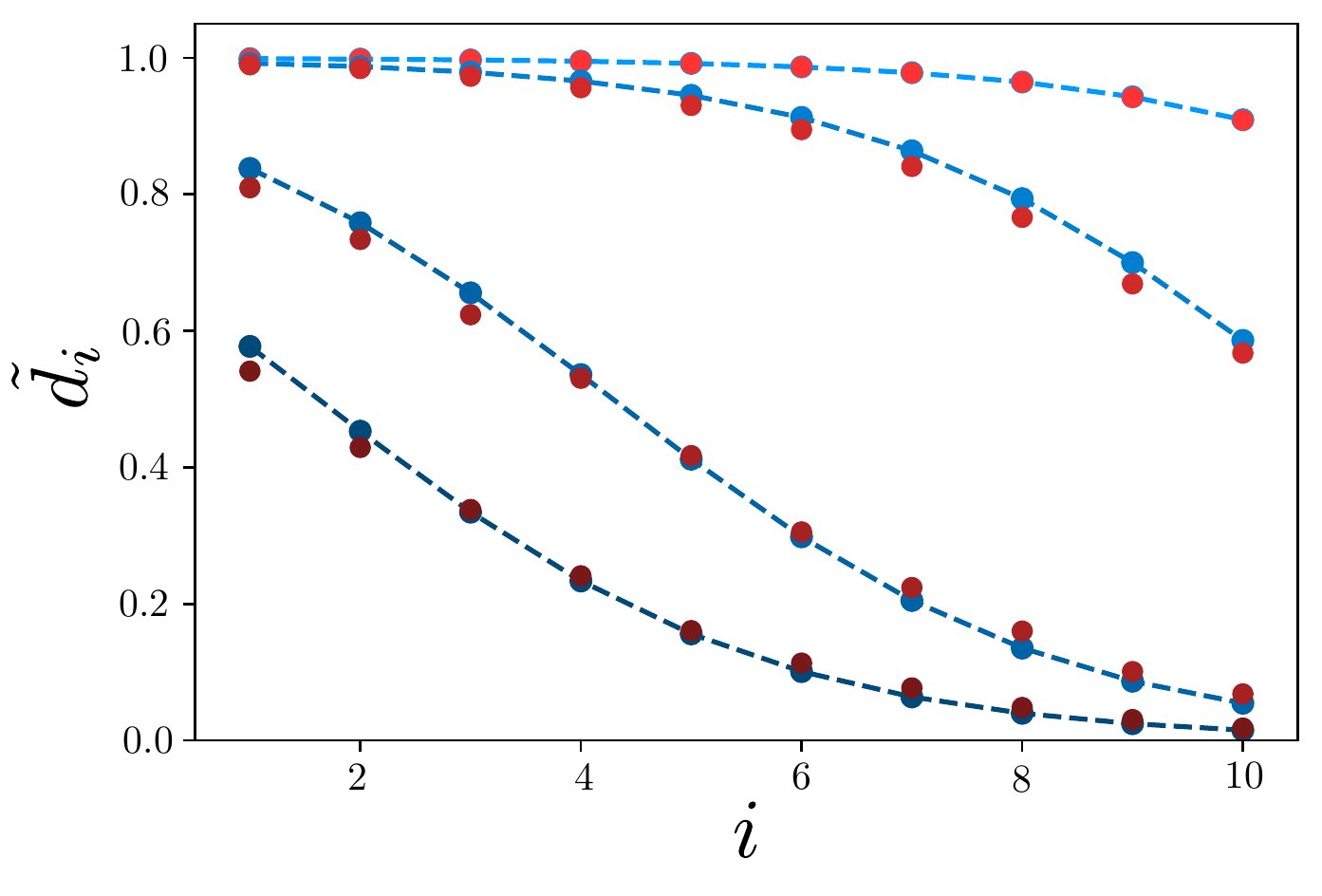}
        \hskip0.2cm
        \rotatebox[origin=r]{-90}{$\lambda = 10^{-3} \qquad \qquad $}
        } \vfill
    {
        \includegraphics[width=0.36\textwidth]{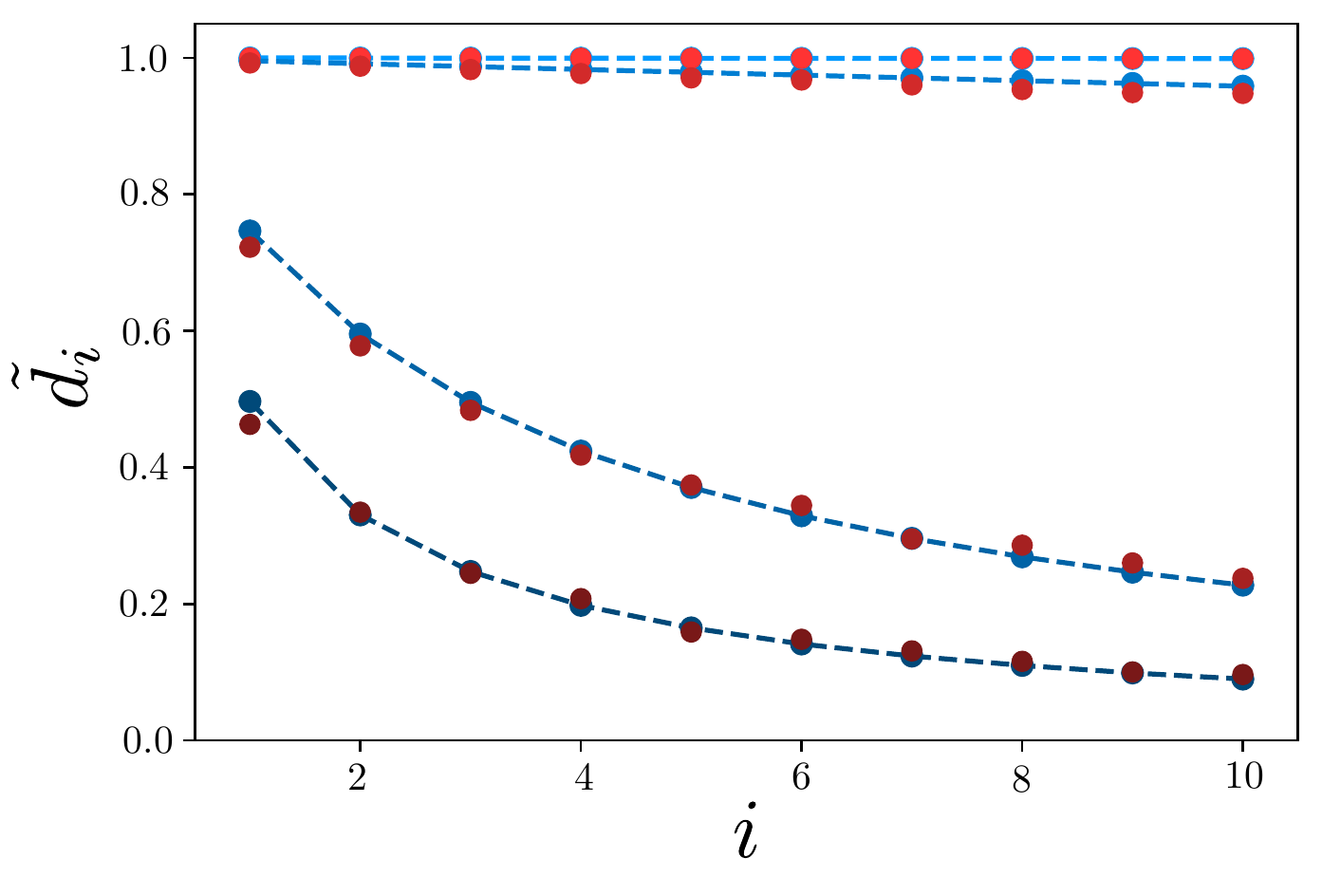}
        } \hskip0.5cm
    {
        \includegraphics[width=0.36\textwidth]{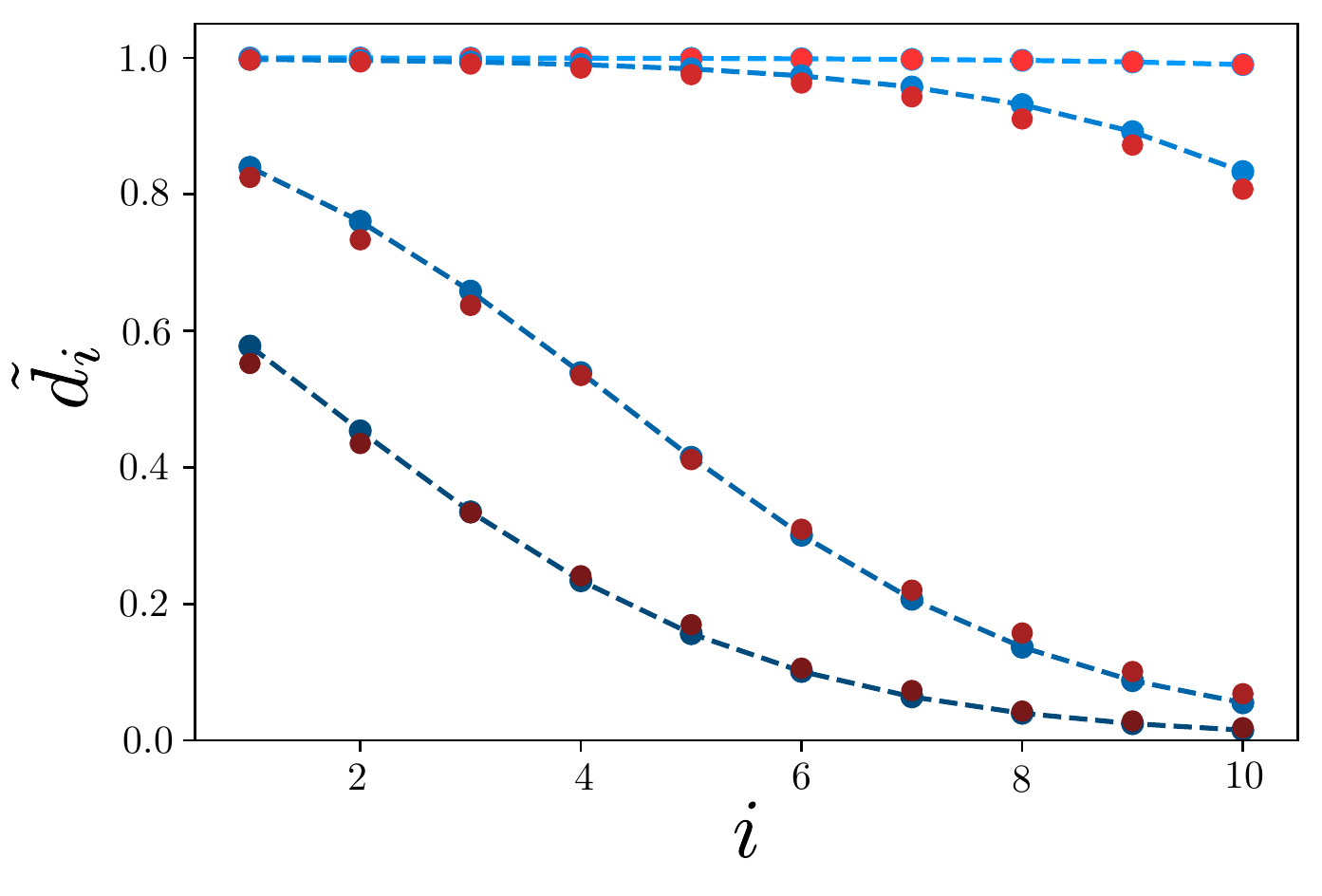}
        \hskip0.2cm
        \rotatebox[origin=r]{-90}{$\lambda = 10^{-4} \qquad \qquad $}
          } \vfill
          \rotatebox{0}{\quad Polynomial} \hskip5.2cm
          \rotatebox{0}{Exponential}
    \label{fig:effective-eigenvalues}
    \caption{\textit{Eigenvalues $\tilde{d}_1, \ldots \tilde{d}_N$ (red\ dots) vs. eigenvalues $\frac{d_1}{d_1+\tilde{\lambda}},\ldots,\frac{d_N}{d_N+\tilde{\lambda}} (blue\ dots)$ for $N=10$.}  We consider various values of $P$ and two different decays for $d_1,\ldots,d_N$: (i) exponential decay in $i$, i.e. $d_i = e^{-\frac{(i-1)}{2}}$ (right plots) and (ii) polynomial decay in $i$, i.e. $d_i=\frac{1}{i}$ (left plots).}
\end{figure}

\pagebreak

\subsection{Average Fourier Features Predictor}
\label{sec:RFF}
The Fourier Features predictor ${\lambda}$-FF is $\hat{f}^{(FF)}(x) = \frac{1}{\sqrt{P}} \sum_{j=1}^{P} \hat{\theta}_j \phi^{(j)}(x)$ where $\phi^{(j)}(x) = \cos(x^{T}w^{(j)} + b^{(j)})$ and  $\hat{\theta}=F^{T}\left(FF^{T}+\lambda \mathrm{I}_N\right)^{-1}y$ with the data matrix $F$ as described in Section \ref{sec:RFF-exp}.

We investigate how close the average ${\lambda}$-FF predictor is to the $\tilde{\lambda}$-KRR predictor and we observe the following:
\begin{enumerate}
  \item The difference of the test errors of the two predictors decreases as  $\gamma$ increases.
  \item In the overparameterized regime, i.e. $P \geq N$, the test error of the $\tilde{\lambda}$-KRR predictor matches with the test error of the ${\lambda}$-FF predictor.
  \item For $N=1000$, strong agreement between the two test errors is observed already for $\gamma > 0.1$. We also observe that Gaussian features achieve lower (or equal) test error than the Fourier features for all $\gamma$ in our experiments.
\end{enumerate}
\vskip1cm

\begin{figure}[h]
    \centering
    \subfloat[$N = 100$]{
        \includegraphics[width=0.45\textwidth]{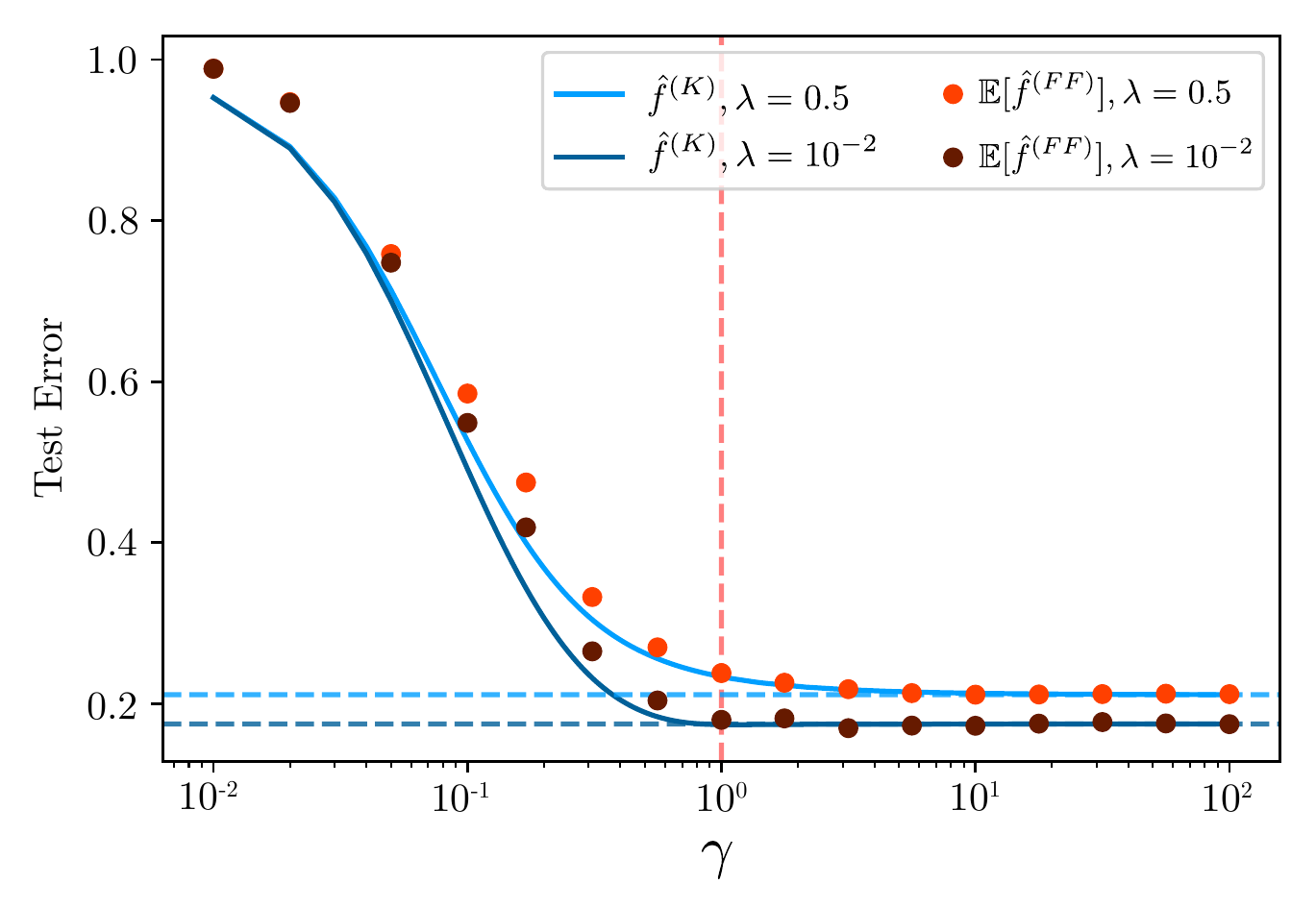}
        \label{fig:eff-ridge-exp}
        }
    \subfloat[$N = 100$]{
        \includegraphics[width=0.45\textwidth]{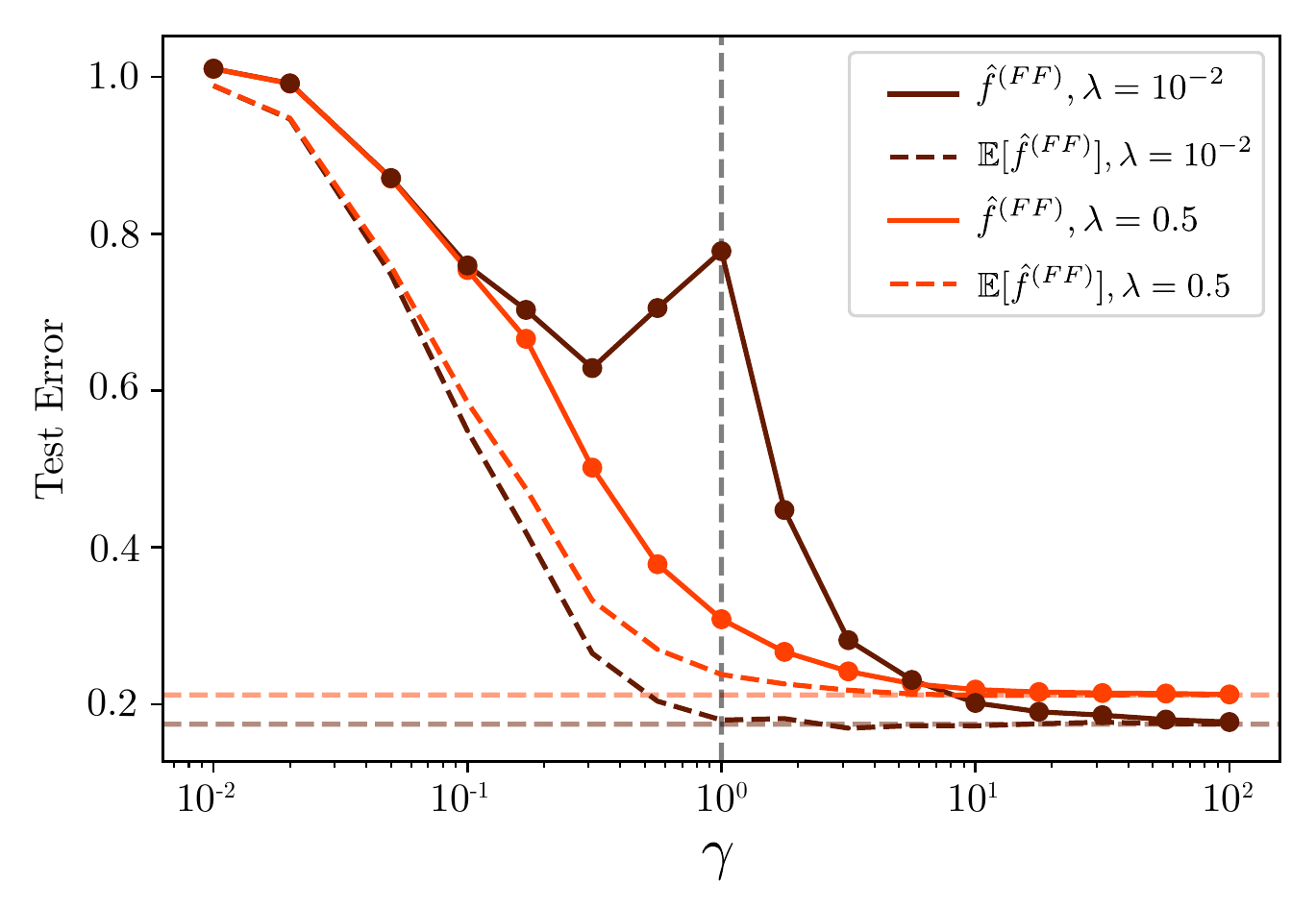}
        \label{fig:der-eff-ridge-exp}
        } \vfill
    \subfloat[$N = 1000$]{
        \includegraphics[width=0.45\textwidth]{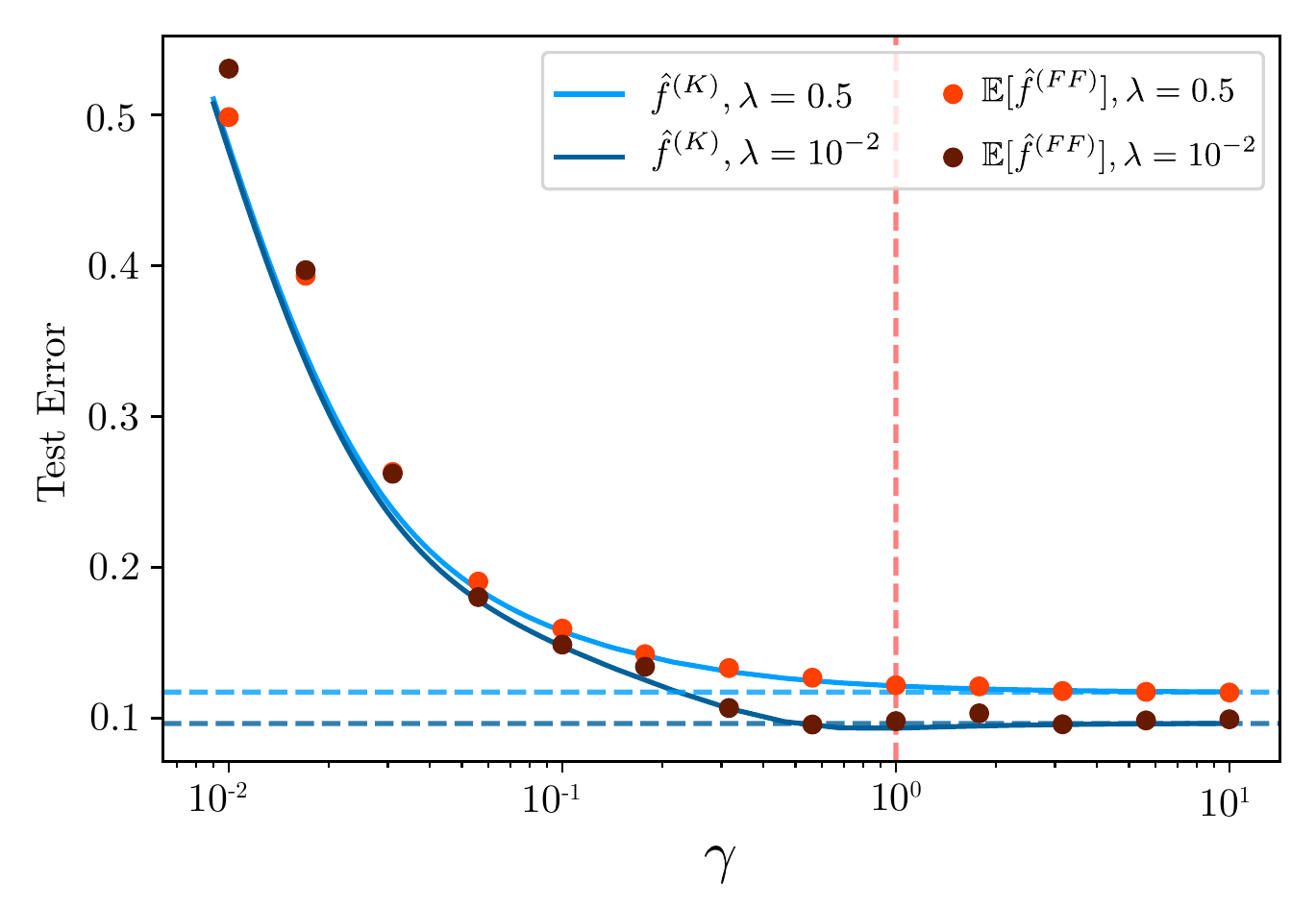}
        \label{fig:eff-ridge-pol}
        }
    \subfloat[$N = 1000$]{
        \includegraphics[width=0.45\textwidth]{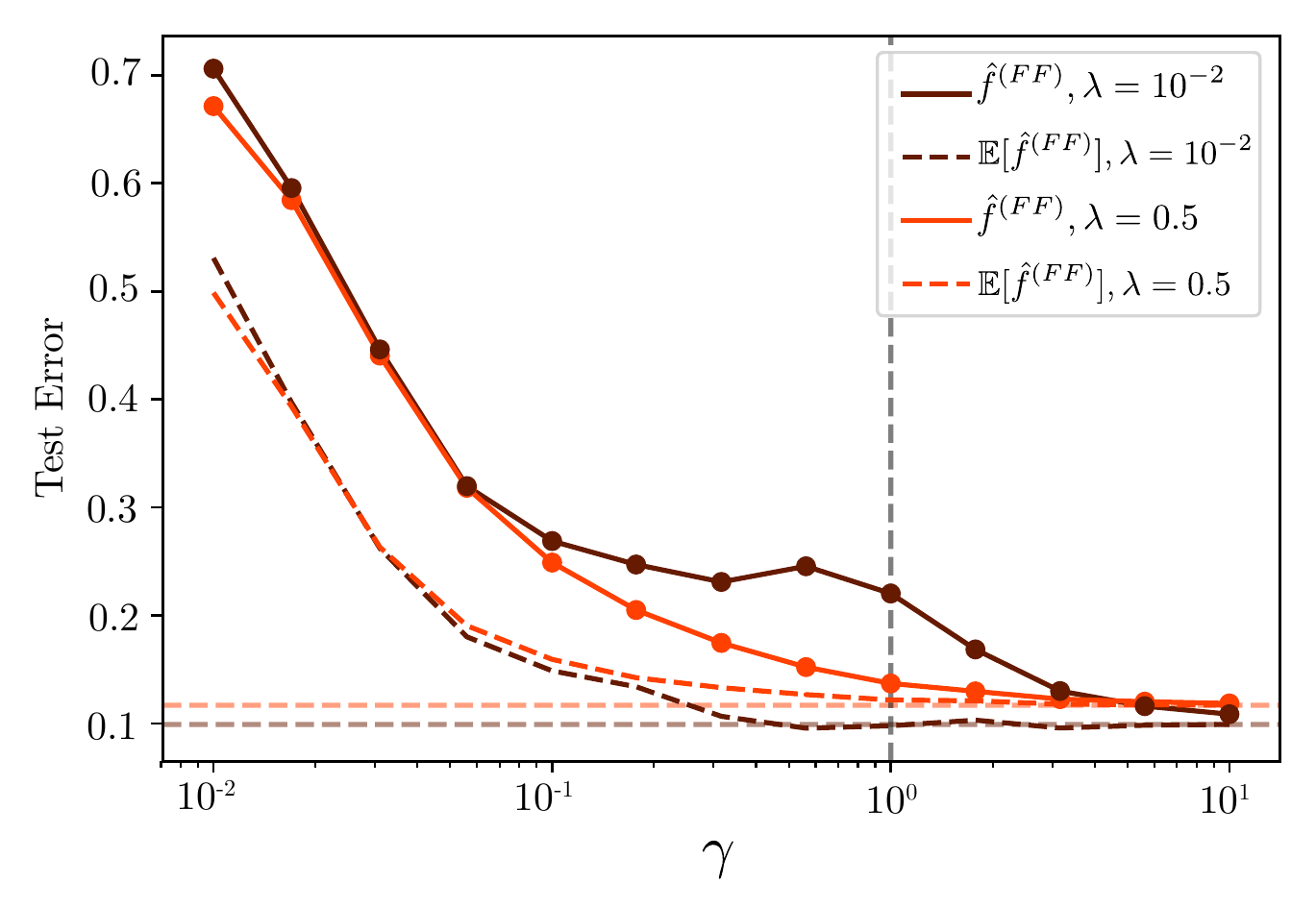}
        \label{fig:der-eff-ridge-pol}
        }
    \caption{\textit{Comparision of the test errors of the average $\lambda$-FF predictor and the $\tilde{\lambda}$-KRR predictor.} In \textbf{(a)} and \textbf{(c)}, the test errors of the average $\lambda$-FF predictor and of the $\tilde{\lambda}$-KRR predictor are reported for various ridge for $N=100$ and $N=1000$ MNIST data points (top and bottom rows). In \textbf{(b)} and \textbf{(d)}, the average test error of the $\lambda$-FF predictor and the test error of its average are reported. \label{fig:RFF}}
\end{figure}

\vskip1cm
{\color{white} .}
\pagebreak

\section{Proofs}
\label{sec:proofs}
\subsection{Gaussian Random Features}
\begin{proposition}\label{prop:distribution-estimator}
Let $\frfl $ be the $\lambda$-RF predictor and let $\hat{y} = F \hat{\theta}$ be the prediction vector on training data, i.e. $\hat{y}_i = \frfl(x_i)$. The process $ \frfl $ is a mixture of Gaussians: conditioned on $ F $, we have that $ \frfl $ is a Gaussian process. The mean and covariance of $ \frfl $ conditioned on $ F $ are given by
\begin{align}\label{eq:mean-gaussian-mixture}
& \mathbb E[\frfl(x) | F]  = K(x,X)K(X,X)^{-1}\hat{y}, \\
\label{eq:cov-gaussian-mixture}
& \mathrm{Cov}[\frfl(x),\frfl(x') | F]  = \frac{\| \hat{\theta}\|^{2}}{P}\tilde{K}(x, x')
\end{align}
where $\tilde{K}(x, x') = K(x,x')-K(x,X)K(X,X)^{-1}K(X,x')$ denotes the posterior covariance kernel.
\end{proposition}

\begin{proof}
Let $F=(\frac{1}{\sqrt{P}}f^{(j)}(x_{i}))_{i,j}$ be the $N\times P$
matrix of values of the random features on the training set. By definition,
$\hat{f}_{\lambda}^{(RF)}=\frac{1}{\sqrt{P}}\sum_{p=1}^{P}\hat{\theta}_{p}f^{(p)}$.
Conditioned on the matrix $F$, the optimal parameters $(\hat{\theta}_{p})_{p}$
are not random and $(f^{(p)})_{p}$ is still Gaussian, hence, conditioned
on the matrix $F$, the process $\hat{f}_{\lambda}^{(RF)}$ is a mixture
of Gaussians. Moreover, conditioned on the matrix $F$, for any $p,p'$,
$f^{(p)}$ and $f^{(p')}$ remain independent, hence
\begin{eqnarray*}
\mathbb{E}\left[\hat{f}_{\lambda}^{(RF)}(x)\mid F\right] & = & \frac{1}{\sqrt{P}}\sum_{p=1}^{P}\hat{\theta}_{p}\mathbb{E}\left[f^{(p)}(x)\mid f_{N}^{(p)}\right]\\
\mathrm{Cov}\left[\hat{f}_{\lambda}^{(RF)}(x),\hat{f}_{\lambda}^{(RF)}(x')\mid F\right] & = & \frac{1}{P}\sum_{p=1}^{P}\hat{\theta}_{p}^{2}\mathrm{Cov}\left[f^{(p)}(x),f^{(p)}(x')\mid f_{N}^{(p)}\right].
\end{eqnarray*}
where we have set $f_{N}^{(p)}=(f^{(p)}(x_{i}))_{i}\in\mathbb{R}^{N}.$
The value of $\mathbb{E}\left[f^{(p)}(x)\mid f_{N}^{(p)}\right]$
and $\mathrm{Cov}\left[f^{(p)}(x),f^{(p)}(x')\mid f_{N}^{(p)}\right]$
are obtained from classical results on Gaussian conditional distributions \cite{eaton-07}:

\begin{eqnarray*}
\mathbb{E}\left[f^{(p)}(x)\mid f_{N}^{(p)}\right] & = & K(x,X)K(X,X)^{-1}f_{N}^{(p)},\\
\mathrm{Cov}\left[f^{(p)}(x),f^{(p)}(x')\mid f_{N}^{(p)}\right] & = & \tilde{K}(x,x'),
\end{eqnarray*}
where $\tilde{K}(x,x')=K(x,x')-K(x,X)K(X,X)^{-1}K(X,x').$ Thus, conditioned
on $F$, the predictor $\hat{f}_{\lambda}^{(RF)}$ has expectation:
\[
\mathbb{E}\left[\hat{f}_{\lambda}^{(RF)}(x)\mid F\right]=K(x,X)K(X,X)^{-1}\frac{1}{\sqrt{P}}\sum_{p=1}^{P}\hat{\theta}_{p}f_{N}^{(p)}=K(x,X)K(X,X)^{-1}\hat{y}
\]
and covariance:
\[
\mathrm{Cov}\left[\hat{f}_{\lambda}^{(RF)}(x),\hat{f}_{\lambda}^{(RF)}(x')\mid F\right]=\frac{1}{P}\sum_{p=1}^{P}\hat{\theta}_{p}^{2}\tilde{K}(x,x')=\frac{\Vert \hat{\theta} \Vert ^{2}}{P}\tilde{K}(x,x').
\]
\end{proof}

\subsection{Generalized Wishart Matrix}
\textbf{Setup. }
In this section, we consider a fixed deterministic matrix $K$ of size $N\times N$ which is diagonal positive semi-definite, with eigenvalues $d_1,\ldots,d_N$. We also consider a $P\times N$ random matrix $W$ with i.i.d. standard Gaussian entries.

The key object of study is the $P\times P$ generalized Wishart random matrix $F^{T}F=\frac{1}{P}WKW^{T}$ and in particular its Stieltjes transform defined on $z\in\mathbb{C}\setminus\mathbb{R}^+$, where $\mathbb{R}^+=[0,+\infty[$:
\[
m_{P}(z)=\frac{1}{P}\mathrm{Tr}\left[\left(F^{T}F-z\mathrm{I}_{P}\right)^{-1}\right]=\frac{1}{P}\mathrm{Tr}\left[\left(\frac{1}{P}WKW^{T}-z\mathrm{I}_{P}\right)^{-1}\right],
\]
where $K$ is a fixed positive semi-definite matrix.

Since $F^{T}F$  has positive real eigenvalues $\lambda_{1},\ldots,\lambda_{P} \in\mathbb{R}_{+}$, and
\[
m_{P}(z)=\frac{1}{P}\sum_{p=1}^{P}\frac{1}{\lambda_{p}-z},
\]
we have that for any $z\in\mathbb{C}\setminus\mathbb{R}^{+}$,
$$\left|m_{P}(z)\right|\leq\frac{1}{d(z,\mathbb{R}_{+})},$$ where $d(z,\mathbb{R}_{+})=\inf\left\{ \left|z-y\right|,y\in\mathbb{R}^{+}\right\} $
is the distance of $z$ to the positive real line. More precisely, $m_{P}(z)$ lies in the convex hull $\Omega_{z}=\mathrm{Conv}\left(\left\{ \frac{1}{d-z}:d\in\mathbb{R}_{+}\right\} \right)$. As a consequence, the argument $\arg\left(m_{P}(z)\right)\in(-\pi,\pi)$ lies between $0$ and $\arg\left(-\frac{1}{z}\right)$, i.e. $m_{P}(z)$ lies in the cone spanned by $1$ and $-\frac{1}{z}$.

Our first lemma implies that the Stieljes transform concentrates around
its mean as $N$ and $P$ go to infinity with $\gamma=\frac{P}{N}$ fixed.

\begin{lemma}\label{lem:concentration_stieltjes}
For any integer $m\in \mathbb N$ and any $z\in\mathbb{C}\setminus\mathbb{R}^{+}$, we have
$$\mathbb{E}\left[\left|m_{P}(z)-\mathbb{E}\left[m_{P}(z)\right]\right|^{m}\right]\leq \mathbf{c} P^{-\frac{m}{2}},$$
\end{lemma}
where $\mathbf{c}$ depends on $z$, $\gamma$, and $m$ only.
\proof
The proof follows Step 1 of \cite{bai-2008}. Let $w_{1},...,w_{N}$ be the columns of $W$ from left to right. Let us introduce the $P\times P$ matrices $B(z)=\frac{1}{P}WKW^{T}-z\mathrm{I}_{P}$ and $B_{(i)}(z)=\frac{1}{P}W_{(i)}K_{(i)}W_{(i)}^{T}-z\mathrm{I}_{P}$ where $W_{(i)}$ is the $P\times(N-1)$ submatrix of $W$ obtained by removing its $i$-th column $w_{i}$, and $K_{(i)}$ is the $(N-1)\times(N-1)$ submatrix of $K$ obtained by removing both its $i$-th column and $i$-th row. Since the eigenvalues of $WKW^{T}$ and $W_{(i)}K_{(i)}W_{(i)}^{T}$ are all real and positive, $B(z)$ and $B_{(i)}(z)$ are invertible matrices for $z\notin \mathbb R^{+}$.

Noticing that $$B(z)=\frac{1}{P}WKW^{T}-z\mathrm{I}_{P}=\frac{1}{P}W_{(i)}K_{(i)}W_{(i)}^{T}-z\mathrm{I}_{P}+\frac{d_{i}}{P}w_{i}w_{i}^{T}$$
is a rank one perturbation of the matrix $B_{(i)}(z)$, by the Sherman--Morrison's formula, the inverse of $B(z)$ is given by:
\[
B(z)^{-1}=\left(B_{(i)}(z)\right)^{-1}-\frac{d_{i}}{P}\frac{1}{1+\frac{d_{i}}{P}w_{i}^{T}\left(B_{(i)}(z)\right)^{-1}w_{i}}\left(B_{(i)}(z)\right)^{-1}w_{i}w_{i}^{T}\left(B_{(i)}(z)\right)^{-1}.
\]

We denote $\mathbb{E}_{i}$ the conditional expectation given $w_{i+1},...,w_{N}$. We have $\mathbb{E}_{0}[m_{P}(z)]=m_{P}(z)$ and $\mathbb{E}_{N}[m_{P}(z)]=\mathbb{E}[m_{P}(z)].$ As a consequence, we get:
\begin{equation*}\label{}
\begin{split}
m_{P}(z)-\mathbb{E}[m_{P}(z)] & =\sum_{i=1}^{N}\left(\mathbb{E}_{i-1}[m_{P}(z)]-\mathbb{E}_{i}[m_{P}(z)]\right)\\
 & =\frac{1}{P}\sum_{i=1}^{N}\left(\mathbb{E}_{i-1}-\mathbb{E}_{i}\right)\left[\mathrm{Tr}\left(B(z)^{-1}\right)\right]\\
 & =\frac{1}{P}\sum_{i=1}^{N}\left(\mathbb{E}_{i-1}-\mathbb{E}_{i}\right)\left[\mathrm{Tr}\left(B(z)^{-1}\right)-\mathrm{Tr}\left(B_{(i)}(z)^{-1}\right)\right].
\end{split}
\end{equation*}
 The last equality comes from the fact that $\mathrm{Tr}\left(B_{(i)}(z)^{-1}\right)$
does not depend on $w_{i}$, hence
\[
\mathbb{E}_{i-1}\left[\mathrm{Tr}\left(B_{(i)}(z)^{-1}\right)\right]=\mathbb{E}_{i}\left[\mathrm{Tr}\left(B_{(i)}(z)^{-1}\right)\right].
\]

Let $g_{i}:\mathbb C\setminus \mathbb R^{+}\to \mathbb C$ be the holomorphic function given by $g_{i}(z):=\frac{1}{P}w_{i}^{T}\left(B_{(i)}(z)\right)^{-1}w_{i}$. Its derivative is given by $g_{i}'(z)=\frac{1}{P}w_{i}^{T}\left(B_{(i)}(z)\right)^{-2}w_{i}$. Hence
\begin{eqnarray*}
\mathrm{Tr}\left(B(z)^{-1}\right)-\mathrm{Tr}\left(B_{(i)}(z)^{-1}\right) & = & -\frac{\frac{d_{i}}{P}\mathrm{Tr}\left(\left(B_{(i)}(z)\right)^{-1}w_{i}w_{i}^{T}\left(B_{(i)}(z)\right)^{-1}\right)}{1+d_{i}g_{i}(z)}\\
 & = & -\frac{d_{i}g_{i}'(z)}{1+d_{i}g_{i}(z)},
\end{eqnarray*}
 where we used the cyclic property of the trace. We can now bound
this difference:
\begin{align*}
\left|\mathrm{Tr}\left(B(z)^{-1}\right)-\mathrm{Tr}\left(B_{(i)}(z)^{-1}\right)\right| & =\left|\frac{d_{i}g_{i}'(z)}{1+d_{i}g_{i}(z)}\right|\\
 & \leq\left|\frac{w_{i}^{T}\left(B_{(i)}(z)\right)^{-2}w_{i}}{w_{i}^{T}\left(B_{(i)}(z)\right)^{-1}w_{i}}\right|\\
 & \leq\max_{w}\left|\frac{w^{T}\left(B_{(i)}(z)\right)^{-2}w}{w^{T}\left(B_{(i)}(z)\right)^{-1}w}\right|\\
 & \leq\| \left(B_{(i)}(z)\right)^{-1}\| _{op} = \max_{j} |\frac{1}{\nu_j-z}|  \leq\frac{1}{d(z,\mathbb{R}^{+})},
\end{align*}
where $\nu_j$ are the eigenvalues of $\frac{1}{P}W_{(i)}K_{(i)}W_{(i)}^{T}$.

The sequence $$\left(\left(\mathbb{E}_{N-i}-\mathbb{E}_{N-i+1}\right)\left[\mathrm{Tr}\left(B(z)^{-1}\right)-\mathrm{Tr}\left(B_{(N-i+1)}(z)^{-1}\right)\right]\right)_{i=1,\ldots,N}$$ is a martingale difference sequence. Hence, by Burkholder's inequality, there exists a positive constant $K_{m}$ such that

\begin{equation*}\label{}
\begin{split}
\mathbb{E}\left[\left|m_{P}(z)-\mathbb{E}\left[m_{P}(z)\right]\right|^{m}\right]&\leq K_{m}\frac{1}{P^{m}}\mathbb{E}\left[\left(\sum_{i=1}^{N}\left|\left[\mathbb{E}_{i-1}-\mathbb{E}_{i}\right]\left(\mathrm{Tr}\left(B(z)^{-1}\right)-\mathrm{Tr}\left(B_{(i)}(z)^{-1}\right)\right)\right|^{2}\right)^{\frac{m}{2}}\right]\\
&\leq K_{m}\frac{1}{P^{m}}\left(N\left(\frac{2}{d(z,\mathbb{R}_{+})}\right)^{2}\right)^{\frac{m}{2}}\\
&\leq K_{m}\gamma^{-\frac{m}{2}}\left(\frac{2}{d(z,\mathbb{R}_{+})}\right)^{m}P^{-\frac{m}{2}},
\end{split}
\end{equation*}
hence the desired result with $\mathbf{c}=K_{m}\gamma^{-\frac m2}\left({2\over d(z,\mathbb R_{+})}\right)^{m}$.
\endproof

The following lemma, which is reminiscent of Lemma 4.5 in \cite{au-18}, is a consequence of Wick\textquoteright s formula
for Gaussian random variables and is key to prove Lemma C.4.
\begin{lemma} \label{fact:Wick-formula}
If $A^{(1)},\ldots,A^{(k)}$ are $k$ square random matrices of size
$P$ independent from a standard Gaussian vector $w$ of size $P$,
\begin{eqnarray}
\mathbb{E}\left[w^{T}A^{(1)}ww^{T}A^{(2)}w\ldots w^{T}A^{(k)}w\right] & = & \sum_{p\in\boldsymbol{P}_{2}(2k)}\sum_{\stackrel[p\leq{\rm Ker}(i_{1},\ldots,i_{2k})]{i_{1},\ldots,i_{2k}\in\{1,\ldots,P\}}{}}\mathbb{E}\left[A_{i_{1}i_{2}}^{(1)}\ldots A_{i_{2k-1}i_{2k}}^{(k)}\right],\label{eq:Wick-lem-1}
\end{eqnarray}
where $\boldsymbol{P}_{2}(2k)$ is the set of pair partitions of $\{1,\ldots,2k\}$,
$\leq$ is the coarser (i.e. $p\leq q$ if $q$ is coarser than $p$),
and for any $i_{1},\ldots,i_{2k}$ in $\{1,\ldots,P\}$, $\mathrm{Ker}(i_{1},\ldots,i_{2k})$
is the partition of $\left\{ 1,\ldots,2k\right\} $ such that two
elements $u$ and $v$ in $\left\{ 1,...,2k\right\} $ are in the
same block (i.e. pair) of $\mathrm{Ker}\left(i_{1},\ldots,i_{2k}\right)$
if and only if $i_{u}=i_{v}$.

Furthermore,
\begin{eqnarray}
\mathbb{E}\left[\left(w^{T}A^{(1)}w-\mathrm{Tr}\left(A^{(1)}\right)\right)\left(w^{T}A^{(2)}w-\mathrm{Tr}\left(A^{(2)}\right)\right)\ldots\left(w^{T}A^{(k)}w-\mathrm{Tr}\left(A^{(k)}\right)\right)\right] \nonumber \\
= \sum_{p\in:\boldsymbol{P}_{2}(2k):}\sum_{\stackrel[p\leq{\rm Ker}(i_{1},\ldots,i_{2k})]{i_{1},\ldots,i_{2k}\in\{1,\ldots,P\}}{}}\mathbb{E}\left[A_{i_{1}i_{2}}^{(1)}\ldots A_{i_{2k-1}i_{2k}}^{(k)}\right],\label{eq:Wick-lem-2}
\end{eqnarray}
 where $:\boldsymbol{P}_{2}(2k):$ is the subset of partitions $p$
in $\boldsymbol{P}_{2}(2k)$ for which $\left\{ 2j-1,2j\right\} $
is not a block of $p$ for any $j\in\{1,\ldots,k\}$.
\end{lemma}

\begin{proof}
Expanding the left-hand side of Equation (\ref{eq:Wick-lem-1}), we
obtain:
\[
\mathbb{E}\left[\sum_{i_{1},\ldots,i_{2k}\in\{1,\ldots,P\}}w_{i_{1}}A_{i_{1}i_{2}}^{(1)}w_{i_{2}}w_{i_{3}}A_{i_{3}i_{4}}^{(2)}w_{i_{4}}\ldots w_{i_{2k-1}}A_{i_{2k-1}i_{2k}}^{(k)}w_{i_{2k}}\right].
\]
 Using Wick's formula, we get:
\[
\sum_{i_{1},\ldots,i_{2k}\in\{1,\ldots,P\}}\sum_{\stackrel[p\leq{\rm Ker}(i_{1},\ldots,i_{2k})]{p\in\boldsymbol{P}_{2}(2k),}{}}\mathbb{E}\left[A_{i_{1}i_{2}}^{(1)}A_{i_{3}i_{4}}^{(2)}\ldots A_{i_{2k-1}i_{2k}}^{(k)}\right],
\]
hence, interchanging the order of summation, we recover the left-hand
side of Equation (\ref{eq:Wick-lem-1}):
\[
\sum_{p\in\boldsymbol{P}_{2}(2k)}\sum_{\stackrel[p\leq{\rm Ker}(i_{1},\ldots,i_{2k})]{i_{1},\ldots,i_{2k}\in\{1,\ldots,P\}}{}}\mathbb{E}\left[A_{i_{1}i_{2}}^{(1)}\ldots A_{i_{2k-1}i_{2k}}^{(k)}\right].
\]
 We now prove Equation (\ref{eq:Wick-lem-2}). Expanding the product,
the left-hand side is equal to:
\[
\sum_{I\subset\{1,\ldots,k\}}(-1)^{k-\#I}\mathbb{E}\left[\prod_{i\in I}w^{T}A^{(i)}w\prod_{i\notin I}\mathrm{Tr}(A^{(i)})\right].
\]
 Expanding the product and the trace, and using Wick's equation, we
obtain: a
\[
\sum_{I\subset\{1,\ldots,k\}}(-1)^{k-\#I}\sum_{i_{1},\ldots,i_{2k}\in\{1,\ldots,P\}}\sum_{\stackrel[p\leq{\rm Ker}(i_{1},\ldots,i_{2k})]{p\in\boldsymbol{P}_{2}(2k),p\leq p_{I}}{}}\mathbb{E}\left[A_{i_{1}i_{2}}^{(1)}\ldots A_{i_{2k-1}i_{2k}}^{(k)}\right].
\]
where $p_{I}$ is the partition composed of blocks of size $2$ given
by $\{2l,2l+1\}$ with $l\notin I$ and the rest of the indices contained
in a single block. Interchanging the order of summation, we get:

\[
\sum_{i_{1},\ldots,i_{2k}\in\{1,\ldots,P\}}\sum_{\stackrel[p\leq{\rm Ker}(i_{1},\ldots,i_{2k})]{p\in\boldsymbol{P}_{2}(2k),}{}}\mathbb{E}\left[A_{i_{1}i_{2}}^{(1)}\ldots A_{i_{2k-1}i_{2k}}^{(k)}\right]\left[\sum_{\stackrel[p\leq p_{I}]{I\subset\{1,\ldots,k\},}{}}(-1)^{k-\#I}\right].
\]
Since $\left[\sum_{I\subset\{1,\ldots,k\},\!p\leq p_{I}}(-1)^{\#I}\right]=\delta_{\{I\subset[k],p\leq p_{I}\}=\{\{1,\ldots,k\}\}}$
and $\{I\subset[k],p\leq p_{I}\}=\{\{1,\ldots,k\}\}$ if and only
if $p\in:\!\!\!\boldsymbol{P}_{2}(2k)\!\!\!:$, interchanging a last
time the order of summation, we recover the left-hand side of Equation
(\ref{eq:Wick-lem-2}):
\[
\sum_{p\in:\boldsymbol{P}_{2}(2k):}\sum_{\stackrel[p\leq{\rm Ker}(i_{1},\ldots,i_{2k})]{i_{1},\ldots,i_{2k}\in\{1,\ldots,P\}}{}}\mathbb{E}\left[A_{i_{1}i_{2}}^{(1)}\ldots A_{i_{2k-1}i_{2k}}^{(k)}\right].
\]
\end{proof}

For any $z\in\mathbb{C}\setminus\mathbb{R}^{+}$, we define the holomorphic function $g_{i}: \mathbb C\setminus \mathbb{R}^{+}\to \mathbb C$ by
\[
g_{i}(z)=\frac{1}{P}w_{i}^{T}\left(\frac{1}{P}W_{(i)}K_{(i)}W_{(i)}^{T}-z\ {I}_{P}\right)^{-1}w_{i},
\]
where $W_{(i)}$ is the $P\times(N-1)$ submatrix of $W$ obtained by removing its $i$-th column $w_{i}$, and $K_{(i)}$ is the $(N-1)\times(N-1)$ submatrix of $K$ obtained by removing both its $i$-th column and $i$-th row. In the following lemma, we bound the distance of $g_{i}(z)$ to its mean.  Then we prove that $\mathbb{E}[g_{i}(z)]$ is close to the expected Stieljes transform of $K$.
\begin{lemma}
\label{lem:concentration_gi}
The random function $g_{i}(z)$ satisfies:
\begin{eqnarray*}
\left|\mathbb{E}\left[g_{i}(z)\right]-\mathbb{E}\left[m_{P}(z)\right]\right| & \leq & \frac{\mathbf{c_{0}}}{P},\\
\mathrm{Var}\left(g_{i}(z)\right) & \leq & \frac{\mathbf{c_{1}}}{P},\\
\mathbb{E}\left[\left(g_{i}(z)-\mathbb{E}\left[g_{i}(z)\right]\right)^{4}\right] & \leq & \frac{\mathbf{c_{2}}}{P^{2}},\\
\mathbb{E}\left[\left(g_{i}(z)-\mathbb{E}\left[g_{i}(z)\right]\right)^{8}\right] & \leq & \frac{\mathbf{c_{3}}}{P^{4}},
\end{eqnarray*}
where $\mathbf{c_{0}}$, $\mathbf{c_{1}}$, $\mathbf{c_{2}}$, and  $\mathbf{c_{3}}$ depend on $\gamma$ and $z$ only.
\end{lemma}
\proof
The random variable $w_{i}$ is independent from $B_{(i)}(z)=\frac{1}{P}W_{(i)}K_{(i)}W_{(i)}^{T}-z\mathrm{I}_{P}$ since the $i$-th column of $W$ does not appear in the definition of $B_{(i)}(z)$. Using Lemma \ref{fact:Wick-formula}, since there exists a unique pair partition $p\in\boldsymbol{P}_{2}(2)$, namely $\{\{1,2\}\}$, the expectation of $g_{i}(z)$ is given by
\[
\mathbb{E}\left[g_{i}(z)\right]=\frac{1}{P}\mathbb{E}\left[\mathrm{Tr}\left[B_{(i)}(z)^{-1}\right]\right].
\]
Recall that $\mathbb{E}\left[m_{P}(z)\right]=\frac{1}{P}\mathbb{E}\left[\mathrm{Tr}\left[B(z)^{-1}\right]\right]$ and $\left|\mathrm{Tr}\left(B(z)^{-1}\right)-\mathrm{Tr}\left(B_{(i)}(z)^{-1}\right)\right|\leq\frac{1}{d(z,\mathbb{R}_{+})}$ (from the proof of Lemma \ref{lem:concentration_stieltjes}). Hence
\[
\left|\mathbb{E}\left[g_{i}(z)\right]-\mathbb{E}\left[m_{P}(z)\right]\right|\leq\frac{1}{P}\mathbb{E}\left[\left|\mathrm{Tr}\left(B(z)^{-1}\right)-\mathrm{Tr}\left(B_{(i)}(z)^{-1}\right)\right|\right]\leq\frac{1}{P}\frac{1}{d(z,\mathbb{R}_{+})}.
\]
which proves the first assertion with $\mathbf{c_0}=\frac{1}{d(z,\mathbb{R}_{+})}.$

Now, let us consider the variance of $g_{i}(z)$. Using our previous computation of $\mathbb{E}\left[g_{i}(z)\right]$,
we have
\begin{eqnarray*}
\mathrm{Var}(g_{i}(z)) & = & \mathbb{E}\left[w_{i}^{T}\frac{\left(B_{(i)}(z)\right)^{-1}}{P}w_{i}w_{i}^{T}\frac{\left(B_{(i)}(z)\right)^{-1}}{P}w_{i}\right]-\mathbb{E}\left[\frac{1}{P}\mathrm{Tr}\left[B_{(i)}(z)^{-1}\right]\right]^{2}.
\end{eqnarray*}
 The first term can be computed using the first assertion of Lemma \ref{fact:Wick-formula}: there are $2$ matrices involved, thus we have to sum over $3$ pair partitions. A simplification arises since $\frac{\left(B_{(i)}(z)\right)^{-1}}{P}$ is symmetric: the partition $\{\{1,2\},\{3,4\}\}$ yields $\mathbb{E}\left[\left(\mathrm{Tr}\left[\frac{\left(B_{(i)}(z)\right)^{-1}}{P}\right]\right)^{2}\right]$ whereas both $\{\{1,3\},\{2,4\}\}$ and $\{\{1,4\},\{2,4\}\}$ yield $\mathbb{E}\left(\mathrm{Tr}\left[\frac{\left(B_{(i)}(z)\right)^{-2}}{P^{2}}\right]\right)$.

Thus, the variance of $g_{i}(z)$ is given by:
\[
\mathrm{Var}(g_{i}(z))=2\mathbb{E}\left(\mathrm{Tr}\left[\frac{\left(B_{(i)}(z)\right)^{-2}}{P^{2}}\right]\right)+\mathbb{E}\left[\left(\frac{1}{P}\mathrm{Tr}\left[\left(B_{(i)}(z)\right)^{-1}\right]\right)^{2}\right]-\mathbb{E}\left[\frac{1}{P}\mathrm{Tr}\left[\left(B_{(i)}(z)\right)^{-1}\right]\right]^{2}
\]
 hence is given by a sum of two terms:
\[
\mathrm{Var}(g_{i}(z))=\frac{2}{P}\mathbb{E}\left(\frac{1}{P}\mathrm{Tr}\left[\left(B_{(i)}(z)\right)^{-2}\right]\right)+\mathrm{Var}\left(\frac{1}{P}\mathrm{Tr}\left[\left(B_{(i)}(z)\right)^{-1}\right]\right).
\]
Using the same arguments as those explained for the bound on the Stieltjes transform, the first term is bounded by $\frac{2}{Pd(z,\mathbb{R}_{+})^{2}}$. In order to bound the second term, we apply Lemma \ref{lem:concentration_stieltjes} for $W_{(i)}$ and $K_{(i)}$ in place of $W$ and $K$. The second term is bounded by $\frac{\mathbf{c}}{P}$, hence the bound $\mathrm{Var}\left(g_{i}(z)\right)\leq\frac{\mathbf{c_{1}}}{P}.$

Finally, we prove the bound on the fourth moment of $g_{i}(z)-\mathbb{E}\left[g_{i}(z)\right]$. We denote $m_{(i)}(z)=\frac{1}{P}\mathrm{Tr}\left[\left(B_{(i)}(z)\right)^{-1}\right]$. Recall that $\mathbb{E}\left[g_{i}(z)\right]=\mathbb{E}\left[m_{(i)}(z)\right]$. Using the convexity of $t\mapsto t^{4}$, we have
\begin{align*}
\mathbb{E}\left[\left(g_{i}(z)-\mathbb{E}[g_{i}(z)]\right)^{4}\right] & =\mathbb{E}\left[\left(g_{i}(z)-m_{(i)}(z)+m_{(i)}(z)-\mathbb{E}\left[m_{(i)}(z)\right]\right)^{4}\right]\\
& \leq8\mathbb{E}\left[\left(g_{i}(z)-m_{(i)}(z)\right)^{4}\right]+8\mathbb{E}\left[\left(m_{(i)}(z)-\mathbb{E}\left[m_{(i)}(z)\right]\right)^{4}\right].
\end{align*}
We bound the second term using the concentration of the Stieljes transform (Lemma \ref{lem:concentration_stieltjes}): it is bounded by $\frac{8\mathbf{c}}{P^{2}}$. The first term is bounded using the second assertion of Lemma \ref{fact:Wick-formula}. Using the symmetry of $B_{(i)}(z)$, the partitions in $:\boldsymbol{P}_{2}(4):$
yield two different terms, namely:
\begin{enumerate}
\item $\frac{1}{P^{2}}\mathbb{E}\left[\left(\frac{1}{P}\mathrm{Tr}\left[\left(B_{(i)}(z)\right)^{-2}\right]\right)^{2}\right]$,
for example if $p=\{\{1,3\},\{2,4\},\{5,7\},\{6,8\}\}$
\item $\frac{1}{P^{3}}\mathbb{E}\left[\frac{1}{P}\mathrm{Tr}\left[\left(B_{(i)}(z)\right)^{-4}\right]\right]$,
for example if $p=\{\{2,3\},\{4,5\},\{6,7\},\{8,1\}\}$.
\end{enumerate}
We bound the two terms using the same arguments as those explained for the bound on the Stieljes transform at the beginning of the section. The first term is bounded by $\frac{d(z,\mathbb{R^{+}})^{-4}}{P^{2}}$ and the second term by $\frac{d(z,\mathbb{R^{+}})^{-4}}{P^{3}}$ hence the bound $\mathbb{E}\left[\left(g_{i}(z)-\mathbb{E}\left[g_{i}(z)\right]\right)^{4}\right]\leq\frac{\mathbf{c_{2}}}{P^{2}}.$

The bound $\mathbb{E}[\left(g_{i}(z)-\mathbb{E}\left[g_{i}(z)\right]\right)^{8}]  \leq  \frac{\mathbf{c_{3}}}{P^{4}}$ is obtained in a similar way, using the second assertion of Lemma \ref{fact:Wick-formula} and simple bounds on the Stieljes transform.
\endproof

In the next proposition we show that the Stieltjes transform $m_{P}(z)$ is close in expectation to the solution of a fixed point equation.

\begin{proposition}\label{prop:convergence_stieltjes}
For any $z\in\mathbb{H}_{<0}=\left\{ z:\mathrm{Re}(z)<0\right\} ,$
\begin{align*}
\left|\mathbb{E}\left[m_{P}(z)\right]-\tilde{m}(z)\right| & \leq\frac{\mathbf{e}}{P},
\end{align*}
where $\mathbf{e}$ depends on $z$, $\gamma$, and $\frac{1}{N}\mathrm{Tr}(K)$ only and where $\tilde{m}(z)$ is the unique solution in the cone $\mathcal C_{z}:=\{u -\frac 1z v: u, v \in \mathbb R_{+} \}$ spanned by
$1$ and $-\frac{1}{z}$ of the equation
\[
\gamma=\frac{1}{N}\sum_{i=1}^{N}\frac{d_{i}\tilde{m}(z)}{1+d_{i}\tilde{m}(z)}-\gamma z\tilde{m}(z).
\]
\end{proposition}
\begin{proof}
We use the same notation as in the previous proofs, namely $B(z)=\frac{1}{P}WKW^{T}-z\mathrm{I}_{P}$,  $B_{(i)}(z)=\frac{1}{P}W_{(i)}K_{(i)}W_{(i)}^{T}-z\mathrm{I}_{P}$ and $g_{i}(z)=\frac{1}{P}w_{i}^{T}\left(B_{(i)}(z)\right)^{-1}w_{i}$.
Let $\nu_{j}\geq0,\ j=1,\dots,P$ be the spectrum of the positive semi-definite matrix $\frac{1}{P}W_{(i)}K_{(i)}W_{(i)}^{T}$. After diagonalization, we have $$B_{(i)}(z)^{-1}=O^{T} \mathrm{diag}(\frac{1}{\nu_1-z},\ldots,\frac{1}{\nu_P-z}) O,$$ with $O$ an orthogonal matrix. Then
\begin{equation}\label{}
\begin{split}
g_{i}(z)=\frac{1}{P}\mathrm{Tr} \left(\left(B_{(i)}(z)\right)^{-1}w_{i}w_{i}^{T}\right)=\frac1P \sum_{j=1}^{P}\frac{((Ow_{i})_{jj})^{2}}{\nu_{j}-z} .
\end{split}
\end{equation}
Since $z\in \mathbb{H}_{<0}$, we conclude that $\Re [g_{i}(z)]\geq0$ for all $i=1,\dots,P$.

In order to prove the proposition, the key remark is that, since $\mathrm{Tr}\left((\frac{1}{P}WKW^{T}-z\mathrm{I}_{P})(B(z))^{-1}\right)=P$, the Stieltjes transform $m_{P}(z)$ satisfies the following equation:
\begin{align*}
P =\mathrm{Tr}\left(\frac{1}{P}KW^{T}B(z)^{-1}W\right)-zPm_{P}(z).
\end{align*}
From the proof of Lemma \ref{lem:concentration_stieltjes}, recall that $B^{-1}(z)=B_{(i)}^{-1}(z)-\frac{d_{i}}{P}\frac{1}{1+\frac{d_{i}}{P}w_{i}^{T}B_{(i)}^{-1}(z)w_{i}}B_{(i)}^{-1}(z)w_{i}w_{i}^{T}B_{(i)}^{-1}(z),$ hence:
\begin{equation}\label{eq:entry-A-matrix}
\begin{split}
\frac{1}{P}w_{i}^{T}B^{-1}(z)w_{i}&=g_{i}(z)-\frac{d_{i}g_{i}(z)^{2}}{1+d_{i}g_{i}(z)}\\&=\frac{g_{i}(z)}{1+d_{i}g_{i}(z)}.
\end{split}
\end{equation}
Expanding the trace,
\[
\mathrm{Tr}\left(\frac{1}{P}KW^{T}B(z)^{-1}W\right)=\sum_{i=1}^{N}d_{i}\frac{1}{P}w_{i}^{T}B^{-1}(z)w_{i}=\sum_{i=1}^{N}\frac{d_{i}g_{i}(z)}{1+d_{i}g_{i}(z)}.
\]
Thus, the Stieljes transform $m_{P}(z)$ satisfies the following equation $P=\sum_{i=1}^{N}\frac{d_{i}g_{i}(z)}{1+d_{i}g_{i}(z)}-zPm_{P}(z),$ or equivalently
\[
\gamma=\frac{1}{N}\sum_{i=1}^{N}\frac{d_{i}g_{i}(z)}{1+d_{i}g_{i}(z)}-z\gamma m_{P}(z).
\]

Recall that $\gamma>0$ and $\mathrm{Re}(z)<0$. The Stieljes transform $m_{P}(z)$ can be written as a function of $g_{i}(z)$ for $i=1,\ldots,n$: $m_{P}(z)=f(g_{1}(z),...,g_{N}(z))$ where
\[
f(g_{1},\ldots,g_{N})=\frac{1}{\gamma zN}\sum_{i=1}^{N}\frac{d_{i}g_{i}}{1+d_{i}g_{i}}-\frac{1}{z}=-\frac{1}{z}\left(1-\frac{1}{\gamma}+\frac{1}{\gamma}\frac{1}{N}\sum_{i=1}^{N}\frac{1}{1+d_{i}g_{i}}\right).
\]
From Lemma \ref{lem:unique-fixpoint}, the map $f(m)=f(m,...,m)$ has a unique non-degenerate fixed point $\tilde{m}(z)$ in the cone $\mathcal C_{z}$. We will show that $\mathbb{E}\left[m_{P}(z)\right]$ is close to $\tilde{m}(z)$ using the following two steps: we show a non-tight bound $\left|\mathbb{E}\left[m_{P}(z)\right]-\tilde{m}(z)\right|\leq\frac{\mathbf{e}'}{\sqrt{P}}$ and use it to obtain the tighter bound $\left|\mathbb{E}[m_{P}(z)]-\tilde{m}(z)\right|\leq\frac{\mathbf{e}}{P}$.

Let us prove the $\frac{\mathbf{e}'}{\sqrt{P}}$ bound. From Lemma \ref{lem:unique-fixpoint}, the distance between $m_{P}(z)$ and the fixed point $\tilde{m}(z)$ of $f$ is bounded by the distance between $f(m_{P}(z),\ldots,m_{P}(z))$ and $m_{P}(z)$ . Using the fact that $m_{P}(z)=f(g_{1}(z),...,g_{N}(z))$, we obtain
\[
\left|\mathbb{E}[m_{P}(z)]-\tilde{m}(z)\right|\leq\mathbb{E}\left[\left|m_{P}(z)-\tilde{m}(z)\right|\right]\leq\mathbb{E}\left[\left|f(m_{P}(z),\ldots,m_{P}(z))-f(g_{1}(z),...,g_{N}(z))\right|\right].
\]

Recall that for any $z\in \mathbb H_{<0}$, $\Re(g_{i}(z))\geq0$: we need to study the function $f$ on $\mathbb H_{\geq 0}^N$ where $\mathbb H_{\geq 0} = \{z\in \mathbb{C}| \Re(z)\geq 0\}$. On $\mathbb H_{\geq 0}^N$, the function $f$ is Lipschitz: \begin{align*}
\left|\partial_{g_{i}}f(g_{1},..,g_{N})\right| & =\left|\frac{1}{\gamma zN}\frac{d_{i}}{(1+d_{i}g_{i})^{2}}\right| \leq\frac{d_{i}}{\gamma\left|z\right|N}.\\
\end{align*}
Thus,
\[
\mathbb{E}\left[\left|f\left(m_{P}(z),...,m_{P}(z)\right)-f\left(g_{1}(z),...,g_{N}(z)\right)\right|\right]\leq\sum_{i=1}^{N}\frac{d_{i}}{\gamma\left|z\right|N}\mathbb{E}\left[\left|m_{P}(z)-g_{i}(z)\right|\right].
\]
Since $$\mathbb{E}\left[\left|m_{P}(z)-g_{i}(z)\right|\right]\leq\mathbb{E}\left[\left|m_{P}(z)-\mathbb{E}\left[m_{P}(z)\right]\right|\right]+\left|\mathbb{E}\left[m_{P}(z)\right]-\mathbb{E}\left[g_{i}(z)\right]\right|+\mathbb{E}\left[\left|g_{i}(z)-\mathbb{E}\left[g_{i}(z)\right]\right|\right],$$
using Lemmas \ref{lem:concentration_stieltjes} and \ref{lem:concentration_gi},
we get that $\mathbb{E}\left[\left|m_{P}(z)-g_{i}(z)\right|\right]\leq\frac{\mathbf{d}}{\sqrt{P}}$, where $\mathbf{d}$ depends on $\gamma$ and $z$ only. This implies that
\[
\mathbb{E}\left[\left|f\left(m_{P}(z),...,m_{P}(z)\right)-f\left(g_{1}(z),...,g_{N}(z)\right)\right|\right]\leq\frac{1}{\sqrt{P}}\frac{\mathbf{d}}{N}\mathrm{Tr}\left(K\right),
\]
 which allows to conclude that $\left|\mathbb{E}[m_{P}(z)]-\tilde{m}(z)\right|\leq\frac{\mathbf{e}'}{\sqrt{P}}$ where $\mathbf{e}'$ depends on $\gamma$, $z$ and $\frac{1}{N}\mathrm{Tr}(K)$ only.

We strengthen this inequality and show the $\frac{\mathbf{e}}{P}$ bound. Using again Lemma \ref{lem:unique-fixpoint}, we bound the distance between $\mathbb{E}[m_{P}(z)]$ and the fixed point $\tilde{m}(z)$ by
\[
|\mathbb{E}[m_{P}(z)] - \tilde{m}(z) |\leq \left|\mathbb{E}[f(g_1(z),\ldots,g_N(z))]-f(\mathbb{E}[m_{P}(z)],\ldots,\mathbb{E}[m_{P}(z)])\right|
\]
and study the r.h.s. using a Taylor approximation of $f$ near $\mathbb{E}\left[m_{P}(z)\right]$. For $i=1,\ldots,N$ and $m_{0}\in \mathbb H_{\geq 0}$, let $\mathrm{T}_{m_{0}}h_{i}$ be the first order Taylor approximation of the map $h_{i}:m\mapsto\frac{1}{1+d_{i}m}$ at a point $m_{0}$. The error of the first order Taylor approximation is given by
\begin{align*}
h_{i}(m)-\mathrm{T}_{m_{0}}h_{i}(m) & =\frac{1}{1+d_{i}m}-\left(\frac{1}{1+d_{i}m_{0}}-\frac{d_{i}(m-m_{0})}{\left(1+d_{i}m_{0}\right)^{2}}\right)=\frac{d_{i}^{2}\left(m_{0}-m\right)^{2}}{\left(1+d_{i}m\right)\left(1+d_{i}m_{0}\right)^{2}},
\end{align*}
which, for $m\in \mathbb H_{\geq 0}$ can be upper bounded by a quadratic term:
\begin{equation}\label{eq:taylor}
\begin{split}
\left|h_{i}(m)-\mathrm{T}_{m_{0}}h_{i}(m)\right|=\left|\frac{d_{i}^{2}}{\left(1+d_{i}m\right)\left(1+d_{i}m_{0}\right)^{2}}\right|\left|m_{0}-m\right|^{2}\leq\frac{1}{\left|m_{0}\right|^{2}}\left|m_{0}-m\right|^{2}.
\end{split}
\end{equation}

The first order Taylor
approximation $\mathrm{T}f$ of $f$ at the $N$-tuple $(\mathbb{E}\left[m_{P}(z)\right],...,\mathbb{E}\left[m_{P}(z)\right])$
is
\[
\mathrm{T}f(g_{1},..,g_{N})=-\frac{1}{z}\left(1-\frac{1}{\gamma}+\frac{1}{\gamma}\frac{1}{N}\sum_{i=1}^{N}\mathrm{T}_{\mathbb{E}\left[m_{P}(z)\right]}h_{i}(g_{i})\right).
\]
 Using this Taylor approximation, $\mathbb{E}[f(g_1(z),\ldots,g_N(z))]-f(\mathbb{E}[m_{P}(z)],\ldots,\mathbb{E}[m_{P}(z)])$ is equal to:
\[
\mathbb{E}\left[\mathrm{T}f(g_{1}(z),..,g_{N}(z))\right]- f(\mathbb{E}[m_{P}(z)],\ldots,\mathbb{E}[m_{P}(z)]) +\mathbb{E}\left[f(g_{1}(z),...,g_{N}(z))-\mathrm{T}f(g_{1}(z),..,g_{N}(z))\right].
\]
Using Lemma \ref{lem:concentration_gi}, we get
\begin{align*}
\left|\mathbb{E}\left[f(g_{1}(z),...,g_{N}(z))-\mathrm{T}f(g_{1}(z),..,g_{N}(z))\right]\right| & \leq\frac{1}{\left|z\right|\gamma}\frac{1}{N}\sum_{i=1}^{N}\frac{1}{\left|\mathbb{E}[m_{P}(z)]\right|^{2}}\mathbb{E}\left[\left|g_{i}(z)-\mathbb{E}\left[m_{P}(z)\right]\right|^{2}\right]\\
&\leq \frac 1P \frac{\alpha}{\left|\mathbb{E}[m_{P}(z)]\right|^{2}}\,
\end{align*}
and
\begin{align*}
\left|\text{\ensuremath{\mathbb{E}}}\left[\mathrm{T}f(g_{1}(z),..,g_{N}(z))\right]-f(\mathbb{E}\left[m_{P}(z)\right],...,\mathbb{E}\left[m_{P}(z)\right])\right| & \leq\frac{1}{\left|z\right|\gamma}\frac{1}{N}\sum_{i=1}^{N}\frac{d_{i}\left|\mathbb{E}\left[g_{i}\right]-\mathbb{E}\left[m_{P}(z)\right]\right|}{\left|1+d_{i}\mathbb{E}\left[m_{P}(z)\right]\right|^{2}}\\
 & \leq\frac{\beta\left(\frac{1}{N}\mathrm{Tr}K\right)}{P}
\end{align*}
where $\alpha$ and $\beta$ depends on $z$ and $\gamma$ only. From the bounds $\left|\mathbb{E}[m_{P}(z)]-\tilde{m}(z)\right|\leq\frac{\mathbf{e}'}{\sqrt{P}}$ and $\left|\tilde{m}(z)\right|\geq (|z|+\frac{1}{N\gamma}\mathrm{Tr}(K))^{-1}$ (Lemma \ref{lem:unique-fixpoint}), the bound $\frac 1 P \frac{\alpha}{\left|\mathbb{E}[m_{P}(z)]\right|^{2}}$ yields a $\frac{\tilde\alpha}{P}$ bound. This implies that $\left|\mathbb{E}[m_{P}(z)]-f(\mathbb{E}[m_{P}(z)],\ldots,\mathbb{E}[m_{P}(z)])\right|\leq\frac{\mathbf{e}}{P}$,
hence the desired inequality $\left|\mathbb{E}\left[m_{P}(z)\right]-\tilde{m}(z)\right|\leq\frac{\mathbf{e}}{P}.$
\end{proof}

For the proof of Proposition \ref{prop:convergence_stieltjes}, we have used the fact that the map $f_{z}$ introduced therein has a unique non-degenerate fixed point in the cone $\mathcal C_{z}:=\{u -\frac 1z v: u, v \in \mathbb R_{+} \}$. We now proceed with proving this statement.
\begin{lemma}\label{lem:unique-fixpoint}
Let $d_{1},\dots, d_{n}\geq0$ and let $\gamma\geq0$. For any fixed $z\in \mathbb H_{<0}$ , let $f_{z}:\mathbb H_{\geq 0} \to \mathbb C$ be the function $t \mapsto f_{z}(t)=-\frac{1}{z}\left(1-\frac{1}{\gamma}\frac{1}{N}\sum_{i=1}^{N}\frac{d_{i} t}{1+d_{i} t}\right)$. Let $\mathcal C_{z}:=\{u -\frac 1z v: u, v \in \mathbb R_{+} \}$ be the convex region spanned by the half-lines $\mathbb R_{+}$ and $-\frac1z \mathbb R_{+}$. Then for every $z\in \mathbb H_{<0}$ there exists a unique fixed point $\tilde t(z)\in \mathcal C_{z}$ such that $\tilde t(z)=f_{z}(\tilde t(z))$.
The map $\tilde t:z\mapsto\tilde t(z)$ is holomorphic in $\mathbb H_{<0}$ and $$|\tilde t(z)| \geq  \left(|z|+{\sum_{i}d_{i}\over \gamma N}\right)^{-1}.$$ Furthermore for every $z\in \mathbb H_{<0}$ and any $t\in \mathbb H_{\geq 0}$, one has $$ |t-\tilde t(z)|\leq |t- f_{z}(t)|.$$
\end{lemma}
\proof
By means of Schwarz reflection principle, we can assume that $\Im (z) \geq0$. Let $z\in \mathbb H_{<0}$ and let $\Pi_{z}:=\{-\frac wz : \Im(w)\leq0\}$ and let $\mathcal C_{z}$ be the wedged region $\mathcal C_{z}:=\Pi_{z}\cap \{ w\in \mathbb C: \Im(w)\geq0\}$. To show the existence of a fixed point in $\mathcal C_{z}$ we show that $0$ is in the image of the function $\psi:t\mapsto f_{z}(t)-t$. Note that since $d_{i}\geq0$, the eventual poles of $f_{z}$ are all strictly negative real numbers, hence $\psi:\mathcal C_{z}\to\mathbb C$ is an holomorphic function.

To prove that $0\in\psi(\mathcal C_{z})$ we proceed with a geometrical reasoning: the image $\psi(\mathcal C_{z})$ is (one of) the region of the plane confined by $\psi\left(\partial \mathcal C_{z}\right)$, so we only need to ``draw'' $\psi\left(\partial \mathcal C_{z}\right)$ and show that $0$ belongs to the ``good'' connected component confined by it.

The boundary of $C_z$ is made up of two half-lines $\RR_+$ and $-\frac 1 z \RR_+$. Under the map $f_z$, $0$ is mapped to $-\frac 1 z$ and $\infty$ is mapped to $-\frac{1-\frac 1 \gamma}{z} $, the two half-lines are hence mapped to paths from $-\frac 1 z$ to $-\frac{1-\frac 1 \gamma}{z} $. Now under $\psi$ the half-lines will be mapped to paths going $-\frac 1 z$ to $\infty$ because by our assumption $-\frac 1 z$ lies in the upper right quadrant, we will show that the image of $\RR_+$ under $\phi$ goes 'above' the origin while the image of $-\frac 1 z \RR_+$ goes 'under' the origin:
\begin{itemize}
\item $\RR_+$ is mapped under $f_z$ to the segment  $-\frac1z[1,1-\frac1\gamma]$, as a result, its map under $\psi$ lies in the Minkowski sum  $-\frac1z[1,1-\frac1\gamma] + (-\RR_+)$ which is contained in $\overline{\mathbb C  \setminus \Pi_z}$.
\item For any $t \in -\frac 1 z \RR_+$ we have for all $d_i$
\[
\Im \left(\frac{d_it}{1+d_i t}\right) = \Im \left(1-  \frac{1}{1+d_i t} \right) = \Im \left(\frac{1}{1+d_i t}\right) \leq 0,
\]
since $\Im (t) \geq 0$. As a result the image of $-\frac 1 z \RR_+$ under $f_z$ lies in $\Pi_z$ and its image under $\psi$ lies in the Minkovski sum $\Pi_z + (-\frac 1 z \RR_+)=\Pi_z$.
\end{itemize}
Thus we can conclude that $0\in \psi\left(\mathcal C_{z}\right)$, which shows that there exists at least a fixed point $\tilde m$ in $\mathcal C_{z}$.

We observe that, for every $t\in \mathcal C_{z}$, the derivative of $f$ has negative real part:
\begin{align*}
\mathrm{Re}\left(f_{z}'(t) \right)& =\frac{1}{\gamma}\frac{1}{N}\sum_{i=1}^{N}\mathrm{Re}\left( \frac{d_{i}}{z\left(1+d_{i}t\right)^{2}}\right)\\
 & =\frac{1}{\gamma}\frac{1}{N}\sum_{i=1}^{N}\frac{d_{i}\left[\Re (z)+2d_{i}\Re (z)\Re (t)-2d_{i}\Im (z)\Im (t)+d_{i}^{2}\Re (zt^{2}) \right]}{\left|z\right|^{2}\left|1+d_{i}t\right|^{4}}\leq 0,
\end{align*}

where we concluded the last inequality by using that $\Re (z) \leq 0$, $\Re(t) \geq0 $, $\Im (z) \Im (t) \geq0$ and $\Re(zt^{2})\leq 0$. Thus, since for no point $t\in \mathcal C_{z}$ has $f_{z}'(t)=1$, any fixed point of $f_{z}$ is a simple fixed point.

We now proceed to show the uniqueness of the fixed point in the region $\mathcal C_{z}$.
Suppose there are two fixed points $t_{1}$ and $t_{2}$, then
\begin{equation*}\label{}
\begin{split}
t_{1}-t_{2}&=f_{z}(t_{1})-f_{z}(t_{2}) \\
&=\left(t_{1}-t_{2}\right)\frac1z\frac{1}{\gamma N}\sum_{i=1}^{N}{d_{i}\over (1+d_{i} t_{1})(1+d_{i} t_{2})}.
\end{split}
\end{equation*}
Again, since $\Re (z) \leq 0$, $\Re(t_{1}), \Re(t_{2}) \geq0 $, $\Im (z) \Im( t_{1}),\Im (z) \Im( t_{2}), \geq0$ and $\Re(z t_{1}t_{2})\leq0$, the factor $\frac1z\frac1N\sum_{i=1}^{N}{d_{i}\over (1+d_{i} t_{1})(1+d_{i} t_{2})}$ has negative real part, and thus the identity is possible only if $t_{1}=t_{2}$. Let's then $\tilde t(z)$ be the only fixed point in $\mathcal C_{z}$.

We proceed now to show that $|t-f_{z}(t)|\geq |t-\tilde t(z)|$, i.e. if $t$ and its image are close, then $t$ is not too far from being a fixed point, and so it is close to $\tilde t(z)$.

For any $t\in \mathcal C_{z}$, we have
\begin{equation*}\label{}
\begin{split}
|t-f_{z}(t)|&=|t-\tilde t(z)+ f_{z}(\tilde t (z)) - \tilde f_{z}(t)|\\
& = \left|(t-\tilde t(z)) - \left(t-\tilde t(z)\right)\left(\frac1z\frac{1}{\gamma N}\sum_{i=1}^{N}{d_{i}\over (1+d_{i} t)(1+d_{i} \tilde t(z))}\right)\right| \\
& = \left|t-\tilde t(z)\right|\left|1- \frac1z\frac{1}{\gamma N}\sum_{i=1}^{N}{d_{i}\over (1+d_{i} t)(1+d_{i} \tilde t(z))}\right| \\
& \geq \left|t-\tilde t(z)\right|
\end{split}
\end{equation*}
where we have used again that $\frac1z\frac1N\sum_{i=1}^{N}{d_{i}\over (1+d_{i} t)(1+d_{i} \tilde t(z))}$ has negative real part.

We provide a lower bound on the norm of the fixed point:
\begin{align*}
\left|\tilde t(z)\right| & =\frac{1}{\left|z\right|}\left|1-\frac{1}{\gamma}\frac{1}{N}\sum_{i=1}^{N}\frac{d_{i}\tilde t(z)}{1+d_{i}\tilde t(z)}\right|\geq\frac{1}{\left|z\right|}\left(1-\frac{1}{\gamma}\frac{1}{N}\sum_{i=1}^{N}\left|\frac{d_{i}\tilde t(z)}{1+d_{i}\tilde t(z)}\right|\right) \geq\frac{1}{\left|z\right|}\left(1-\frac{\left|\tilde t(z)\right|}{\gamma N}\sum_{i=1}^{N} d_{i}\right).\\
\end{align*}
hence
$$|\tilde t(z)| \geq  \left(|z|+{\sum_{i}d_{i}\over \gamma N}\right)^{-1}.$$

Finally, note that $z$ can be expressed from the fixed point $\tilde m$, hence defining an inverse for the map $\tilde t$:
\[
\tilde t^{-1}(\tilde m)=z=-\frac{1}{\tilde m}\left(1-\frac{1}{\gamma}\frac{1}{N}\sum_{i=1}^{N}\frac{d_{i} \tilde m}{1+d_{i} \tilde m}\right)
\]
because the inverse is holomorphic, so is $\tilde t$.
\endproof

\subsection{Ridge}
Using Proposition \ref{prop:distribution-estimator}, in order to have a better description of the distribution of the predictor $\hat{f}_{\lambda,\gamma}^{(RF)}$,  it remains to study the distributions of both the final labels $\hat{y}$ on the training set and the parameter norm $\|\hat{\theta}\| ^{2}$. In Section \ref{subsubsec:expectation-predictor}, we first study the expectation of the final labels $\hat{y}$: this allows us to study the loss of the average predictor $\mathbb{E}\left[\hat{f}_{\lambda,\gamma}^{(RF)}\right]$.
Then in Section \ref{subsec:Variance}, a study of the variance of the predictor allows us to study the average loss of the RF predictor.

\subsubsection{Expectation of the predictor}\label{subsubsec:expectation-predictor}

The optimal parameters $\hat{\theta}$ which minimize the regularized
MSE loss is given by $\hat{\theta}=F^{T}(FF^{T}+\lambda \mathrm I_{N})^{-1}y$,
or equivalently by $\hat{\theta}=(F^{T}F+\lambda)^{-1}F^{T}y$.
Thus, the final labels take the form $\hat{y}=A(-\lambda)y$ where
$A(z)$ is the random matrix defined as
\begin{align*}
A(z) & :=F\left(F^{T}F-z\mathrm I_{P}\right)^{-1}F^{T}\\
 & =\frac{1}{P}K^{\frac{1}{2}}W^{T}\left(\frac{1}{P}WKW^{T}-z\mathrm I_{P}\right)^{-1}WK^{\frac{1}{2}}.
\end{align*}
Note that the matrix $A_{\lambda}$ defined in the proof sketch of Theorem 4.1 in the main text is given by $A_{\lambda}=A(-\lambda)$.

\begin{proposition}
\label{prop:ridge_expectation}For any $\gamma>0$, any $z\in\mathbb{H}_{<0}$,
and any symmetric positive definite matrix $K$,
\begin{equation}
\| \mathbb{E}\left[A(z)\right]-K(K+\tilde{\lambda}(-z)I_{N})^{-1} \|_{op}\leq\frac{c}{P},\label{eq:norm-op}
\end{equation}
where $\tilde{\lambda}(z):=\frac {1}{\tilde{m}(-z)}$ and $c>0$ depends on $z$, $\gamma$ and $\frac{1}{N}Tr(K)$ only.
\end{proposition}

 \begin{proof}
 Since the distribution of $W$ is invariant under orthogonal transformations,
by applying
a change of basis, in order to prove Inequality (\ref{eq:norm-op}), we may assume that $K$ is diagonal with diagonal
entries $d_{1},\ldots,d_{N}$. Denoting $w_{1},\ldots,w_{N}$ the
columns of $W$, for any $i,j=1,\ldots,N$,
\[
(A(z))_{ij}=\frac{1}{P}\sqrt{d_{i}d_{j}}w_{i}^{T}\left(\frac{1}{P}WKW^{T}-zI_{P}\right)^{-1}w_{j},
\]
where $WKW^{T}=\sum_{i=1}^{N}d_{i}w_{i}w_{i}^{T}$. Replacing $w_{i}$
by $-w_{i}$ does not change the law $W$ hence does not change
the law of $(A(z))_{ij}$. Since $WKW^{T}$ is invariant under this
change of sign, we get that for $i\neq j$, $\mathbb{E}\left[(A(z))_{ij}\right]=-\mathbb{E}\left[(A(z))_{ij}\right]$,
hence the off-diagonal terms of $\mathbb{E}\left[A(z)\right]$ vanish.

Consider a diagonal term $(A(z))_{ii}$. From
Equation \eqref{eq:entry-A-matrix}, we get
\begin{align}\label{eq:diag_A}
(A(z))_{ii}=\frac{d_{i}}{P}w_{i}^{T}B^{-1}(z)w_{i}=\frac{d_{i}g_{i}(z)}{1+d_{i}g_{i}(z)}.
\end{align}
By Lemma \ref{lem:concentration_gi}, $g_{i}$ lies close to $m_{P}(z)$
which itself is approximatively equal to $\tilde{m}(z)$ by Proposition
\ref{prop:convergence_stieltjes}. Therefore, we expect $\mathbb{E}\left[(A(z))_{ii}\right]=\mathbb{E}\left[\frac{d_{i}g_{i}}{1+d_{i}g_{i}}\right]$
to be at short distance from $\frac{d_{i}\tilde{m}(z)}{1+d_{i}\tilde{m}(z)}$.

In order to make rigorous this heuristic and to prove that $\mathbb{E}\left[(A(z))_{ii}\right]$
is within $\mathcal{O}(\frac{1}{P})$ distance to $\frac{d_{i}\tilde{m}(z)}{1+d_{i}\tilde{m}(z)}$,
we consider the first order Taylor approximation $\mathrm{T}_{\tilde{m}(z)}h_{i}$
of the map $h_{i}:g\mapsto\frac{1}{1+d_{i}g}$ (as in the proof Proposition
\ref{prop:convergence_stieltjes} but this time centered at $\tilde{m}(z)$).
Using the fact that $\frac{d_{i}t}{1+d_{i}t}=1-\frac{1}{1+d_{i}t}=1-h_{i}(t)$,
and inserting the Taylor approximation, $\mathbb{E}\left[(A(z))_{ii}\right]-\frac{d_{i}\tilde{m}(z)}{1+d_{i}\tilde{m}(z)}$
is equal to:
\[
h_{i}(\tilde{m}(z))-h_{i}(g_{i}(z))=\frac{1}{1+d_{i}\tilde{m}(z)}-\mathbb{E}\left[\mathrm{T}_{\tilde{m}(z)}h(g_{i}(z))\right]+\mathbb{E}\left[\mathrm{T}_{\tilde{m}(z)}h(g_{i}(z))-h(g_{i}(z))\right].
\]
Thus,
\[
\left|\mathbb{E}\left[(A(z))_{ii}\right]-\frac{d_{i}\tilde{m}(z)}{1+d_{i}\tilde{m}(z)}\right|\leq\left|\frac{1}{1+d_{i}\tilde{m}(z)}-\mathbb{E}\left[\mathrm{T}_{\tilde{m}(z)}h(g_{i}(z))\right]\right|+\left|\mathbb{E}\left[\mathrm{T}_{\tilde{m}(z)}h(g_{i}(z))-h(g_{i}(z))\right]\right|.
\]

Using Lemma \ref{lem:concentration_gi} and Proposition \ref{prop:convergence_stieltjes},
the first term $\left|\frac{1}{1+d_{i}\tilde{m}(z)}-\mathbb{E}\left[\mathrm{T}_{\tilde{m}(z)}h(g_{i}(z))\right]\right|=\frac{d_{i}\left|\mathbb{E}\left[g_{i}(z)\right]-\tilde{m}(z)\right|}{\left|1+d_{i}\tilde{m}(z)\right|^{2}}$
can be bounded by $\frac{\delta}{P}\frac{d_{i}}{\left|1+d_{i}\tilde{m}(z)\right|^{2}}$
where $\delta$ depends on $z,\gamma$ and $\frac{1}{N}\mathrm{Tr}(K)$
only. Since $\mathrm{Re}\left[\tilde{m}(z)\right]\geq0$ thus $\left|1+d_{i}\tilde{m}(z)\right|\geq\max(1,\left|d_{i}\tilde{m}(z)\right|)$,
and $\left|\tilde{m}(z)\right|\geq\frac{1}{\left|z\right|+\frac{1}{\gamma}\frac{1}{N}\mathrm{Tr}K}$ (Lemma \ref{lem:unique-fixpoint}),
the denominator can be lower bounded:
\[
\left|1+d_{i}\tilde{m}(z)\right|^{2}\geq\left|d_{i}\tilde{m}(z)\right|\geq\frac{d_{i}}{\left|z\right|+\frac{1}{\gamma}\frac{1}{N}\mathrm{Tr}K},
\]
yielding the upper bound:
\[
\left|\frac{1}{1+d_{i}\tilde{m}(z)}-\mathbb{E}\left[\mathrm{T}_{\tilde{m}(z)}h(g_{i}(z))\right]\right|\leq\frac{1}{P}\delta\left[\left|z\right|+\frac{1}{\gamma}\frac{1}{N}\mathrm{Tr}K\right].
\]

For the second term, using the same arguments as for the proof of
Proposition \ref{prop:convergence_stieltjes}, we have:
\[
\left|\mathbb{E}\left[\mathrm{T}_{\tilde{m}(z)}h(g_{i}(z))-h(g_{i}(z))\right]\right|\leq\frac{\mathbb{E}\left[\left|\tilde{m}(z)-g_{i}(z)\right|^{2}\right]}{\left|\tilde{m}(z)\right|^{2}}.
\]
Recall that $\left|\tilde{m}(z)\right|\geq\frac{1}{\left|z\right|+\frac{1}{\gamma}\frac{1}{N}\mathrm{Tr}K}$
and that, by Lemma \ref{lem:concentration_gi} and Proposition \ref{lem:concentration_stieltjes},
$\mathbb{E}\left[\left|\tilde{m}(z)-g_{i}(z)\right|^{2}\right]\leq\frac{\tilde{\delta}}{P}$
where $\tilde{\delta}$ depends on $z,\gamma$ and $\frac{1}{N}\mathrm{Tr}(K)$
only. This implies that $$\left|\mathbb{E}\left[\mathrm{T}_{\tilde{m}(z)}h(g_{i}(z))-h(g_{i}(z))\right]\right|\leq\frac{\tilde{\delta}}{P}\left[\left|z\right|+\frac{1}{\gamma}\frac{1}{N}\mathrm{Tr}K\right]^{2}.$$

As a consequence, there exists a constant $c$ which depends on $z,\gamma$
and $\frac{1}{N}\mathrm{Tr}(K)$ only such that:
\[
\left|\mathbb{E}\left[(A(z))_{ii}\right]-\frac{d_{i}\tilde{m}(z)}{1+d_{i}\tilde{m}(z)}\right|\leq\frac{c}{P}.
\]
Using the effective ridge $\tilde{\lambda}(z):=\frac{1}{\tilde{m}(-z)}$,
the term $\frac{d_{i}\tilde{m}(z)}{1+d_{i}\tilde{m}(z)}=\frac{d_{i}}{d_{i}+\tilde{\lambda}(-z)}$
is equal to $(K(K+\tilde{\lambda}I_{N})^{-1})_{ii}$ since, in the basis
considered, $K(K+\tilde{\lambda}I_{N})^{-1}$ is a diagonal matrix. Hence, we obtain:
\[
\left\Vert \mathbb{E}\left[A(z)\right]-K(K+\tilde{\lambda}I_{N})^{-1}\right\Vert _{op}\leq\frac{c}{P}
\]
which allows us to conclude.
 \end{proof}

Using the above proposition, we can bound the distance between the expected $\lambda$-RF predictor and the $\Lt$-RF predictor.

\begin{theorem}\label{app:average-rf}
For $N,P>0$ and $\lambda>0$, we have
\begin{equation}\label{app:eq:average-rf-bound}
\left|\mathbb{E} [ \frflg(x) ] - \fkl (x)\right|\leq\frac{c\sqrt{K(x,x)}\| y\| _{K^{-1}}}{P}
\end{equation}
where the effective ridge $\tilde{\lambda}(\lambda, \gamma) > \lambda$ is the unique
positive number satisfying
\begin{align}\label{app:eq:defining-effective-ridge}
\tilde{\lambda} & =\lambda+\frac{\tilde{\lambda}}{\gamma}\frac{1}{N}\sum_{i=1}^{N}\frac{d_{i}}{\tilde{\lambda}+d_{i}},
\end{align}
and where $ c > 0 $ depends on $ \lambda, \gamma $, and $ \frac 1 N \mathrm{Tr} K(X, X) $ only.
\end{theorem}

\begin{proof}
Recall that $\tilde{m}(-\lambda)$ is the unique non negative real such that $\gamma=\frac{1}{N}\sum_{i=1}^{N}\frac{d_{i}\tilde{m}(-\lambda)}{1+d_{i}\tilde{m}(-\lambda)}+\gamma \lambda\tilde{m}(-\lambda).$
Dividing this equality by $\gamma\tilde{m}(-\lambda)$ yields Equation \eqref{app:eq:defining-effective-ridge}. From now on, let $\tilde{\lambda}=\tilde{\lambda}(\lambda,\gamma)$.

We now bound the l.h.s. of Equation \eqref{app:eq:average-rf-bound}. By Proposition \ref{prop:distribution-estimator}, since $\hat{y}=A(-\lambda)y$,
the average $\lambda$-RF predictor is $\mathbb{E}\left[f_{\lambda,\gamma}^{(RF)}(x)\right]=K(x,X)K^{-1}\mathbb{E}\left[A(-\lambda)\right]y$. The  $\Lt$-KRR predictor is $f_{\tilde{\lambda}}^{(K)}(x)=K(x,X)\left(K+\tilde{\lambda}I_{N}\right)^{-1}y$.
Thus:
\begin{align*}
\left|\mathbb{E}[f_{\lambda,\gamma}^{(RF)}(x)]-f_{\tilde{\lambda}}^{(K)}(x)\right| & =\left|K(x,X)K^{-1}\left[\mathbb{E}\left[A(-\lambda)\right]-K\left(K+\tilde{\lambda}I_{N}\right)^{-1}\right]y\right|.
\end{align*}
The r.h.s. can be expressed as the absolute value of the scalar product $\left|\left\langle w,v\right\rangle _{K^{-1}}\right|=\left|v^{T}K^{-1}w\right|$
where $v=K(x,X)$ and $w=[\mathbb{E}\left[A(-\lambda)\right]-K(K+\tilde{\lambda}I_{N})^{-1}]y$.
By Cauchy-Schwarz inequality, $\left|\left\langle v,w\right\rangle _{K^{-1}}\right|\leq\left\Vert v\right\Vert _{K^{-1}}\left\Vert w\right\Vert _{K^{-1}}$.

For a general vector $v$, the $K^{-1}$-norm $\left\Vert v\right\Vert _{K^{-1}}$
is equal to the norm mininum Hilbert norm (for the RKHS associated
to the kernel $K$) interpolating function:
\[
\left\Vert v\right\Vert _{K^{-1}}=\min_{f\in\mathcal{H}, f(x_{i})=v_{i}}\left\Vert f\right\Vert _{\mathcal{H}}.
\]
Indeed the minimal interpolating function is the kernel regression
given by $f^{(K)}(\cdot)=K(\cdot,X)K(X,X)^{-1}v$ which has norm (writing
$\beta=K^{-1}v$):
\begin{align*}
\left\Vert f^{(K)}\right\Vert _{\mathcal{H}} & =\left\Vert \sum_{i=1}^{N}\beta_{i}K(\cdot,x_{i})\right\Vert _{\mathcal{H}}=\sqrt{\sum_{i,j=1}^{N}\beta_{i}\beta_{j}K(x_{i},x_{j})}=\sqrt{v^{T}K^{-1}KK^{-1}v}=\left\Vert v\right\Vert _{K^{-1}}.
\end{align*}
We can now bound the two norms $\left\Vert v\right\Vert _{K^{-1}}$
and $\left\Vert w\right\Vert _{K^{-1}}$. For $v=K(x,X)$, we have
\begin{align}
\label{eq:norm_K_x_X}
\left\Vert v\right\Vert _{K^{-1}}=\min_{f\in\mathcal{H}, f(x_{i})=v_{i}}\left\Vert f\right\Vert _{\mathcal{H}}\leq\left\Vert K(x,\cdot)\right\Vert _{\mathcal{H}}=K(x,x)^{\frac{1}{2}}.
\end{align}
since $K(x,\cdot)$ is an interpolating function for $v$.

It remains to bound $\left\Vert w\right\Vert _{K^{-1}}$. Recall that $K=UDU^T$ with $D$ diagonal, and that, from the previous proposition, $\mathbb{E}\left[A(-\lambda)\right]=UD_AU^T$ where $D_A=\mathrm{diag}\left(\frac{d_1g_1(-\lambda)}{1+d_1g_1(-\lambda)},\ldots,\frac{d_Ng_N(-\lambda)}{1+d_Ng_N(-\lambda)}\right)$. The norm $\left\Vert w\right\Vert _{K^{-1}}$ is equal to
\[
\sqrt{\tilde{y}^{T}\left[D_A-D\left(D+\tilde{\lambda}(\lambda)I_{N}\right)^{-1}\right]^{T}D^{-1}\left[D_A-D\left(D+\tilde{\lambda}(\lambda)I_{N}\right)^{-1}\right]\tilde{y}},
\]
where $\tilde{y}=U^T y$. Expanding the product, $\left\Vert w\right\Vert _{K^{-1}}=\sqrt{\sum_{i=1}^{N}\frac{\tilde{y}_{i}^{2}}{d_i}\left((D_A)_{ii}-\frac{d_{i}}{\tilde{\lambda}(\lambda)+d_{i}}\right)^{2}}$, hence by Proposition \ref{prop:ridge_expectation}, $\left\Vert w\right\Vert _{K^{-1}}\leq \frac{c}{P}\sqrt{\sum_{i=1}^{N}\frac{\tilde{y}^{2}}{d_{i}}}$. The result follows from noticing that $\sum_{i=1}^{N}\frac{\tilde{y}^{2}}{d_{i}} = \tilde{y}^T D^{-1} \tilde{y} =\|y\|_{K^{-1}}^2$:
\begin{align*}
\left|\mathbb{E}[f_{\lambda,\gamma}^{(RF)}(x)]-f_{\tilde{\lambda}}^{(K)}(x)\right| \leq \left\Vert v\right\Vert _{K^{-1}}\left\Vert w\right\Vert _{K^{-1}} \leq\frac{cK(x,x)^{\frac{1}{2}}\| y\| _{K^{-1}}}{P}.
\end{align*}
which allows us to conclude.
\end{proof}

\begin{corollary}
\label{app:cor:difference_loss_espected_kernel_loss}If $\EE_{\mathcal{D}} [K(x,x)]<\infty$,
we have that the difference of errors
$ \delta_E = \left|L (\mathbb{E} [ \frflg ] )-L (\hat{f}_{\tilde{\lambda}}^{(K)})\right| $
is bounded from above by
\[
\delta_E \leq
\frac{C\| y\| _{K^{-1}}}{P}\left(2\sqrt{L\left(\hat{f}_{\tilde{\lambda}}^{(K)}\right)}+\frac{C\| y\| _{K^{-1}}}{P}\right),
\]
where $ C $ is given by $ c \sqrt { \EE_{\mathcal{D}} [K(x,x) ] } $, with $ c $ the constant appearing in \eqref{app:eq:average-rf-bound} above.
\end{corollary}

\begin{proof}
For any function $f:\mathbb{R}^{d}\to\mathbb{R}$, we denote by $\| f\| =(\mathbb{E}_{ \mathcal{D}}\left[f(x)^{2}\right])^{\frac{1}{2}}$
its $L^{2}(\mathcal{D})$-norm. Integrating $\left|\mathbb{E}[f_{\lambda,\gamma}^{(RF)}(x)]-f_{\tilde{\lambda}}^{(K)}(x)\right|^{2}\leq\frac{c^2K(x,x)\| y\| _{K^{-1}}^{2}}{P^{2}}$
over $x\sim \mathcal{D}$,  we get the following bound:
\[
\| \mathbb{E}[f_{\lambda,\gamma}^{(RF)}]-f_{\tilde{\lambda}}^{(K)}\| \leq\frac{c\left[\mathbb{E}_{ \mathcal{D}}\left[K(x,x)\right]\right]^{\frac{1}{2}}\| y\| _{K^{-1}}}{P}.
\]
 Hence, if $f^{*}$ is the true function, by the triangular inequality,
\[
\left|\| \mathbb{E}[f_{\lambda,\gamma}^{(RF)}]-f^{*}\| -\| f_{\tilde{\lambda}}^{(K)}-f^{*}\| \right|\leq\frac{c\left[\mathbb{E}_{ \mathcal{D}}\left[K(x,x)\right]\right]^{\frac{1}{2}}\| y\| _{K^{-1}}}{P}.
\]
Notice that $L(\mathbb{E}[\hat{f}_{\gamma,\lambda}^{(RF)}])=\| \mathbb{E}[f_{\lambda,\gamma}^{(RF)}]-f^{*}\| ^{2}$
and $L(\hat{f}_{\tilde{\lambda}}^{(K)})=\| f_{\tilde{\lambda}}^{(K)}-f^{*}\| ^{2}$.
Since $\left|a^{2}-b^{2}\right|\leq\left|a-b\right|(\left|a-b\right|+2\left|b\right|)$,
we obtain
\[
\left|L\left(\mathbb{E}[\hat{f}_{\gamma,\lambda}^{(RF)}]\right)-L\left(\hat{f}_{\tilde{\lambda}}^{(K)}\right)\right|\leq\frac{c\left[\mathbb{E}_{\mathcal{D}}\left[K(x,x)\right]\right]^{\frac{1}{2}}\| y\| _{K^{-1}}}{P}\left(2\sqrt{L\left(\hat{f}_{\tilde{\lambda}}^{(K)}\right)} + \frac{c\left[\mathbb{E}_{\mathcal{D}}\left[K(x,x)\right]\right]^{\frac{1}{2}}\| y\| _{K^{-1}}}{P}\right),
\]
which allows us to conclude.
\end{proof}

\subsubsection{Properties of the effective ridge}

Thanks to the implicit definition of the effective ridge $ \tilde \lambda $, we obtain the following:
\begin{proposition}
\label{prop:fact-effective-ridge}\label{prop:effective-ridge}The effective ridge $ \tilde{\lambda} $
satisfies the following properties:
\begin{enumerate}
\item for any $\gamma>0$, we have $ \lambda <  \tilde{\lambda}(\lambda,\gamma) \leq \lambda+\frac{1}{\gamma} T $;
\item the function $\gamma\mapsto\tilde{\lambda}(\lambda,\gamma)$ is decreasing;
\item for $\gamma>1$, we have $\tilde{\lambda}\leq \frac{\gamma}{\gamma-1} \lambda$;
\item for $\gamma<1$, we have $\tilde{\lambda}\geq \frac{1-\sqrt{\gamma}}{\sqrt{\gamma}} \min_{i} d_i $.
\end{enumerate}
\end{proposition}

\begin{proof}
\textbf{(1) }The upper bound in the first statement follows directly
from Lemma \ref{lem:unique-fixpoint} where it was shown that $\tilde{m}(-\lambda)\geq\frac{1}{\lambda+\frac{1}{\gamma}\frac{1}{N}\mathrm{Tr}K}$
and from the fact that $\text{\ensuremath{\tilde{\lambda}}(\ensuremath{\lambda},\ensuremath{\gamma})}=\frac{1}{\tilde{m}(-\lambda)}$.
For the lower bound, remark that Equation (\ref{app:eq:defining-effective-ridge})
can be written as:
\[
\tilde{\lambda}(\lambda,\gamma)=\lambda+\frac{1}{\gamma}\frac{1}{N}\mathrm{Tr}[\tilde{\lambda}(\lambda,\gamma)K(\tilde{\lambda}(\lambda,\gamma)I_{N}+K)^{-1}].
\]
Since $\tilde{\lambda}(\lambda,\gamma)\geq0$ and $K$ is a positive
symmetric matrix, $\mathrm{Tr}[K[\tilde{\lambda}(\lambda,\gamma)I_{N}+K]^{-1}]\geq0$:
this yields $\tilde{\lambda}(\lambda,\gamma)\geq\lambda$.

\textbf{(2) }We show that $\gamma\mapsto\tilde{\lambda}(\lambda,\gamma)$
is decreasing by computing\textbf{ }the derivative of the effective
ridge with respect to $\gamma$. Differentiating both sides of Equation
(\ref{app:eq:defining-effective-ridge}), $\partial_{\gamma}\tilde{\lambda}=\partial_{\gamma}\left[\lambda+\frac{\tilde{\lambda}}{\gamma}\frac{1}{N}\sum_{i=1}^{N}\frac{d_{i}}{\tilde{\lambda}+d_{i}}\right]$.
The r.h.s. is equal to:

\[
\frac{\partial_{\gamma}\tilde{\lambda}}{\gamma}\frac{1}{N}\sum_{i=1}^{N}\frac{d_{i}}{\tilde{\lambda}+d_{i}}-\frac{\tilde{\lambda}}{\gamma^{2}}\frac{1}{N}\sum_{i=1}^{N}\frac{d_{i}}{\tilde{\lambda}+d_{i}}-\frac{\tilde{\lambda}}{\gamma}\frac{1}{N}\sum_{i=1}^{N}\frac{d_{i}\partial_{\gamma}\tilde{\lambda}}{(\tilde{\lambda}+d_{i})^{2}}.
\]
 Using Equation (\ref{app:eq:defining-effective-ridge}), $\frac{1}{\gamma}\frac{1}{N}\sum_{i=1}^{N}\frac{d_{i}}{\tilde{\lambda}+d_{i}}=\frac{\tilde{\lambda}-\lambda}{\tilde{\lambda}}$
and thus:
\[
\partial_{\gamma}\tilde{\lambda}\left[\frac{\lambda}{\tilde{\lambda}}+\frac{\tilde{\lambda}}{\gamma}\frac{1}{N}\sum_{i=1}^{N}\frac{d_{i}}{\left(\tilde{\lambda}+d_{i}\right)^{2}}\right]=-\frac{\tilde{\lambda}-\lambda}{\gamma}.
\]
Since $\tilde{\lambda}\geq\lambda\geq0$, the derivative of the effective
ridge with respect to $\gamma$ is negative: the function \textbf{$\gamma\mapsto\tilde{\lambda}(\lambda,\gamma)$
}is decreasing.

\textbf{(3)} Using the bound $\frac{d_{i}}{\tilde{\lambda}+d_{i}}\leq1$
in Equation (\ref{app:eq:defining-effective-ridge}), we obtain $\tilde{\lambda}\leq\lambda+\frac{\tilde{\lambda}}{\gamma}$ which, when $\gamma\geq1$, implies that
$\tilde{\lambda}\leq\lambda\frac{\gamma}{\gamma-1}.$

\textbf{(4) }Recall that $\lambda>0$ and that the effective ridge
$\tilde{\lambda}$ is the unique fixpoint of the map $f(t)=\lambda+\frac{t}{\gamma}\frac{1}{N}\sum_{i}\frac{d_{i}}{t+d_{i}}$
in $\mathbb{R}_{+}$. The map is concave and, at $t=0$, we have $f(t)=\lambda>0=t$:
this implies that $f'(\tilde{\lambda})<1$ otherwise by concavity,
for any $t\leq\tilde{\lambda}$ one would have $f(t)\leq t$. The
derivative of $f$ is $f'(t)=\frac{1}{\gamma}\frac{1}{N}\sum_{i=1}^{N}\frac{d_{i}^{2}}{\left(t+d_{i}\right)^{2}}$,
thus $\frac{1}{\gamma}\frac{1}{N}\sum_{i=1}^{N}\frac{d_{i}^{2}}{\left(\tilde{\lambda}+d_{i}\right)^{2}}<1$.
Using the fact that $d_{0}$ is the smallest eigenvalue of $K(X,X),$
i.e. $d_{i}\geq d_{0}$, we get $1>\frac{1}{\gamma}\frac{d_{0}^{2}}{\left(\tilde{\lambda}+d_{0}\right)^{2}}$
hence $\tilde{\lambda}\geq d_{0}\frac{1-\sqrt{\gamma}}{\sqrt{\gamma}}.$
\end{proof}
Similarily, we gather a number of properties of the derivative $\partial_{\lambda}\tilde{\lambda}(\lambda,\gamma)$.

\begin{proposition} \label{prop:fact-effective-ridge-derivative}For $ \gamma > 1 $, as $ \lambda \to 0 $, the derivative $ \dLt $ converges to $ \frac{\gamma}{\gamma-1} $. As $ \lambda \gamma \to \infty $, we have $ \dLt (\lambda, \gamma) \to 1 $.
\end{proposition}

\begin{proof}
Differentiating both sides of Equation (\ref{app:eq:defining-effective-ridge}),
\[
\partial_{\lambda}\tilde{\lambda}=1+\partial_{\lambda}\tilde{\lambda}\frac{1}{\gamma}\frac{1}{N}\sum_{i=1}^{N}\frac{d_{i}}{\tilde{\lambda}+d_{i}}-\tilde{\lambda}\partial_{\lambda}\tilde{\lambda}\frac{1}{\gamma}\frac{1}{N}\sum_{i=1}^{N}\frac{d_{i}}{(\tilde{\lambda}+d_{i})^{2}}.
\]
 Hence the derivative $\partial_{\lambda}\tilde{\lambda}$ satisfies
the following equality
\begin{equation}
\partial_{\lambda}\tilde{\lambda}\left(1-\frac{1}{\gamma}\frac{1}{N}\sum_{i=1}^{N}\frac{d_{i}}{\tilde{\lambda}+d_{i}}+\tilde{\lambda}\frac{1}{\gamma}\frac{1}{N}\sum_{i=1}^{N}\frac{d_{i}}{(\tilde{\lambda}+d_{i})^{2}}\right)=1.\label{eq:deriv-lambda}
\end{equation}

\textbf{(1)} Assuming $\gamma>1$, from the point 3. of Proposition
\ref{prop:effective-ridge}, we already know that $\tilde{\lambda}(\lambda,\gamma)\leq\lambda\frac{\gamma}{\gamma-1}$
hence $\tilde{\lambda}(0,\gamma)=0$. Actually, using similar arguments
as in the proof of point 3., this holds also for $\gamma=1$. Using
the fact that $\tilde{\lambda}(0,\gamma)=0$, we get $\partial_{\lambda}\tilde{\lambda}(0,\gamma)=1+\frac{\partial_{\lambda}\tilde{\lambda}(0,\gamma)}{\gamma},$
hence $\partial_{\lambda}\tilde{\lambda}(0,\gamma)=\frac{\gamma}{\gamma-1}$.

\textbf{(2)} From the first point of Proposition \ref{prop:effective-ridge},
$\tilde{\lambda}\sim\lambda$ as $\lambda\gamma\to\infty.$ Since
Equation \eqref{eq:deriv-lambda} can be expressed as:
\[
\partial_{\lambda}\tilde{\lambda}\left(1-\frac{1}{\gamma\lambda}\frac{1}{N}\sum_{i=1}^{N}\frac{d_{i}}{\frac{\tilde{\lambda}}{\lambda}+d_{i}}+\frac{1}{\gamma\lambda}\frac{\tilde{\lambda}}{\lambda}\frac{1}{N}\sum_{i=1}^{N}\frac{d_{i}}{(\frac{\tilde{\lambda}}{\lambda}+d_{i})^{2}}\right)=1,
\]
we obtain that $\partial_{\lambda}\tilde{\lambda}\to1$ as $\lambda\to\infty.$
\end{proof}

\subsubsection{Variance of the predictor\label{subsec:Variance}}
By the bias-variance decomposition, in order to bound the difference between $\mathbb{E}[L(\hat{f}_{\gamma,\lambda}^{(RF)})]$ and $L(\hat{f}_{\tilde{\lambda}}^{(K)}$, we have to bound $\mathbb{E}_{\mathcal{D}}[\mathrm{Var}(f(x))].$ The law of total variance yields $\mathrm{Var}(\hat{f}(x))=\mathrm{Var}(\mathbb{E}[\hat{f}(x)| F])+\mathbb{E}[\mathrm{Var}[\hat{f}(x)| F]].$
By Proposition \ref{prop:distribution-estimator}, we have $\mathbb{E}[\hat{f}(x)| F]=K(x,X)K(X,X)^{-1}\hat{y}$ and $\mathrm{Var}[\hat{f}(x)| F]=\frac{1}{P}\| \hat{\theta}\| ^{2}\tilde{K}(x,x).$ Hence, it remains to study $\mathrm{Var}\left(K(x,X)K(X,X)^{-1}\hat{y}\right)$ and $\mathbb{E}[\| \hat{\theta}\| ^{2}]$. Recall that we denote $T={1\over N}\Tr K(X,X)$.

This section is dedicated to the proof of the variance bound of Theorem 5.1 of the paper:

{\bf Theorem 5.1}\label{variance-ridge}\emph{
There are constants $ c_1, c_2 > 0 $ depending on $ \lambda, \gamma, T $ only such that
\begin{align*}
& \mathrm{Var}\left(K(x,X)K(X,X)^{-1}\hat{y}\right) \leq\frac{c_1 K(x,x) \ynik^2}{P} \\
& \left|\mathbb{E}\| [\hat{\theta} \| ^{2}]-\dLt y^{T}M_{\Lt}y\right| \leq\frac{c_2 \ynik^{2}}{P},
\end{align*}
where $\dLt $ is the derivative of $\tilde{\lambda}$
with respect to $\lambda$ and for $ M_{\Lt} = K(X,X)(K(X,X)+\Lt
I_N )^{-2} $. As a result
\[
\mathrm{Var}\left(\frfl(x)\right)\leq\frac{c_3 K(x,x)\ynik^2}{P},
\]
where $ c_3 > 0 $ depends on $ \lambda, \gamma, T $.}

{\bf $\bullet$ Bound on $\mathrm{Var}\left(K(x,X)K(X,X)^{-1}\hat{y}\right)$. } We first study the covariance of the entries of the matrix $$A_{\lambda}=\frac1P K^{\frac12}W^{T}\left(\frac1P WKW^{T}+\lambda\mathrm{I}_{P}\right)^{-1}WK^{\frac12},$$
where $K=\mathrm{diag}(d_{1},\dots,d_{N})$ is a positive definite diagonal matrix and $W$ is a $ P\times N$ matrix with i.i.d. Gaussian entries. In the next proposition we show a $c_{1}\over P$ bound for the covariance of the entries of $A_{\lambda}$, then we exploit this result in order to prove the bound on the variance of $K(x,X)K(X,X)^{-1}\hat{y}$.

\begin{proposition}
\label{prop:variance_labels}
There exists a constant $c'_{1}>0$ depending on $\lambda, \gamma$, and $\frac{1}{N}\mathrm{Tr}(K)$ only, such that the following bounds hold:
\begin{align*}
|\mathrm{Cov}\left((A_{\lambda})_{ii},(A_{\lambda})_{jj}\right)| & \leq\frac{c'_{1}}{P}\\
\mathrm{Var}\left((A_{\lambda})_{ij}\right) & \leq \min\left\{ \frac{d_{i}}{d_{j}},\frac{d_{j}}{d_{i}}\right\} \frac{c'_{1}}{P}.
\end{align*}
For all other cases (i.e. if $i$,$j$, $k$ and $l$ take more than two different values)$,\mathrm{Cov}\left((A_{\lambda})_{ij},(A_{\lambda})_{kl}\right)=0$.
\end{proposition}

\begin{proof}
We want to study the covariances $\mathrm{Cov}\left((A_{\lambda})_{ij},(A_{\lambda})_{kl}\right)$ for any $i,j,k,l$. Using the same symmetry argument as in the proof of Proposition \ref{prop:ridge_expectation},  $\mathbb{E}\left[(A_{\lambda})_{ij}(A_{\lambda})_{kl}\right]=0$ whenever each value in $\{i,j,k,l\}$ does not appear an even number of times in $(i,j,k,l)$. Using the fact that $A_\lambda$ is symmetric, it remains to study $\mathrm{Cov}\left((A_{\lambda})_{ii},(A_{\lambda})_{jj}\right)$, $\mathrm{Var}\left((A_{\lambda})_{ii}\right)$ and $\mathrm{Var}\left[(A_{\lambda})_{ij}\right]$ for all $i\neq j$. By the Cauchy-Schwarz inequality, any bound on  $\mathrm{Var}\left((A_{\lambda})_{ii}\right)$ will imply a similar bound on $\mathrm{Cov}\left((A_{\lambda})_{ii},(A_{\lambda})_{jj}\right)$. Besides, as we have seen in the proof of Proposition  \ref{prop:ridge_expectation}, $\mathbb{E}\left[(A_{\lambda})_{ij}\right]=0$ for any $i\neq j$. Thus, we only have to study $\mathrm{Var}\left((A_{\lambda})_{ii}\right)$ and $\mathbb{E}\left[(A_{\lambda})_{ij}^{2}\right]$.

$\bullet$ Bound on $\mathrm{Var}\left((A_{\lambda})_{ii}\right)$: From Equation \eqref{eq:diag_A},
\[
\mathrm{Var}\left((A_{\lambda})_{ii}\right)=\mathrm{Var}\left(\frac{d_{i}g_{i}}{1+d_{i}g_{i}}\right)=\mathrm{Var}\left(1-\frac{1}{1+d_{i}g_{i}}\right)=\mathrm{Var}\left(\frac{1}{1+d_{i}g_{i}}\right)\leq\mathbb{E}\left[\left(\frac{1}{1+d_{i}g_{i}}-\frac{1}{1+d_{i}\tilde{m}}\right)^{2}\right],
\]
where $g_i:=g_i(-\lambda)$. Again, we use the first order Taylor approximation $\mathrm{T}h$ of $h:x\to\frac{1}{1+d_{i}x}$ centered at $\tilde{m}:=\tilde{m}(-\lambda)$, as well as the bound \eqref{eq:taylor}, to obtain
\begin{align*}
\mathbb{E}\left[\left(\frac{1}{1+d_{i}g_{i}}-\frac{1}{1+d_{i}\tilde{m}}\right)^{2}\right] & =\mathbb{E}\left[\left(-\frac{d_{i}}{\left(1+d_{i}\tilde{m}\right)^{2}}(g_{i}-\tilde{m})+h(g_{i})-\mathrm{T}h(g_{i})\right)^{2}\right]\\
 & \leq\frac{2d_{i}^{2}}{\left(1+d_{i}\tilde{m}\right)^{4}}\mathbb{E}\left[\left(g_{i}-\tilde{m}\right)^{2}\right]+2\mathbb{E}\left[\left(h(g_{i})-\mathrm{T}h(g_{i})\right)^{2}\right]\\
 & \leq\frac{2}{6\tilde{m}^{2}}\mathbb{E}\left[\left(g_{i}-\tilde{m}\right)^{2}\right]+\frac{2}{\tilde{m}^{4}}\mathbb{E}\left[\left(g_{i}-\tilde{m}\right)^{4}\right].
\end{align*}
Using Lemma \ref{lem:concentration_gi}, we get $\mathrm{Var}\left((A_{\lambda})_{ii}\right)\leq\frac{c'_{1}}{P}$, where $c'_{1}>0$ depends on $\lambda, \gamma$, and $\frac{1}{N}\mathrm{Tr}(K)$ only.

$\bullet$ Bound on $\mathbb{E}\left((A_{\lambda})_{ij}\right)$ for $i\neq j$: Following the same arguments as for Equation \eqref{eq:diag_A}, $(A_{\lambda})_{ij}$ is equal to
\begin{align*}
(A_{\lambda})_{ij} & =\frac{\sqrt{d_{i}d_{j}}}{P}\left[w_{i}^{T}B_{(i)}^{-1}w_{j}-\frac{d_{i}g_{i}}{1+d_{i}g_{i}}w_{i}^{T}B_{(i)}^{-1}w_{j}\right] =\frac{\sqrt{d_{i}d_{j}}}{1+d_{i}g_{i}}\frac{1}{P}w_{i}^{T}B_{(i)}^{-1}w_{j},
\end{align*}
where we set $B_{(i)}:=B_i(-\lambda)$. Since $w_{i}$ and $B_{(i)}$ are independent, $\mathbb E\left [\left(w_{i}^{T}B_{(i)}^{-1}w_{j}\right)^{2}\right ]=\mathbb E\left [w_{j}^{T}B_{(i)}^{-2}w_{j}\right]$, and thus, by the Cauchy-Schwarz inequality, we have
\begin{align}
\label{eq:bound_square_A}
\mathbb{E}\left[(A_{\lambda})_{ij}^{2}\right]\leq\frac{1}{P^{2}}\sqrt{\mathbb{E}\left[\frac{d_{i}^{2}d_{j}^{2}}{\left(1+d_{i}g_{i}\right)^{4}}\right]}\sqrt{\mathbb{E}\left[\left(w_{j}^{T}B_{(i)}^{-2}w_{j}\right)^{2}\right]}.
\end{align}
Recall that $\tilde{m}:=\tilde{m}(-\lambda)$. Using the fact that $\frac{1}{1+d_{i}g_{i}} =\frac{1}{1+d_{i}\tilde{m}} + \frac{1}{1+d_{i}g_{i}}-\frac{1}{1+d_{i}\tilde{m}}$ and inserting the first Taylor approximation $\mathrm{T}h$ of $h:x\to\frac{1}{1+d_{i}x}$ centered at $\tilde{m}$, we get:
$$\mathbb{E}\left[\left(\frac{1}{1+d_{i}g_{i}}\right)^{4}\right] =\mathbb{E}\left[\left(\frac{1}{1+d_{i}\tilde{m}}-\frac{d_{i}}{\left(1+d_{i}\tilde{m}\right)^{2}}(g_{i}-\tilde{m})+h(g_{i})-\mathrm{T}h(g_{i})\right)^{4}\right].$$
Using a convexity argument,  the bound \eqref{eq:taylor}, and the lower bound on $\tilde{m}$ given by Lemma \ref{lem:unique-fixpoint}, there exists three constants $\tilde{c}_1$, $\tilde{c}_2$, $\tilde{c}_3$, which depend on $\lambda$, $\gamma$ and $\frac{1}{N}\mathrm{Tr}(K)$ only, such that  $\mathbb{E}\left[\left(\frac{1}{1+d_{i}g_{i}}\right)^{4}\right] $ is bounded by

$$ \frac{\tilde{c}_1}{\left(1+d_{i}\tilde{m}\right)^{4}}+\frac{\tilde{c}_2d_{i}^{4}}{\left(1+d_{i}\tilde{m}\right)^{8}}\mathbb{E}\left[\left(g_{i}-\tilde{m}\right)^{4}\right]+\tilde{c}_3 \mathbb{E}\left[\left(g_{i}-\tilde{m}\right)^{8}\right].$$
Thanks to Lemma \ref{lem:concentration_gi} and Proposition \ref{prop:convergence_stieltjes}, this last expression can be bounded by an expression of the form $ \frac{\tilde e_{1}}{d_{i}^{4}}+\frac{\tilde e_{2}}{P^{2}d_{i}^{4}}+\frac{\tilde e_{3}}{P^{4}}$. Note that $\frac{\tilde e_{2}}{P^{2}d_{i}^{4}}\leq \frac{\tilde e_{2}}{d_{i}^{4}}$ and $\frac{\tilde e_{3}}{P^{4}} \leq \frac{\tilde e_{3}}{\gamma^4}\frac{(\frac{1}{N}\mathrm{Tr}(K))^4}{d_i^4}$. Hence, we obtain the bound:
\begin{align*}
\mathbb{E}\left[\left(\frac{1}{1+d_{i}g_{i}}\right)^{4}\right]\leq \frac{\tilde{c}}{d_i^4},
\end{align*}
where $\tilde{c}=\tilde{e}_1+\tilde{e}_2+\frac{\tilde{e}_3(\frac{1}{N}\mathrm{Tr}(K))^4)}{\gamma^4}$ depends on  $\lambda$, $\gamma$ and  and $\frac{1}{N}\mathrm{Tr}(K)$ only.

Let us now consider the second term in the r.h.s. of \eqref{eq:bound_square_A} . Using the fact that $\|B_{(i)}\|_{op}\geq \frac{1}{\lambda}$, we get
\begin{align*}
\sqrt{\mathbb{E}\left[\left(w_{j}^{T}B_{(i)}^{-2}w_{j}\right)^{2}\right]} & \leq\sqrt{\frac{1}{\lambda^{4}}\mathbb{E}\left[\left(w_{j}^{T}w_{j}\right)^{2}\right]} =\sqrt{\frac{1}{\lambda^{4}}N(N+2)}\leq\frac{N+1}{\lambda^{2}},
\end{align*}
where we have used the fact that the second moment of a $\chi^{2}(N)$ distribution is $N(N+2).$ Together, we obtain
\begin{align*}
\mathbb{E}\left[(A)_{ij}^{2}\right] & \leq\frac{1}{P^{2}}\sqrt{\mathbb{E}\left[\frac{d_{i}^{2}d_{j}^{2}}{\left(1+d_{i}g_{i}\right)^{4}}\right]}\sqrt{\mathbb{E}\left[\left(w_{j}^{T}B_{(i)}^{-2}w_{j}\right)^{2}\right]}\\
 & \leq\frac{\tilde c d_{i}d_{j}}{d_{i}^{2}}\frac{N+1}{P^{2}\lambda^{2}}\\
 & \leq\frac{\tilde c d_{j}}{Pd_{i}\lambda^{2}\gamma}\frac{N+1}{N}\leq { c'_1 \over P} {d_{i}\over d_{j}},
\end{align*}
for $c'_{1}= 2{\tilde c\over \lambda^{2}\gamma}$. Since the matrix $A_{\lambda}$ is symmetric, we finally conclude that
\[
\mathbb{E}\left[(A_\lambda)_{ij}^{2}\right]\leq\frac{c'_{1}}{P}\min\left\{ \frac{d_{i}}{d_{j}},\frac{d_{j}}{d_{i}}\right\}.
\]

Note that $c'_{1}$ is a constant related to the bounds constructed in Lemma  \ref{lem:concentration_stieltjes} and Proposition \ref{prop:convergence_stieltjes} and as such it depends  on $\frac1N \Tr (K)$, $\gamma$ and $\lambda$ only.
\end{proof}

\begin{proposition}
\label{prop:var_exp_cond}
There exists a constant $c_1>0$ (depending on $\lambda, \gamma, T$ only) such that the variance of the estimator is bounded by
\[
\mathrm{Var}\left(K(x,X)K(X,X)^{-1}\hat{y}\right)\leq\frac{c_{1} \| y\| _{K^{-1}}^{2}K(x,x)}{P}.
\]
\end{proposition}
\begin{proof}
As in the proof of Theorem \ref{app:average-rf}, with the right change of basis, we may assume the Gram matrix $K(X,X)$ to be diagonal.

We first express the covariances of $\hat y=A(-\lambda)y$. Using Proposition Proposition \ref{prop:variance_labels}, for $i\neq j$ we have
\begin{align*}
\mathrm{Cov}\left(\hat{y}_{i},\hat{y}_{j}\right) & =\sum_{k,l=1}^{N}\mathrm{Cov}\left((A_{\lambda})_{ik},(A_{\lambda})_{lj}\right)y_{k}y_{l}=\mathrm{Cov}\left((A_{\lambda})_{ii},(A_{\lambda})_{jj}\right)y_{i}y_{j}+\mathbb{E}\left[(A_{\lambda})_{ij}^{2}\right]y_{j}y_{i},
\end{align*}
whereas for $i=j$ we have
\begin{align*}
\mathrm{Cov}\left(\hat{y}_{i},\hat{y}_{i}\right) & =\sum_{k=1}^{N}\mathrm{Cov}\left((A_{\lambda})_{ik},(A_{\lambda})_{ki}\right)y_{k}^{2}=\mathrm{Var}\left((A_{\lambda})_{ii}\right)y_{i}^{2}+\sum_{k\neq i}\mathbb{E}\left[(A_{\lambda})_{ik}^{2}\right]y_{k}^{2}.
\end{align*}

We decompose $K^{-\frac{1}{2}}\mathrm{Cov}(\hat{y},\hat{y})K^{-\frac{1}{2}}$ into two terms: let $C$ be the matrix of entries
\[
C_{ij}=\frac{\mathrm{Cov}((A_{\lambda})_{ii},(A_{\lambda})_{jj})+\delta_{i\neq j}\mathbb{E}\left[(A_{\lambda})_{ij}^{2}\right]}{\sqrt{d_{i}d_{j}}}y_{i}y_{j},
\]
and let $D$ the diagonal matrix with entries
\[
D_{ii}=\frac{\sum_{k\neq i}\mathbb{E}\left[(A_{\lambda})_{ik}^{2}\right]y_{k}^{2}}{d_{i}}.
\]
We have the decomposition  $K^{-\frac{1}{2}}\mathrm{Cov}(\hat{y},\hat{y})K^{-\frac{1}{2}}=C+D$.

Proposition \ref{prop:variance_labels} asserts that $\mathrm{Cov}((A_{\lambda})_{ii},(A_{\lambda})_{jj}\leq \frac{c'_1}{P}$ and $\mathbb{E}\left[(A_{\lambda})_{ij}^{2}\right]\leq \frac{c'_1}{P}$, and thus the operator norm of $C$ is bounded by
\begin{align*}
\| C\| _{op} & \leq\| C\| _{F}\\
 & =\sqrt{\sum_{i,j}\frac{\left(\mathrm{Cov}((A_{\lambda})_{ii},(A_{\lambda})_{jj})+\delta_{i\neq j}\mathbb{E}\left[(A_{\lambda})_{ij}^{2}\right]\right)^{2}}{d_{i}d_{j}}y_{i}^{2}y_{j}^{2}}\\
 & \leq\frac{2c'_{1}}{P}\sqrt{\sum_{ij}\frac{1}{d_{i}d_{j}}y_{i}^{2}y_{j}^{2}}\ =\ \frac{2c'_{1}\| y\| _{K^{-1}}^{2}}{P}
\end{align*}
For the matrix $D$, we use the bound  $\mathbb{E}\left[(A_{\lambda})_{ik}^{2}\right]\leq \frac{c'_1}{P} \frac{d_{i}}{d_{k}}$ to obtain
\begin{align*}
D_{ii} & =\frac{\sum_{k\neq i}\mathbb{E}\left[(A_{\lambda})_{ik}^{2}\right]y_{k}^{2}}{d_{i}} \leq\frac{c'_{1}}{P}\sum_{k\neq i}\frac{y_{k}^{2}}{d_{k}} \leq\frac{c'_{1}\| y\| _{K^{-1}}^{2}}{P},
\end{align*}
which implies that $\|D\|_{op}\leq \frac{c'_{1}\| y\| _{K^{-1}}^{2}}{P}$. As a result
\begin{align*}
\mathrm{Var}\left(K(x,X)K^{-1}\hat{y}\right) & =K(x,X)K^{-1}\mathrm{Cov}(\hat{y},\hat{y})K^{-1}K(X,x)\\
 & \leq K(x,X)K^{-\frac{1}{2}}\| C+D\| _{op}K^{-\frac{1}{2}}K(X,x)\\
 & \leq\frac{3 c'_{1}\| y\| _{K^{-1}}^{2}}{P}\| K(x,X)\| _{K^{-1}}^{2}\\
 & \leq\frac{3 c'_{1}K(x,x)\| y\| _{K^{-1}}^{2}}{P},
\end{align*}
where we used Inequality \eqref{eq:norm_K_x_X}. This yields the result with $c_1=3c'_1$.\end{proof}

{\bf $\bullet$ Bound on $\mathbb{E}_{\pi}\left[\| \hat{\theta}\| ^{2}\right]$.}
To understand the variance of the $\lambda$-RF estimator $\frfl$,
we need to describe the distribution of the squared norm of the parameters:

\begin{proposition}
\label{prop:expectation_parameter_norm} For $\gamma,\lambda>0$ there exists a constant $c_{2}>0$ depending on $\lambda,\gamma, T$ only such that
\begin{align}
\label{eq:bound_theta_norm}
\left|\mathbb{E}[\| \hat{\theta}\| ^{2}]-\partial_{\lambda }\tilde \lambda y^{T}K(X,X)\left(K(X,X)+\tilde{\lambda}I_{N}\right)^{-2}y\right|\leq\frac{c_{2}\| y\| _{K^{-1}}^{2}}{P}.
\end{align}
\end{proposition}
\begin{proof}
As in the proof of Theorem \ref{app:average-rf}, with the right change of basis, we may assume the Gram matrix $K(X,X)$ to be diagonal.
Recall that $\hat{\theta}=\frac{1}{\sqrt{P}}\left(\frac{1}{P}WK(X,X)W^{T}+\lambda I_{N}\right){}^{-1}WK(X,X)^{\frac{1}{2}}y$, thus we have:
\begin{align}
\label{eq:norm_theta}
\| \hat{\theta}\| ^{2} \ =\ \frac{1}{P}y^{T}K(X,X)^{\frac{1}{2}}W^{T}(\frac{1}{P}WK(X,X)W^{T}+\lambda I_{P})^{-2}WK(X,X)^{\frac{1}{2}}y\ =\ y^{T} A'(-\lambda)y,
\end{align}
where $A'(-\lambda)$ is the derivative of $$A(z)=\frac1P K(X,X)^{\frac12}W^{T}\left(\frac1P WK(X,X)W^{T}-z\mathrm{I}_{P}\right)^{-1}WK(X,X)^{\frac12}$$ with respect to $z$ evaluated at $-\lambda$. Let $$\tilde{A}(z)=K(X,X)(K(X,X)+\tilde\lambda(-z) \mathrm{I}_N)^{-1}.$$
Remark that the derivative of $\tilde{A}(z)$ is given by $\tilde{A}'(z)=\tilde{\lambda}'(-z)K(X,X)(K(X,X)+\tilde{\lambda}(-z)I_{N})^{-2}$. Thus, from Equation \eqref{eq:norm_theta}, the l.h.s. of \eqref{eq:bound_theta_norm} is equal to:
\begin{align}
\label{eq:bound_lhs}
\left| y^T \left( \mathbb{E}[{A}'(-\lambda)] - \tilde{A}'(-\lambda) \right) y\right|.
\end{align}
Using a classical complex analysis argument, we will show that $\mathbb{E}[A'(-\lambda)]$ is close to $\tilde{A}'(-\lambda)$  by proving a bound of the difference between $\mathbb{E}[A(z)]$ and  $\tilde{A}(z)$ for any $z\in \mathbb{H}_{<0}$.

Note that the proof of Proposition \ref{prop:ridge_expectation} provides  a bound on the diagonal
entries of $\mathbb{E}[A(z)]$, namely that for any $z\in \mathbb{H}_{<0}$,
$$\left|\mathbb{E}[(A(z))_{ii}]-(\tilde{A}(z))_{ii}\right| \leq \frac{c}{P},$$
 where $\hat{c}$ depends on $z$, $\gamma$ and $T$ only. Actually, in order to prove \eqref{eq:bound_theta_norm}, we will derive the following slightly different bound: for any $z\in \mathbb{H}_{<0}$,
\begin{align}\label{eq:new_bound}
\left|\mathbb{E}[(A(z))_{ii}]-(\tilde{A}(z))_{ii}\right| \leq \frac{\hat{c}}{d_i P},
\end{align}
where $\hat{c}$ depends on $z$, $\gamma$ and $T$ only. Let $g_i:=g_i(z)$ and $\tilde{m}:=\tilde{m}(z)$. Recall that for $h_{i}:x\mapsto\frac{d_{i}x}{1+d_{i}x}$,
one has $(A(z))_{ii}=h_i(g_i)$,  $(\tilde{A}(z))_{ii}=h_i(\tilde{m})$ and
\begin{align*}
\mathrm{T}_{\tilde{m}}h_{i}(g_{i}) & =\frac{d_{i}\tilde{m}}{1+d_{i}\tilde{m}}-\frac{d_{i}\left(g_{i}-\tilde{m}\right)}{\left(1+d_{i}\tilde{m}\right)^{2}},\\
h_{i}(g_{i})-\mathrm{T}_{\tilde{m}}h_{i}(g_{i}) & =\frac{d_{i}^{2}\left(g_{i}-\tilde{m}\right)^{2}}{\left(1+d_{i}g_{i}\right)\left(1+d_{i}\tilde{m}\right)^{2}},
\end{align*}
where $\mathrm{T}_{\tilde{m}}h_i$ is the first order Taylor approximation of $h_i$ centered at $\tilde{m}$. Using this first order Taylor approximation, we can bound the difference $\left|\mathbb{E}[h_{i}(g_{i})]-h_{i}(\tilde{m})\right| $:
\begin{align*}
\left|\mathbb{E}[h_{i}(g_{i})]-h_{i}(\tilde{m})\right| & \leq\frac{d_{i}\left|\mathbb{E}[g_{i}]-\tilde{m}\right|}{\left(1+d_{i}\tilde{m}\right)^{2}}+\frac{d_{i}^{2}}{\left(1+d_{i}\tilde{m}\right)^{2}}\mathbb{E}\left[\frac{\left|g_{i}-\tilde{m}\right|^{2}}{1+d_{i}g_{i}}\right]\\
 & \leq\frac{\mathbf{a}}{d_{i}P}+\mathbf{a}\sqrt{\mathbb{E}\left[\frac{1}{\left(1+d_{i}g_{i}\right)^{2}}\right]\mathbb{E}\left[\left|g_{i}-\tilde{m}\right|^{4}\right]},
\end{align*}
where $\mathbf{a}$ depends on $z$, $\gamma$ and $T$. We need to bound $\mathbb{E}\left[\frac{1}{\left(1+d_{i}g_{i}\right)^{2}}\right]$. Recall that in the proof of Proposition \ref{prop:variance_labels}, we bounded $\mathbb{E}\left[\frac{1}{\left(1+d_{i}g_{i}\right)^{4}}\right]$. Using similar arguments, one shows that
$$\mathbb{E}\left[\frac{1}{\left(1+d_{i}g_{i}\right)^{2}}\right]\leq \frac{\hat{e}^2}{d_i^2},$$
where $\hat{e}$ depends on $z$, $\gamma$ and $\frac{1}{N}\Tr(K(X,X))$ only.
The term $\mathbb{E}\left[\left|g_{i}-\tilde{m}\right|^{4}\right]$ is bounded using Lemmas \ref{lem:concentration_gi}, \ref{lem:concentration_stieltjes} and Proposition \ref{prop:convergence_stieltjes}. This allows us to conclude that:
\[
\left|\mathbb{E}[h_{i}(g_{i})]-h_{i}(\tilde{m})\right|\leq\frac{\hat c}{d_{i}P},
\]
where $\hat c$ depends on $z$, $\gamma$ and $\frac{1}{N}\Tr(K(X,X))$ only, hence we obtain the Inequality (\ref{eq:new_bound}).

We can now prove Inequality \ref{eq:bound_theta_norm}. We bound the difference of the derivatives of the diagonal terms of $A(z)$ and $\tilde A(z)$ by means of Cauchy formula. Consider a  simple closed path $\phi:[0,1]\to\mathbb H_{<0}$ which surrounds $z$.  Since
\[
\mathbb{E}[(A'(z))_{ii}]-(\tilde{A}'(z))_{ii}=\frac{1}{2\pi i}\oint_{\phi}\frac{\mathbb{E}[(A(z))_{ii}]-(\tilde{A}(z))_{ii}}{\left(w-z\right)^{2}}dw,
\]
using the bound \eqref{eq:new_bound}, we have:
\begin{align*}
\left|\mathbb{E}[(A'(z))_{ii}]-(\tilde{A}'(z))_{ii}\right| & \leq\frac{\hat{c}}{d_i P} \frac{1}{2\pi}\oint_{\phi}\frac{1}{\left|w-z\right|^{2}}dw \leq \frac{c_{2}}{d_{i}P},
\end{align*}
where $c_2$ depends on $z$, $\gamma$, and $T$ only.
This allows one to bound the operator norm of $K(X,X) (\mathbb{E}[A'(z)]-\tilde{A}'(z))$:
\begin{align*}
\| K(X,X) (\mathbb{E}[A'(z)]-\tilde{A}'(z))\| _{op}\leq\frac{c_{2}}{P}.
\end{align*}
Using this bound and \eqref{eq:bound_lhs}, we have
\begin{align*}
\left|\mathbb{E}[\| \hat{\theta}\| ^{2}]-\partial_{\lambda }\tilde \lambda\  y^{T}K(X,X)\left(K(X,X)+\tilde{\lambda}I_{N}\right)^{-2}y\right| & =\left|y^{T}\left(\mathbb{E}[A'(-\lambda)]-\tilde{A}'(-\lambda)\right)y\right|  \leq\frac{c_{2}\| y\| _{K^{-1}}^{2}}{P},
\end{align*}
which allows us to conclude.
\end{proof}

{\bf $\bullet$ Bound on $\mathrm{Var}\left(\frfl(x)\right)$.} We have shown all the bounds needed in order to prove the following proposition.
\begin{proposition}
 For any $x\in \mathbb{R}^d$, we have
\[
\mathrm{Var}\left(\frfl(x)\right)\leq\frac{c_3 K(x,x)\ynik^2}{P},
\]
where $ c_3 > 0 $ depends on $ \lambda, \gamma, T $.
\end{proposition}
\begin{proof}
Recall that for any $x\in \mathbb{R}^d$,
\begin{align*}
\mathrm{Var}(\frfl(x)) & =\mathrm{Var}\left(\mathbb{E}\left[\frfl(x)\mid F\right]\right)+\mathbb{E}\left[\mathrm{Var}\left[\frfl(x)\mid F\right]\right]\\
 & =\mathrm{Var}\left(K(x,X)K(X,X)^{-1}\hat{y}\right)+\frac{1}{P}\mathbb{E}\left[\| \hat{\theta}\| ^{2}\right]\left[K(x,x)-K(x,X)K(X,X)^{-1}K(X,x)\right].
\end{align*}
From Proposition \ref{prop:var_exp_cond}, $$\mathrm{Var}\left(K(x,X)K(X,X)^{-1}\hat{y}\right)\leq\frac{c_1K(x,x)\| y\| _{K^{-1}}^{2}}{P},$$ and from Proposition \ref{prop:expectation_parameter_norm}, we have:
\begin{align*}
\mathbb{E}\left[\| \hat{\theta}\| ^{2}\right] & \leq\partial_\lambda \tilde{\lambda}\ y^{T}K\left(K+\tilde{\lambda}I_{N}\right)^{-2}y+\frac{c_2\| y\| _{K^{-1}}^{2}}{P}\leq\partial_\lambda \tilde{\lambda}\ \| y\| _{K^{-1}}^{2}+\frac{c_2\| y\| _{K^{-1}}^{2}}{P}\leq \alpha \| y\| _{K^{-1}}^{2},
 \end{align*}
 where  $\alpha=\partial_\lambda \tilde{\lambda}+c_2$. Using the fact that $\tilde{K}(x,x)\leq K(x,x)$, we get
\begin{align*}
\mathbb{E}\left[\mathrm{Var}\left[\hat{f}(x)\mid F\right]\right] & =\frac{1}{P} \mathbb{E}\left[\| \hat{\theta}\| ^{2}\right]\left[K(x,x)-K(x,X)K(X,X)^{-1}K(X,x)\right]\\
 & \leq\frac{\alpha \| y\| _{K^{-1}}^{2}K(x,x)}{P}.
\end{align*}
This yields
\[
\mathrm{Var}\left(\frfl(x)\right)\leq\frac{c_3\| y\| _{K^{-1}}^{2}K(x,x)}{P},
\]
where $c_3 = \alpha+c_1$.
\end{proof}

\subsubsection{Average loss of $\lambda$-RF predictor and loss of $\tilde{\lambda}$-KRR: }

Putting the pieces together, we obtain the following bound on the difference $ \Delta_E = | \EE [ L(\frflg) ] - L (\fkl) | $  between the expected RF loss and the KRR loss:

\begin{corollary}\label{cor:expected-loss-krr-loss}
If $\mathbb{E}_{\mathcal{D}}[K(x,x)]<\infty$,
we have
\[
\Delta_E \leq \frac{C_1 \ynik }{P}\left(2\sqrt{L ( \fkl ) }+{C_2 \ynik }\right),
\]
where $C_1$ and $C_2$ depend on $\lambda$, $\gamma$, $T$ and $\mathbb{E}_{\mathcal{D}}[K(x,x)]$ only.
\end{corollary}

\begin{proof}
Using the bias/variance decomposition,  Corollary \ref{app:cor:difference_loss_espected_kernel_loss}, and the bound on the variance of the predictor, we obtain
\begin{align*}
\left|\mathbb{E}\left[L\left(\hat{f}_{\gamma,\lambda}^{(RF)}\right)\right]-L\left(\hat{f}_{\tilde{\lambda}}^{(K)}\right)\right| & \leq\left|L\left(\mathbb{E}\left[\hat{f}_{\gamma,\lambda}^{(RF)}\right]\right)-L\left(\hat{f}_{\tilde{\lambda}}^{(K)}\right)\right|+\mathbb{E}_{\mathcal{D}}\left[\mathrm{Var}\left(\hat{f}(x)\right)\right]\\
 & \leq\frac{C\| y\| _{K^{-1}}}{P}\left(2\sqrt{L\left(\hat{f}_{\tilde{\lambda}}^{(K)}\right)}+\frac{C\| y\| _{K^{-1}}}{P}\right)+\frac{c_3\| y\| _{K^{-1}}^{2}\mathbb{E}_{\mathcal{D}}\left[K(x,x)\right]}{P}\\
 & \leq \frac{C_1\| y\| _{K^{-1}}}{P}\left(2\sqrt{L\left(\hat{f}_{\tilde{\lambda}}^{(K)}\right)}+{C_2\| y\| _{K^{-1}}} \right),\end{align*}
 where $C_1$ and $C_2$ depends on $\lambda$, $\gamma$, $T$ and $\mathbb{E}_{\mathcal{D}}\left[K(x,x)\right]$ only.
\end{proof}

\subsubsection{Double descent curve}

Recall that for any $\tilde \lambda$, we denote $M_{\tilde \lambda}=K(X,X)(K(X,X)+\tilde\lambda I_N)^{-2}.$ A direct consequence of Proposition \ref{prop:expectation_parameter_norm} is the following lower bound on the variance of the predictor.

\begin{corollary}\label{lower-bound-variance}
There exists $ c_4 > 0 $ depending on $ \lambda, \gamma, T $ only such that $ \mathrm{Var} \left(\frfl(x)\right) $ is bounded from below by
\[
  \dLt \frac{y^T M_{\Lt} y }{P}\tilde{K} (x,x)
	 - \frac{c_4 K(x,x) \ynik^2}{P^2}.
\]
\end{corollary}

\begin{proof}
By the law of total cumulance, $$\mathrm{Var}\left(\frfl(x)\right) \geq \mathbb{E}\left[\mathrm{Var}\left[\frfl(x)\mid F\right]\right]\geq \frac{1}{P}\mathbb{E}\left[\| \hat{\theta}\| ^{2}\right]\tilde{K}(x,x).$$

From Proposition \ref{prop:expectation_parameter_norm}, $ \mathbb{E}[\| \hat{\theta}\| ^{2}]\geq \dLt\ y^{T}M_{\Lt} y-\frac{c_2\| y\| _{K^{-1}}^{2}}{P}, $ hence
\[
\mathrm{Var}\left(\frfl(x)\right) \geq   \dLt \frac{y^T M_{\Lt} y }{P}\tilde{K} (x,x)
	 - \frac{c_4 \tilde{K}(x,x) \ynik^2}{P^2}.\]
The result follows from the fact that $\tilde{K}(x,x)\leq K(x,x)$.
\end{proof}

\end{document}